\documentclass[twoside,11pt]{article}

\usepackage{etoolbox}
\newtoggle{arxiv}
\usepackage[dvipsnames]{xcolor}

\toggletrue{arxiv}% default

\iftoggle{arxiv}{
\usepackage[preprint]{jmlr2e}
}{
\usepackage{jmlr2e}
}
\usepackage{amsmath}
\usepackage{paralist}
\usepackage{booktabs}
\usepackage{bm}
\usepackage[bb=boondox]{mathalfa} % double struck 0 and 1
\usepackage{xcolor}

\usepackage{siunitx}
\usepackage{multirow} % for the big matrix
\usepackage{nicematrix} % for the big matrix

\usepackage{wrapfig}
\usepackage{caption}
\usepackage[rightcaption]{sidecap}

\usepackage{cleveref}

\DeclareMathOperator*{\argmin}{arg\,min}
\DeclareMathOperator*{\argmax}{arg\,max}

\newcommand\vone{\mathbb{1}}
\newcommand\vzero{\mathbb{0}}

\newcommand\ico{\bm{\Upsilon}}
\newcommand\tsp{\tau}
\newcommand\Tsp{\mathbf{T}}

\newcommand\TP{\psi}

\newcommand\LL{\ell}

\newcommand\myheader[1]{\noindent\(\bullet\,\,\)\textbf{#1}.}

\newcommand{\bfA}{\mathbf{A}}

\newcommand{\bfv}{\mathbf{v}}

\newcommand{\bfz}{\mathbf{z}}
\newcommand{\bfu}{\mathbf{u}}
\newcommand{\bfx}{\mathbf{x}}
\newcommand{\bfr}{\mathbf{r}}
\newcommand{\bmzeta}{\bm{\zeta}}
\newcommand{\bmalpha}{\bm{\alpha}}

\newcommand{\bmpi}{\mathbf{D}}
\newcommand{\bmPi}{\mathbf{M}}
\newcommand{\bfP}{\mathbf{P}}

\newcommand{\bmxi}{\bm{\xi}}

\newcommand{\bfw}{\mathbf{w}}
\newcommand{\bfp}{\mathbf{p}}
\newcommand{\bfq}{\mathbf{q}}
\newcommand{\calL}{\mathcal{L}}
\newcommand{\sqfunc}{g} %  sigma superscript t
\newcommand{\interior}{\mathsf{int}}

\newcommand{\bfe}{\mathbb{e}} %  sigma superscript t
\newcommand{\ran}{\mathsf{ran}} %  sigma superscript t
\newcommand{\cran}{\mathsf{ran}_{\mathtt{c}}} %  sigma superscript t
\newcommand{\bfc}{\mathbf{c}} %  sigma superscript t

\newcommand{\inj}{\mathsf{pad}} %  sigma superscript t
\newcommand{\proj}{\mathbf{Q}} %  sigma superscript t
\newcommand{\drop}{\mathbf{P}} %  sigma superscript t
\newcommand{\con}{\mathsf{trunc}} %  sigma superscript t
\newcommand{\closure}{\mathsf{clos}} %  sigma superscript t

\newcommand{\gsm}{\mathsf{link}}
\newcommand{\iter}{\times}

\newcommand{\margin}{\bmpi}
\usepackage{pict2e}

\newsavebox\foobox

\usepackage{xparse}

\numberwithin{theorem}{section} % important bit

\hypersetup{hidelinks}

\usepackage{lastpage}
\jmlrheading{25}{2024}{1-\pageref{LastPage}}{12/23; Revised 5/24}{5/24}{23-1599}{Yutong Wang and Clayton Scott}

\ShortHeadings{Unified Margin-Based Classification}{Wang and Scott}
\firstpageno{1}

\begin{document}

\title{Unified Binary and Multiclass Margin-Based Classification}

\author{%
  \name{Yutong Wang\textsuperscript{1,2}} \email{yutongw@umich.edu}\\
  \name{Clayton Scott\textsuperscript{1,3}} \email{clayscot@umich.edu}\\
  \addr
  \textsuperscript{1}Department of Electrical Engineering and Computer Science\\
  \textsuperscript{2}Michigan Institute of Data Science\\
 \textsuperscript{3}Department of Statistics\\
 University of Michigan\\
 Ann Arbor, MI 48109, USA
 % Seattle, WA 98195-4322, USA
}
% \author{%
%  \name{Yutong Wang} \email{yutongw@umich.edu}\\
%  \addr Electrical and Computer Engineering\\
%  University of Michigan
%  \AND
%  \name{Clayton Scott} \email{clayscot@umich.edu}\\
%  \addr Electrical and Computer Engineering, Statistics\\
%  University of Michigan
% }

\editor{Zhihua Zhang}

\maketitle

\begin{abstract}
  The notion of margin loss has been central to the development and analysis of algorithms for binary classification. To date, however, there remains no consensus as to the analogue of the margin loss for multiclass classification. In this work, we show that a broad range of multiclass loss functions, including many popular ones, can be expressed in the \emph{relative margin form}, a generalization of the margin form of binary losses. The relative margin form is broadly useful for understanding and analyzing multiclass losses as shown by our prior work \citep{wang2020weston,wang2021exact}. To further demonstrate the utility of this way of expressing multiclass losses, we use it to extend the seminal result of \cite{bartlett2006convexity} on classification-calibration of binary margin losses to multiclass.
  We then analyze the class of Fenchel-Young losses, and expand the set of these losses that are known to be classification-calibrated.
\end{abstract}

% \begin{abstract}
% Multiclass classification research has received less attention and has been somewhat separated from the development in binary classification.
% Towards bridging this gap, we introduce
% a matrix label encoding
%    that extends the \({\pm 1}\) binary label encoding.
%    We introduce
% \emph{{p}ermutation {e}quivariant {r}elative {m}argin} (PERM) losses, a novel family of multiclass loss functions that unify well-known binary and multiclass losses including cross entropy, exponential, and Fenchel-Young losses.
%    We demonstrate that classifiers trained using PERM losses have reduced memory cost while achieving equivalent performance.
%    Moreover, we prove a novel sufficient condition for classification-calibration of PERM losses, which extends a seminal result of \citet{bartlett2006convexity} in the binary case.
% \end{abstract}

% \noindent\textcolor{red}{Things that should be moved to the GP loss paper}
% \\
% \noindent\textcolor{blue}{Things that should be copied to the GP loss paper}
% \tableofcontents
%

% \yw{Make sure transposes are denoted via \(\top\)}

% \yw{What did we reuse from Weston-Watkins paper?}

% \yw{Consistency transfer property}

\section{Introduction}

Classification into \(k \ge 2\) categories is the learning task of selecting a \emph{classifier}, i.e., a function from the feature space \(\mathcal{X}\) to the set of labels \([k]:=\{1,\dots, k\}\), given training data. Many of the most popular and successful classification methods aim to find a classifier with minimum risk, where the risk of a classifier is the expected value of a loss function that measures the quality of predictions. Indeed, logistic regression, support vector machines, boosting, and neural network methods can all be viewed as algorithms to minimize the risk associated with a certain loss. In practice, ``loss-based'' approaches to binary (\(k = 2\)) and multiclass (\(k \ge 3\)) cases are formulated differently. Because of this discrepancy, the theory and practice of multiclass methods often lag behind their binary counterparts. The purpose of this work is to bridge this gap by introducing a framework that unifies binary and multiclass loss-based classification.

The standard approach to binary classification is to learn a \emph{discriminant} function \(g : \mathcal{X} \to \mathbb{R}\) that maps an instance \(x\) to a \emph{discriminant} \(g(x)\), e.g., \(g(x) = w^T x + b \) for linear classification. Identifying \(\{1,2\}\) with \(\{- 1, +1\}\), a label \(y\) is predicted by \(\mathrm{sign}(g(x))\). The loss ascribed to a discriminant \( g \) and a pair \( (x,y)\) is defined in terms of a function \(\psi: \mathbb{R} \to \mathbb{R}\), with the loss being equal to \(\TP((-1)^{y} g(x))\). Loss functions of this form are referred to as binary \emph{margin loss} functions. Examples include the logistic loss \( \psi(t) = \log(1 + e^{-t})\) (logistic regression), the hinge loss \( \psi(t) = \max(0, 1-t)\) (support vector machines), the exponential loss \( \psi(t) = e^{-t}\) (AdaBoost), and the sigmoid loss \( \psi(t) = 1/(1 + e^t)\) (some neural networks).

The conventional approach to multiclass classification seeks to learn a \emph{class-score function} \(f = (f_{1},\dots, f_{k}): \mathcal{X} \to \mathbb{R}^{k}\), e.g., a feed-forward neural network. The label for an input \(x\) is predicted by taking the argmax of the \emph{class-score vectors} \(\bfv := f(x)\). Multiclass loss functions, e.g., cross entropy, may be viewed as functions \(\calL: \{1,\dots,k\} \times \mathbb{R}^{k} \to \mathbb{R}\). The loss incurred by a class-score function \(f\) on a pair \((x,y) \in \mathcal{X} \times [k]\) is then \(\calL(y, f(x))\). When a multiclass loss function is expressed in terms of a class-score function output, we say that it is in \emph{class-score form}.

We show that a large family of multiclass loss functions in class-score form, including cross-entropy, multiclass exponential loss \citep{mukherjee2013theory}, multiclass hinge losses \citep{crammer2001algorithmic,weston1998multi},
Gamma-Phi losses \citep{beijbom2014guess},
and Fenchel-Young losses \citep{blondel2020learning}, can be expressed in what we call the \emph{relative-margin form}. Instead of being defined in terms of a class-score function output \(f(x)\), in the relative margin form a loss is expressed \(\TP(\ico_{y} g(x))\), which involves the following elements:

\emph{1.}~a symmetric\footnote{
  Recall that a function  is symmetric if its value does not change when its input is permuted.
  } function \(\psi : \mathbb{R}^{k-1} \to \mathbb{R}\),

\emph{2.}~a set of \((k-1)\times (k-1)\) matrices \(\{\ico_{y}\}_{y \in [k]}\) that encode the label, and

\emph{3.}~a discriminant function of the form \(g : \mathcal{X} \to \mathbb{R}^{k-1}\).

When \(k=2\), we have \(\ico_{y} = (-1)^{y}\), and the relative margin form coincides with the standard binary margin form. See Figure~\ref{fig:framework} for an illustration of the framework. Thus, the relative margin form generalizes the notion of margin loss from binary classification.

\begin{figure}
  \centering
  \includegraphics[width=0.95\textwidth]{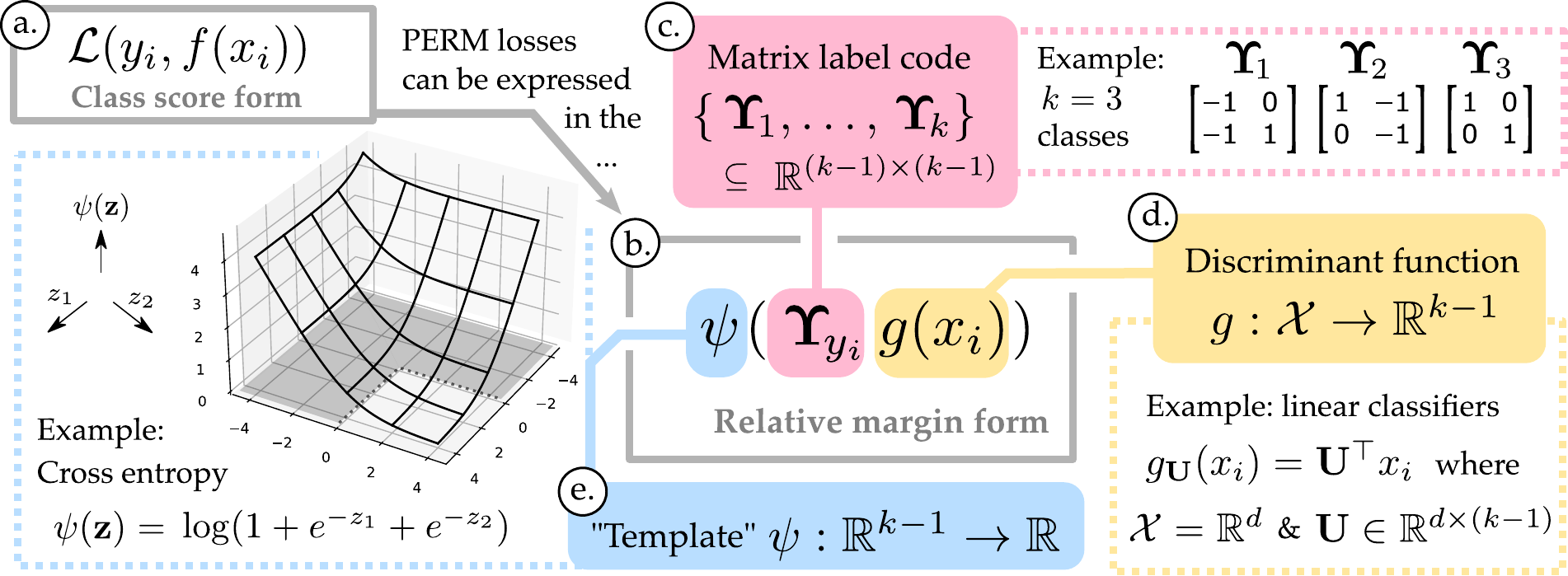}
  \caption{\textbf{Relative margin form}.\ \emph{Panel \((\mathrm{a})\)}:\ Multiclass losses \(\calL\) satisfying the permutation equivariant and  relative margin-based conditions (i.e., PERM losses, Definition~\ref{definition:PERM-loss}) can be expressed in the relative margin form as in \emph{Panel \((\mathrm{b})\)}. See
    {Theorem}~\ref{theorem:relative-margin-form}. The relative margin form employs three components: \emph{Panels \((\mathrm{c})\)}.\ the matrix label code \(\{\ico_{y}\}_{y \in [k]}\),
\emph{\((\mathrm{d})\)}.\
the discriminant function \(g\),
and
\emph{\((\mathrm{e})\)}.\
the ``template'', a symmetric function
 \(\psi : \mathbb{R}^{k-1} \to \mathbb{R}\).
  }\label{fig:framework}
\end{figure}

\subsection{Our contributions}\label{section:our-contribution}

In this work, we develop the relative margin from as described above, and illustrate the utility of this framework by establishing
novel results on \emph{classification-calibration}\footnote{For readers unfamiliar with this theory, we provide a brief review below in the ``Background on Classification-Calibration'' section
(Section~\ref{section:background-on-CC}).}. More specifically, our contributions are:

\noindent\(\bullet\) \textbf{Theorem~\ref{theorem:relative-margin-form}: characterization of losses expressible in the relative margin form}.
We show that a multiclass loss function  can be expressed in the relative margin form if and only if both of the following conditions are met: \emph{1.~\underline{p}ermutation \underline{e}quivariance}: \(\calL\) treats all classes equally, and \emph{2.~\underline{r}elative \underline{m}argin-based}: \(\calL(y,f(x))\) depends only on the differences between the \(f_{j}(x)\)'s rather than the values of the \(f_{j}(x)\)'s themselves.
Below, such losses are referred to as \underline{PERM} losses.

Our characterization implies that many popular losses such as the cross entropy, multiclass exponential, and Fenchel-Young losses~\citep{blondel2019structured} are PERM losses.
PERM losses are characterized by their \emph{template}, a symmetric function \(\TP: \mathbb{R}^{k-1} \to \mathbb{R}\).
Thus, our result implies that for the analysis of existing and design of new PERM losses, it suffices to focus on the template \(\TP\) which can be simpler than the original \(\calL\).

\noindent\(\bullet\) \textbf{{Theorem}~\ref{theorem: nested family of regular PERM losses are CC - exposition version}:
  Extending a fundamental result on classification-calibration (CC) to the multiclass case}.
A seminal result of \citet{bartlett2006convexity} in the binary case shows that a convex margin loss is classification-calibrated if and only if \(\TP\) is differentiable at $0$ and has negative derivative there.
We use the PERM loss framework to prove a multiclass extension  of this result
for \(\calL\) that are totally regular
({Definition}~\ref{definition:regular-PERM-loss}).
We use Theorem~\ref{theorem: nested family of regular PERM losses are CC - exposition version} to prove a novel sufficient condition for CC of \emph{sums} of losses
({Proposition}~\ref{proposition:sum-of-totally-regular-loss-is-totally-regular}),
and, as a corollary, establish CC for sums of Gamma-Phi losses (Example~\ref{example:sum-of-Gamma-Phi-losses}).
A key component of the totally regular property is that the gradient of the template \(\nabla_{\TP}\) must be entry-wise negative everywhere\footnote{Thus, in the binary case, our result is weaker than that of \citet{bartlett2006convexity}, which only requires negativity at \(0 \in \mathbb{R}\). We leave as an open question whether this gap can be closed.}.

\noindent\(\bullet\) \textbf{{Theorem}~\ref{theorem: sufficient condition for FY loss to be CC - exposition version}:
Expanding previous sufficient conditions of CC for Fenchel-Young losses}.
Fenchel-Young losses, defined via taking the convex conjugate of a certain \emph{negentropy} function, have recently been proposed for structured prediction, including multiclass classification.
However, an open question is whether Fenchel-Young losses are classification-calibrated under the assumption that the negentropy is \emph{strictly convex}.
Our Theorem~\ref{theorem: sufficient condition for FY loss to be CC - exposition version}
answers this in the affirmative.

\noindent\(\bullet\) \textbf{The matrix label code}.
The key component in our analysis is
the
\emph{matrix label code}  (Definition~\ref{definition:matrix-label-code}) which extends to the multiclass case the universally adopted \(\{\pm 1\}\) label code for binary classification.
To obtain our main results, we developed a suite of technical lemmas for working with the matrix label code in
Appendix~\ref{section:appendix:matrix-label-code}.
We expect the matrix label code to be useful for multiclass classification research beyond the scope here.

% \noindent\textbf{Review of classification-calibration}.
\subsection{Background}\label{section:background-on-CC}
In this section, we define the probabilistic setting assumed throughout this work.
Moreover, we first review the background on classification-calibration.

% \noindent\(\bullet\) \textbf{Classification-calibration}.
Let \(\{(x_{i},y_{i})\}_{i = 1}^{n}\) be drawn from a joint distribution \(P\) over \(\mathcal{X} \times [k]\).
Let \(\mathbb{I}\) be the indicator function and \( \mathbb{I}\{y \ne j\}\) be the \emph{01-loss} for \(y,j \in [k]\).
The goal of classification is to select a {classifier} \(h : \mathcal{X} \to [k]\) minimizing the  \emph{01-risk}
  \(R_{01,P}(h) := \mathbb{E}_{(X,Y) \sim P} \left[ \mathbb{I}\{Y \ne h(X)\}\right]\) objective.
A fundamental approach for learning a classifier is \emph{empirical risk minimization} (ERM), which selects an \(h\) minimizing the empirical 01-risk
  \(\hat{R}_{01}^{n}(h) := \frac{1}{n}\sum_{i=1}^{n} \mathbb{I}\{y_{i} \ne h(x_{i})\}\) over a class of functions \(\mathcal{H}\).

Directly minimizing the empirical 01-risk objective is often intractable\footnote{See \citet{bhattacharyya2018hardness} and the references therein} due to the discreteness of the objective.
The {surrogate-based approach} addresses this issue as follows.
Define
the \emph{\(\calL\)-risk}
for a fixed
continuous loss function \(\calL\)
by \(  R_{\calL}(f):=\mathbb{E}_{(X,Y) \sim P} [\calL(Y,f(X))]\)
and
the {empirical} \(\calL\)-risk by \(\hat{R}_{\calL}^{n}(f)
:=\frac{1}{n}\sum_{i=1}^{n} \calL(y_{i},f(x_{i}))
\).  Minimizing \(\hat{R}_{\calL}^{n}(f)\) over a family of \(f\)'s, e.g., neural networks, is employed as a tractable surrogate objective.
A class-score function \(f\) induces a discrete-valued classifier
\(\argmax \circ f(x) :=\argmax_{j=1,\dots,k} f_{j}(x)\) with ties broken arbitrarily.
An important question is whether good performance with respect to the surrogate \(\calL\)-risk ``transfers'' back to the 01-risk, the original objective.
% The quantity \(\calL(y, f(x))\) is
% the loss incurred by \(f\) when instance \(x\) has label \(y\).

%  formally relates 01-risk minimization with \(\calL\)-risk minimization.
%
\begin{definition}\label{definition:consistency-transfer-property}
A surrogate loss \(\calL\) has the \emph{consistency transfer property}\footnote{
This property appears in many works, e.g., \citet{steinwart2007compare,bartlett2006convexity,tewari2007consistency,zhang2004statistical}. The name ``consistency transfer property'' was used by \citet{wang2023classification}.
} if
% \begin{quote}
for any distribution \(P\) over \(\mathcal{X} \times [k]\)
and
any sequence of class-score functions \(\{\hat{f}^{(n)}\}_{n}\), e.g., \(\hat{f}^{(n)}\) is an empirical \(\calL\)-risk minimizer over a family of functions depending on \(n\), the following is satisfied:
\(\lim_{n\to\infty}R_{\calL}(\hat{f}^{(n)}) = \inf_{f} R_{\calL}(f)\)
implies that
\(\lim_{n\to\infty}R_{01}(\argmax \circ \hat{f}^{(n)}) = \inf_{h}R_{01}(h)\).
  The infimums are, respectively, over all measurable functions \(f: \mathcal{X} \to \mathbb{R}^{k}\) and \(h : \mathcal{X} \to [k]\).
\end{definition}

Thus, the consistency transfer property (CTP) provides justification for \(\calL\)-risk minimization when minimizing the 01-risk is the original objective of interest.
For binary margin-based losses \(\TP\), the seminal result of \citet[Theorem 1.3]{bartlett2006convexity} shows that CTP is equivalent to a functional property of \(\TP\) known as \emph{classification-calibration} (CC).
For the multiclass case,
\citet{tewari2007consistency} define the CC property (reviewed in
detail in
Section~\ref{sec:cc-and-consistency} below) and show its equivalence to the CTP as well.

\subsection{Related work}\label{section:related-work}

\noindent\(\bullet\) \textbf{Sufficient conditions for CC of multiclass convex margin-based losses}.
For binary convex loss in margin form, the sufficient conditions for CC are easy to verify and apply to common losses of interest. Indeed, \citet[Theorem 2.1]{bartlett2006convexity} asserts that a nonnegative binary convex loss in margin form \(\TP: \mathbb{R} \to \mathbb{R}_{\ge 0}\) is CC if and only if \(\TP\) is differentiable at \(0\) and  \(\TP'(0)<0\).

There are several difficulties for extending the above characterization to the multiclass case.
The result in the binary case relies on the expression of a binary margin-based loss as \(\TP((-1)^{y}z)\).
While proposed multiclass extensions of the expression exist (to be reviewed below),
existing work on sufficient conditions for CC of multiclass losses focus on losses in the class-score form \(\calL : [k] \times \mathbb{R}^{k} \to \mathbb{R}\).

Most previous works studied CC of losses of a specific form  such as the Gamma-Phi losses \citep{zhang2004statistical,beijbom2014guess,wang2023classification},
Fenchel-Young losses \citep{duchi2018multiclass,blondel2020learning},
and hinge-like losses \citep{tan2022loss}.
Notably, \citet{tewari2007consistency} derive a characterization of CC for multiclass losses analogous to that of \citet[Theorem 2.1]{bartlett2006convexity}.
However, their sufficient condition \citep[Theorem 7]{tewari2007consistency}, while geometrically elegant, is hard to verify for common losses such as the multiclass exponential loss\footnote{See Footnote 7 in \citet[Table 1]{tewari2007consistency}.}.

\noindent\(\bullet\) \textbf{Multiclass margins \& margin-based loss}.
A line of work has considered the problem of extending the binary margin \((-1)^{y}z\) and binary margin form \(\TP((-1)^{y}z)\) to the multiclass setting. For the margin,
\citet{crammer2001algorithmic}
and
\citet[\S9.2]{mohri2018foundations}
defined a notion of \emph{scalar-valued} multiclass margin as \(\min_{j \in [k]:j\ne y} v_{y} -v_{j}\) where \(y\) is the ground truth label.
In other words, this notion of multiclass margin is the minimum difference of score between the ground truth label and the rest of the labels.
While intuitive, according to \citet[\S5.1]{tewari2007consistency}, convex losses based on this notion of margin are never classification-calibrated.

\citet{lee2004multicategory,zou2008new} proposes definitions of \emph{vector-valued} multiclass margin vectors and associated margin losses for multiclass SVMs and boosting, respectively.
These losses are known to be classification-calibrated if and only if sum-to-zero constraints are enforced on the input margin vectors \citep[Theorem 6]{dogan2016unified}.
\citet{lee2004multicategory} develops a multiclass SVM based on a hinge loss that leverages this notion of margin.
However, in practice, enforcing these sum-to-zero constraints leads to significantly slower computations \citep{fu2022proximal}.

\citet{dogan2016unified} developed a framework using relative margins to unify the analysis of several variants of multiclass support vector machines.
This notion of relative margin\footnote{\citet{rosset2003margin} also uses the name ``margin vector'' in conflict with  the later work by \citet{zou2008new} which defines a different notion of margin. Due to this confusion, we will refer to the earlier notion of \citet{rosset2003margin} by ``relative margins''.} is first introduced by \citet[\S 4]{rosset2003margin}.
Both of these prior works only established classification-calibration of specific losses for specific algorithms.
By contrast, our work develops a framework that characterizes when multiclass losses can be expressed via relative margins and proves general classification-calibration results for a large family of losses.

% We first review the literature that relates specifically to two types of losses relevant to this chapter:  the Gamma Phi losses and the Fenchel-Young losses.

% Perhaps the most ubitiqous label code is the \emph{softmax coding} \citep[\S 4.3.5]{james2013introduction} also known as the \emph{one-hot} encoding \citep{goodfellow2016deep}.
% However, the softmax coding is

\noindent\(\bullet\) \textbf{Label encodings for multiclass classification}.
As alluded to earlier, the multiclass SVM introduced by \citet{lee2004multicategory} is classification-calibrated provided that certain sum-to-zero constraints are enforced.
Moreover, enforcing these constraints is computationally prohibitive.
The \emph{simplex code}~\citep{hill2007framework}
is a multiclass label code designed to address this computational issue by using a constraint-free reparametrization of the multiclass SVM dual formulation
\citep{wu2010multicategory,saberian2011multiclass,mroueh2012multiclass,van2016gensvm,pouliot2018equivalence}.

% it is not clear
% how popular existing losses can be expressed in this framework, including the original cross entropy\footnote{\citet{mroueh2012multiclass} mentions that the logistic loss can be expressed in the simplex-encoding framework, but an explicit description was not given and not obvious. \citet[\S 2]{vigogna2022multiclass} mentions a multiclass generalization  of the logistic loss using simplex-encoding, but it is unclear if this generalization is equivalent to the ordinary cross entropy/multinomial logistic loss.} (Example~\ref{example:cross entropy}).
% To the best of our knowledge, only the MSVM of \citet{lee2004multicategory} has been formulated in the simplex encoding. The formulation was originally proposed a new kind of MSVM in \citep{mroueh2012multiclass}, but the equivalence to that of \citet{lee2004multicategory} was discovered later by \citet{pouliot2018equivalence}.
% \citet{saberian2019multiclass} applied simplex encoding in the context of multiclass boosting with Gamma Phi loss.

Another approach to label encoding is the
{multivector construction}\footnote{
% In contrast to these previous works, our matrix product label encodings are \emph{matrix-valued}, which
  Also known as ``Kesler's construction'' which, according to \citet{crammer2003ultraconservative}, is credited to \citep{kesler1961preliminary}.
  A description of the construction can be found in  \citet[\S5.12]{duda2006pattern} and \citet[\S17.7]{shalev2014understanding}.
}~\citep[\S17.7]{shalev2014understanding}
which has been used to study the multiclass perceptron, logistic regression and SVMs \citep[\S5.12]{duda2006pattern}.
Moreover, the multivector construction has been used to study the PAC learning sample complexity of multiclass linear classifiers~\citep{daniely2014optimal}.

While the simplex code and the multivector constructions are elegant techniques, the scope of these work are on specific algorithms and classifiers. In particular, 
a framework for general multiclass margin-based losses and analysis of their properties, e.g., classification-calibration, is lacking.

Our work fills this gap by using the framework of the relative-margin form and matrix label code to \emph{1.} characterize the set of multiclass losses expressible in this form (Theorem~\ref{theorem:relative-margin-form}),
\emph{2.} prove sufficient conditions for classification-calibration of a large family of multiclass losses (Theorems~\ref{theorem: nested family of regular PERM losses are CC - exposition version})
and \emph{3.} apply our results to expand the previously known family of classification-calibrated Fenchel-Young losses (Theorem~\ref{theorem: sufficient condition for FY loss to be CC - exposition version}).

% \footnote{To the best of our knowledge, this negentropy is only known to be strictly convex as strong convexity was not discussed~\citep{mensch2019geometric,feydy2019interpolating}.
%   However, we note that our result on Fenchel-Young loss is specifically for multiclass classification and thus does \emph{not} apply to their setting.
% Expanding our analysis to their setting is an interesting direction of future work.}.

% \noindent\textbf{Multiclass overparametrized learning}. Loss functions for multiclass classification have recently been studied in the context of learning in overparametrized settings where models can interpolate the training data.
% While the cross entropy/multinomial logistic loss is the \emph{de facto} choice in training neural networks, recent works  have questioned this convention and pushed forward understanding of alternative losses such as the squared loss~\citep{hui2020evaluation,muthukumar2021classification}.

\subsection{Notations}

In this subsection, we briefly discuss some of the key notations.
For the reader's convenience, we include a more comprehensive reference for the notations in Table~\ref{table:conventions}
in
Section~\ref{section:appendix:full-notations} of the appendix.
Moreover, we tabulate the mathematical objects defined in Table~\ref{table:definitions}.

Throughout this work, let \(k \ge 2\) denote the number of classes.
Denote the \(k\)-probability simplex by \(\Delta^k = \{ \bfp \in \mathbb{R}_{\ge 0}^k: \sum_{j=1}^k p_j= 1\}\).

\noindent\textbf{Vectors/matrices}. Let the square bracket with subscript \([\cdot]_j\) be the projection of a vector onto its \(j\)-th component, i.e., \([\bfv]_j := v_j\) where \(\bfv = (v_1,\dots, v_k) \in \mathbb{R}^k\).
Given two vectors \(\bfw, \bfv \in \mathbb{R}^{k}\), we write \(\bfw \succeq \bfv\) (resp.\ \(\bfw \succ \bfv\)) if \(w_{j } \ge v_{j}\) (resp.\ \(w_{j } > v_{j}\)) for all \(j \in [k]\).
% Likewise,  we write \(\bfw \succ \bfv\) if \(w_{j } > v_{j}\) {for all} \(j \in [k]\).
All-zeros/all-ones/\(i\)-th elementary basis vector in \(\mathbb{R}^{n}\) are denoted \(\vzero^{(n)}\)/\(\vone^{(n)}\)/\(\bfe^{(n)}_{i}\), respectively.
When the ambient dimension is clear, we drop the superscript \((n)\).
The \(n\times n\) identity matrix is denoted \(\mathbf{I}_{n}\).

\noindent\textbf{Permutations and permutation matrices}.
A bijection from \([k]\) to itself is called a \emph{permutation} (on \([k]\)).
Denote by \(\mathtt{Sym}(k)\) the set of all permutations on \([k]\).
% We often write \(\sigma \sigma'\) instead of \(\sigma \circ \sigma'\) for the compositions of two permutations \(\sigma, \sigma' \in \mathtt{Sym}(k)\).
% For \(i,j \in [k]\), let \(\tsp_{(i,j)} \in \mathtt{Sym}(k)\) denote the \emph{transposition} which swaps \(i\) and \(j\), leaving all other elements unchanged.
% More precisely, \(\tsp_{(i,j)}(i) = j\), \(\tsp_{(i,j)}(j) = i\) and \(\tsp_{(i,j)}(y) = y\) for \(y \in [k] \setminus \{i,j\}\).
% Define the notational shorthand \(\tsp_{i} := \tsp_{(1,i)}\), the transposition that swaps \(1\) and \(i\).
For each \(\sigma \in \mathtt{Sym}(k)\), let \(\mathbf{S}_{\sigma}\) denote the permutation matrix corresponding to \(\sigma\).
In other words, if \(\bfv \in \mathbb{R}^{k}\) is a vector, then \([\mathbf{S}_{\sigma} \bfv]_{j} = [ \bfv ]_{\sigma(j)}= v_{\sigma(j)}\).
% Note that if \(\sigma, \sigma' \in \mathtt{Sym}(k)\), then \(\mathbf{S}_{\sigma \sigma'} = \mathbf{S}_{\sigma} \mathbf{S}_{\sigma'}\).
% Define the notational shorthand \(\Tsp_{(i,j)} := \mathbf{S}_{\tsp_{(i,j)}}\) the matrix corresponding to the transposition of \(i\) and \(j\).
% Likewise, define \(\Tsp_{i} := \Tsp_{(1,i)}\).

% \usepackage{booktabs}

\section{Permutation equivariant relative margin (PERM) losses}

In this section, we define PERM losses and prove fundamental properties used throughout the rest of the work.
We begin with a discussion of the two defining properties: permutation equivariance and relative-margins.

Recall that in the introduction, we denoted loss functions as \(\calL : [k] \times \mathbb{R}^{k} \to \mathbb{R}\).
  For a class-score function \(f: \mathcal{X} \to \mathbb{R}^{k}\), the loss is evaluated as \(\calL(y,f(x))\) on an instance \((x,y)\).
  Throughout the rest of this work, for mathematical convenience we will use the equivalent notation convention \(\calL_{y}(f(x))\) where \(y\) appears  in the subscript.
In this convention, a multiclass loss function
is
a vector-valued
function
  \(\calL  =(\calL_1,\dots, \calL_k) : \mathbb{R}^k \to \mathbb{R}^k\).

  With this convention, the property that the loss treats each of the \(k\) classes equally is captured by
  \emph{permutation equivariance}, a notion similar to but distinct from symmetry\footnote{
  Recall that a function  is symmetric if its value does not change when its input is permuted. Sometimes, this is also referred to as permutation \emph{in}variance \citep[\S 3.1]{bronstein2021geometric}.
For our purposes, \emph{equi}variance is the correct descriptor.
  }.
   Denote by \(\bfv\) the class-score vector \(f(x)\) on some generic instance \((x,y)\).
The permutation equivariance property states that if the class-score vector \(\bfv\) is ``relabeled'' by some permutation \(\sigma \in \mathtt{Sym}(k)\), then the vector of losses \(\calL(\bfv)\) should be ``relabeled'' in the same way, i.e., \(\calL(\mathbf{S}_{\sigma}\bfv) = \mathbf{S}_{\sigma}\calL(\bfv)\).

{Relative margins}\footnote{The idea of relative margins goes back to \cite{rosset2003margin}.
  We note that \citet{jebara2008relative} also use this term in a unrelated context in \emph{binary} SVMs.
} have been recently used by \citet{dogan2016unified} and by \citet{fathony2016adversarial} in the context of multiclass SVMs
that only utilize
the set of differences \(v_{y}-v_{j}\)  over all \(y,j \in [k]\) such that \(y\ne j\), rather than the non-relative or ``absolute''
class-scores \((v_{1},\dots, v_{k})\) themselves.
Below, it will be notationally convenient to use a matrix that converts the vector of class-score \(\bfv\)
into relative margins:

\begin{definition}\label{definition:relative-marginalization-mapping}
  Let
  \(    \bmpi:=
    % \begin{bmatrix}
    %   1 & -1 & 0 & \cdots & 0\\
    %   1 & 0 & -1 & \cdots & 0\\
    %   1 & 0 & 0 & \ddots & 0\\
    %   1 & 0 & 0 & \ddots & -1
    % \end{bmatrix}
    \begin{bmatrix}
                    -\mathbf{I}_{k-1}
      &
      \vone^{(k-1)}
    \end{bmatrix} \in \mathbb{R}^{(k-1) \times k}
\).
Observe that
\([\bmpi\bfv]_{y}
  = v_k - v_{y}
\) for \(y \in [k-1]\) and \(\bfv \in \mathbb{R}^{k}\).
Equivalently,
\(  \bmpi\bfv  = (v_k-v_1, v_k - v_2,\dots, v_k - v_{k-1})^{\top}
\).
  \end{definition}
% \color{blue}
%

 We are now ready to define

\begin{definition}[PERM losses]\label{definition:PERM-loss}
  Let \(k \ge 2\) be an integer.
A \emph{\(k\)-ary multiclass loss function}
is
a vector-valued
function
  \(\calL  =(\calL_1,\dots, \calL_k) : \mathbb{R}^k \to \mathbb{R}^k\).
  % We say that \(\calL\) is \emph{strict} if in addition \(v_y < v_{y'}\) implies \(\calL_y(\mathbf{v}) > \calL_{y'}(\mathbf{v})\).
  We say that \(\calL\) is
  \begin{compactenum}
    \item\label{definition:PERM-loss-permutation-equivariant}
    {\emph{permutation equivariant}}
        if
        \(\calL(\mathbf{S}_{\sigma}\bfv) = \mathbf{S}_{\sigma}\calL(\bfv)\)
for all \(\bfv \in \mathbb{R}^k\) and \(\sigma \in \mathtt{Sym}(k)\),
    \item\label{definition:PERM-loss-reduced-form} \emph{relative margin-based} if
      for each \(y \in [k]\) there exists a function \(\LL_{y} : \mathbb{R}^{k-1} \to \mathbb{R}\) so that
      \begin{equation}
        \calL_y(\mathbf{v}) = \LL_{y}(\bmpi \bfv) =
        \LL_{y}(v_{k}-v_{1},\,v_{k}- v_{2},\,\dots,\, v_{k}-v_{k-1}), \quad \mbox{for all \(\bfv \in \mathbb{R}^{k}\).}
        \label{equation:relative-margin-form-with-reduced-form}
    \end{equation}
We refer to the vector-valued function \(\LL := (\LL_{1},\dots, \LL_{k})\)  as the \emph{reduced form} of \(\calL\).
    \item\label{definition:PERM-loss-template} \underline{\emph{PERM}} if
      \(\calL\) is both \underline{p}ermutation \underline{e}quivariant and \underline{r}elative \underline{m}argin-based.
      In this case, the function \(\TP := \LL_{k}\)  is referred to as the \emph{template} of \(\calL\).
  \end{compactenum}
\end{definition}

While the concepts of permutation equivariance and relative margin have been studied largely in isolation in previous works,   our work is the first to systematically study the properties of  losses having both properties.
Below in
  {Theorem}~\ref{theorem:relative-margin-form}, we show
that a PERM loss is completely determined by its template
via what we call the
\emph{relative-margin form}.

\begin{remark}[On the name ``template'']
  The symbol of the template \(\TP\) is chosen intentionally to match that of \citet{bartlett2006convexity}, where
  a (binary) margin loss is expressed as \(\TP((-1)^{y} g(x))\).
  Treating \(g(x) = z\) as an arbitrary input to \(\TP\), the rationale behind the name is that the ``positive branch'' \(\TP((-1)^{2}z)\),  of the margin loss serves as a ``template''
  for the ``negative branch'' \(\TP((-1)^{1}z)\).
\end{remark}

Before proceeding, let us examine some notable PERM losses:
\begin{example}\label{example:cross entropy}
  The \emph{cross entropy} (also multinomial logistic) loss is given by
  \[
\calL_{y}^{\mathsf{CE}}(\bfv) = \log\left(1 + \textstyle\sum_{j \in [k] : j \ne y} \exp(- (v_{y} - v_{j}))\right), \quad \mbox{for all }\bfv \in \mathbb{R}^{k}.
  \]
  It is easy to see that its template is the function \(\psi^{\mathsf{CE}}(\bfz) =\log\left(1 + \textstyle\sum_{j=1}^{k-1} \exp(- z_{j})\right) \).
 When \(k=2\), we have \(\psi^{\mathsf{CE}}(z) = \log(1+ \exp(-z))\) which is the binary logistic loss (also known as the binary cross entropy).
 When \(k=3\), we have
 \(\psi^{\mathsf{CE}}(\bfz) = \log(1 + \exp(-z_{1}) + \exp(-z_{2}))\) is a function defined over the 2D plane \(\mathbb{R}^{2}\), plotted in Figure~\ref{fig:framework}.
\end{example}
\begin{example}\label{example:multiclass-exponential-loss}
 The \emph{Gamma-Phi} loss \citep{beijbom2014guess}, denoted \(\calL_{y}^{\gamma,\phi}(\bfv) \), is a generalization of the cross entropy loss defined as follows: Let \(\gamma : \mathbb{R} \to \mathbb{R}\)
  and \(\phi : \mathbb{R} \to \mathbb{R}_{\ge 0}\) be functions.
  Define
  \[
\calL_{y}^{\gamma,\phi}(\bfv) := \gamma\Big(\textstyle\sum_{j \in [k] : j \ne y} \phi( v_{y} - v_{j})\Big), \quad \mbox{for all }\bfv \in \mathbb{R}^{k}.
  \]
  It is easy to see that its template is the function \(\psi^{\gamma,\phi}(\bfz) =\gamma\left(\textstyle\sum_{j=1}^{k-1} \phi( z_{j})\right) \).
  The cross entropy \(\calL^{\mathsf{CE}}\) is the Gamma-Phi loss where
  \(\phi(t) = \exp(-t)\) and
  \(\gamma(t) = \log(1+t)\).
  The \emph{multiclass exponential}  \citep{mukherjee2013theory} loss \(\calL^{\mathsf{Exp}}\) is the Gamma-Phi loss where
  \(\phi(t) = \exp(-t)\) and
  \(\gamma\) is the identity.
\end{example}
\begin{example}
  A notable case of the Gamma-Phi loss is when \(\phi(t) = \max\{0,1-t\}\) is the hinge loss and
  \(\gamma\) is the identity. The resulting loss is known as the Weston-Watkins hinge loss \citep{weston1998multi,bredensteiner1999multicategory,vapnik1998statistical}, which is well-known to be \emph{not} classification-calibrated \citep{liu2007fisher,tewari2007consistency}.
  However, the Weston-Watkins hinge loss is calibrated with respect to a discrete loss called the ``ordered partition loss'' related to ranking with ties \citep{wang2020weston}.
\end{example}

\begin{example}
  The \emph{Crammer-Singer} hinge loss \citep{crammer2001algorithmic} is a well-known loss that is a PERM loss but \emph{not} a Gamma-Phi loss. It is defined as
  \[
    \textstyle
\calL_{y}^{\mathsf{CS}}(\bfv) := \max_{j \in [k] : j \ne y} \left\{ \max\{0,1-( v_{y} - v_{j})\}\right\}, \quad \mbox{for all }\bfv \in \mathbb{R}^{k}.
  \]
\end{example}

Section~\ref{section:FY-loss-main} features another example of a family of PERM losses, namely, the  \emph{Fenchel-Young losses} \citep{blondel2020learning}.

\subsection{Matrix label code and the relative margin form}\label{section:matrix-label-code}

% \begin{example}
%   The multinomial logistic, also known as the cross entropy, is both a Gamma-Phi loss and a Fenchel-Young loss.
% \end{example}

% \begin{definition}[Obsolete]\label{definition:ordered}
%   A multiclass loss function \(\calL\) is \emph{ordered} if the following holds:
% Let \(\bfv \in \mathbb{R}^k\) and \(y,y' \in [k]\) be such that \(v_y \le v_{y'}\). Then \(\calL_{y}(\bfv) \ge \calL_{y'}(\bfv)\).
% \end{definition}

This section introduces the multiclass generalization of \(\{\pm1\}\): the \emph{matrix label code} \(\{\ico_{y}\}_{y=1}^{k}\).
 % In Theorem~\ref{theorem:relative-margin-form} below, we show that using a PERM loss with the matrix label code leads to a direct multiclass generalization of the binary margin loss framework.
% In the following definition, we introduce the \emph{matrix label code}, a set of matrices \(\{ \ico_{1},\dots, \ico_{k}\}\) generalizing of the familiar \(\{\pm 1\}\) label in binary classification to the \(k\)-ary multiclass classification.
      % \subfile{ico}
\begin{definition}[Matrix label code]\label{definition:matrix-label-code}
  For \(k \ge 2\) and \(y \in [k]\), define the \((k-1)\times(k-1)\) matrix \(\ico_y\)  as follows:
  For \(y=k\), \(\ico_k := \mathbf{I}_{k-1}\).
  For \(y \in [k-1]\), define \(\ico_{y}\) column-wise by
  \[
    [\ico_{y}]_{:j} :=
    \begin{cases}
      \bfe^{(k-1)}_{j} &: j \ne y \\
      - \vone^{(k-1)} &: j = y,
    \end{cases}
   \quad \mbox{for each \(j \in [k-1]\).}
  \]
\end{definition}
Equivalently, to construct \(\ico_{y}\)  for each \(y \in [k-1]\), first take the identity matrix \(\mathbf{I}_{k-1}\), then replace the \(y\)-th column by all
\(-1\)'s.

Note that \(\ico_{y}^{2} = \mathbf{I}_{k-1}\) for all \(y \in [k-1]\), i.e., \(\ico_{y}\) is an \emph{involution}.
Moreover, in the binary case where \(k=2\), we have \(\ico_{1} = -1\) and \(\ico_{2}=1\), i.e., the matrix label code reduces to the label encoding \(\{\pm 1\}\).

Next, recall that {a function \(f: \mathbb{R}^n \to \mathbb{R}\)
  is \emph{symmetric} if \(f (\mathbf{S}_{\sigma}(\cdot)) = f(\cdot)\) for all \(\sigma \in \mathtt{Sym}(n)\).}
  We now state the main property of the PERM loss and the matrix label code:
  \begin{theorem}[Relative-margin form]\label{theorem:relative-margin-form}
    Let \(\calL: \mathbb{R}^k \to \mathbb{R}^k\) be a PERM loss with template \(\TP\),
    and let
\(\bfv \in \mathbb{R}^{k}\) and \(y \in [k]\) be arbitrary.
    Then
    \(\TP\) is a symmetric function.
    Moreover,
  \begin{equation}
    \label{equation:relative-margin-form}
\calL_y(\bfv) =
  \TP(\ico_y \bmpi\bfv).
\end{equation}
Conversely, let \(\TP : \mathbb{R}^{k-1} \to \mathbb{R}\) be a symmetric function. Define a multiclass loss function \(\calL = (\calL_{1},\dots, \calL_{k}) : \mathbb{R}^{k} \to \mathbb{R}^{k}\)  according to \cref{equation:relative-margin-form}. Then \(\calL\) is a PERM loss with template \(\TP\).
  % Furthermore, \(y \in \argmax \bfv \) if and only if \(\ico_{y}\bmpi \bfv \ge 0\).
  \end{theorem}
    \Cref{equation:relative-margin-form} is referred to as the \emph{relative margin form} of \(\calL\).
Theorem~\ref{theorem:relative-margin-form} and other results in this section are proved in Section~\ref{section:appendix:proof-of-theorem-relative-margin-form} in the appendix.
  \begin{remark}
The relative margin form  can be interpreted as a one-to-one correspondence between PERM losses \(\calL : \mathbb{R}^{k} \to \mathbb{R}^{k}\) and symmetric functions \(\TP:\mathbb{R}^{k-1} \to \mathbb{R}\).
Thus, we can refer to a PERM loss by either \(\calL\) or \(\TP\) without ambiguity.
% Theorem~\ref{theorem:relative-margin-form} can be used to
% easily show that many losses currently in the literature are in fact PERM losses.
  \end{remark}

\iftoggle{arxiv}{
  \begin{remark}[Uniqueness of the matrix label code]
  In Section~\ref{section:uniqueness-of-MLC} of the appendix, we
  prove a unique-ness result: any set of matrices \(\{{\ico}_{y}'\}_{y=1}^{k}\) satisfying the conclusion of Theorem~\ref{theorem:relative-margin-form}  is equal to the matrix label code
 (Definition~\ref{definition:matrix-label-code}) up to row permutations of the \(\ico_{y}\)'s.
  \end{remark}
}{
  \begin{remark}[Uniqueness of the matrix label code]
  In
  Section~\ref{section:uniqueness-of-MLC} of the appendix,
we prove\footnote{We omit the proofs of some of the technical intermediate lemmas. For the omitted proofs, see  the arXiv version of our
    manuscript \citep{wang2023unified}.
  } a unique-ness result: any set of matrices \(\{{\ico}_{y}'\}_{y=1}^{k}\) satisfying the conclusion of Theorem~\ref{theorem:relative-margin-form}  is equal to the matrix label code
 (Definition~\ref{definition:matrix-label-code}) up to row permutations of the \(\ico_{y}\)'s.
  \end{remark}
}
% \begin{remark}
%         Ostensibly, Definition~\label{definition:matrix-label-code}  singles out class ``\(k\)'' as \emph{special} and thus a natural question is whether the PERM losses framework has a built-in bias toward or against class ``\(k\)''.
%         Later in
%         Section~\ref{section:matrix-label-code}, we show how to use \(\TP\) to reformulate existing algorithms such as the multiclass SVM.
%         We will see  in
%         Theorem~\ref{proposition:PERM-loss-ERM} that the  PERM loss approach using \(\TP\)
% and proper regularization\footnote{See the Frobenius norm term in Eqn.~\ref{equation:SVM-with-U}}
%         is completely equivalent to the traditional approach using \(\calL\), which does not have \emph{any} class bias.
%         In the appendix, we show that similar techniques work for applying PERM loss to gradient-based methods.
% \end{remark}
\begin{remark}\label{remark:explicit-form-of-matrix-label-form-times-discrimaint}
  Observe
  that  \(\ico_y \bmpi\bfv = (v_{y} - v_{1}, v_{y}-v_{2}, \dots, v_{y} - v_{k})^{\top} \in \mathbb{R}^{k-1} \) where the \(v_{y} -v_{y}\) entry is omitted.
  This is Lemma~\ref{lemma: rho pi sigma relation} in the appendix.
  Now, note that the right hand side is exactly the relative margin as defined in \citet[Eqn.\ (8)]{rosset2003margin}.
  The advantage of the left hand expression  \(\ico_y \bmpi\bfv \) is that the label \(y\) acts  via its label encoding \(\ico_{y}\)  by left matrix multplication.
  This ``disentanglement'' of the label encoding and the loss can facilitate calculations and will be crucial in our key
{Lemma}~\ref{lemma:Az-matrix-is-nonsing-M-matrix}.
%   For a more concrete example, we calculate the gradient of \(\TP(\ico_{y} g_{\theta}(x))\) with respect to the parameters \(\theta\) of a family of functions
%   \(\{g_{\theta}\}_{\theta}\) using the fact that ``derivative and matrix multplication commutes''.
%   Consider linear classifiers
% \(\{g_{\mathbf{U}}\}_{\mathbf{U} \in \mathbb{R}^{d \times (k-1)}}\) on \(\mathcal{X} = \mathbb{R}^{d}\) as in Figure~\ref{fig:framework}.
% Then it is easy to verify using the chain rule that \(\frac{\partial}{\partial \mathbf{U}} \TP(\ico_{y} \mathbf{U}^{\top} x)
% =
% x \nabla_{\TP}(\ico_{y} \mathbf{U}^{\top} x)^{\top}\ico_{y}
% \).
\end{remark}
\begin{remark}[Prior work]\label{remark:our-prior-work}
The matrices in the definition of the matrix label code have been used by \citet{wang2020weston,wang2021exact} for
analyzing the Weston-Watkins (WW) SVMs \citep{weston1998multi}.
\citet{wang2020weston} show that the hinge loss from the WW-SVM is calibrated with respect to the ordered partition loss, and uses this theory to explain the empirical observation made by \cite{dogan2016unified} that WW-SVM performs well even under significant label noise.
\citet{wang2021exact}
derive a reparametrization of the Weston-Watkins dual problem that decomposes into subproblems that can be solved exactly in \(O(k\log(k))\) time, leading to faster performance for the linear WW-SVMs with many classes.
The ``disentanglement'' discussed in the previous remark is the key ingredient for deriving this reparametrization.

This work extends \citet{wang2020weston,wang2021exact} beyond the multiclass SVM setting, and  provides a framework for using these matrices for
margin-based losses.
\end{remark}

\subsection{Relationship between class-score functions and discriminant functions}
As mentioned in the introduction, the ``sign'' function
converts a real-valued discriminant function \(g: \mathcal{X} \to \mathbb{R}\) to a discrete-valued classifier \(h : \mathcal{X} \to \{1,2\}\)
by choosing \(h(x)\) to be \(y \in \{1,2\}\) such that \((-1)^{y}g(x) \ge 0\) with ties broken arbitrarily.
To generalize this to the multiclass case, first recall that when \(k=2\), the matrix label code is equal to \(\ico_{y} = (-1)^{y}\).
Thus,  \((-1)^{y}g(x) \ge 0\) iff \(\ico_{y}g(x) \ge 0\).
More simply put, \(y\) is the predicted class if and only if \(\ico_{y}\) is the sign of \(g(x)\).

Next, we define the multiclass ``sign'' function analogously.
 Given
 a vector-valued discriminant function \(g: \mathcal{X} \to \mathbb{R}^{k-1}\) we define a discrete-valued classifier \(h : \mathcal{X} \to [k]\) by taking \(y \in [k]\) such that \(\ico_{y} g(x) \succeq \vzero\). Similar to the binary case, ties can be broken arbitrarily.
Thus, \(\ico_{y}\) can be viewed as the sign of \(g(x)\).
 To relate this multiclass ``sign'' function method for producing a classifier to the conventional ``\(\argmax\)'' method, we have
  \begin{proposition}\label{proposition:argmax-equivalence}
    Let \(\bfv \in \mathbb{R}^{k}\), \(\bfz := \bmpi \bfv\), and \(y \in [k]\). Then
    \(y \in \argmax_{j} v_{j}\) iff
    \(\ico_{y} \bfz \succeq \vzero\).
    Furthermore, let \(\bfv' :=
    \begin{bmatrix}
0&-\mathbf{z}^{\top}
    \end{bmatrix}^{\top}
    \). Then
\(y \in \argmax_{j} v_{j}\)
iff
\(y \in \argmax_{j} v_{j}'\).
  \end{proposition}
  \begin{proof}
    For the first part, recall from
   {Remark}~\ref{remark:explicit-form-of-matrix-label-form-times-discrimaint} that
    \(\ico_{y} \bmpi \bfv = (v_{y}-v_{1},v_{y}-v_{2},\dots, v_{y} - v_{k})^{\top}\) where the \(v_{y}- v_{y}\) entry is omitted.
    Therefore,
    the assertion
    ``\(y \in \argmax_{j} v_{j}\) iff
    \(\ico_{y} \bfz \succeq 0\)'' follows immediately from the definition of \(\succeq\) and \(\bfz\).
    For the ``Furthermore'' part, note that by construction we have
    \(\bfv' = (0,v_{2}-v_{1}, v_{3}-v_{1}, \dots, v_{k} -v_{1})^{\top}\).
    By adding the constant \(v_{1}\) to all entries of \(\bfv'\), we recover \(\bfv\). Thus, the indices maximizing \(\bfv'\) are the same as those of \(\bfv = (v_{1},v_{2},\dots, v_{k})^{\top}\).
  \end{proof}
  The first part  of Proposition~\ref{proposition:argmax-equivalence} gives an intuitive explanation for our earlier definition of a multiclass ``sign''.
  The second part of
Proposition~\ref{proposition:argmax-equivalence} gives a simple formula for computing the ``\(\argmax\)'' from the discriminant function.
\begin{remark}
    Proposition~\ref{proposition:argmax-equivalence} can be seen as an equivalence between class-score functions and discriminant functions.
    Given a class-score function \(f : \mathcal{X} \to \mathbb{R}^{k}\),
    we can derive
    a discriminant function \(g : \mathcal{X} \to \mathbb{R}^{k-1}\)
    by defining \(g(x) := \bmpi f(x)\) for all \(x \in \mathcal{X}\).
    Conversely, we can view the discriminant function \(g\) as the ``given'',
    and derive a class-score function \(f\) by defining \(f(x) :=
    \begin{bmatrix}
      0 & -g(x)
    \end{bmatrix}^{\top}
    \) for all \(x \in \mathcal{X}\).
    In both cases, \(\argmax f(x) = \argmax
    \begin{bmatrix}
      0 & -g(x)
    \end{bmatrix}^{\top}
    \).
  \end{remark}

\section{Classification-calibration and Consistency}\label{sec:cc-and-consistency}

This section is a brief review of the core definitions and theorem from \citet{tewari2007consistency} in preparation for our results on sufficient conditions for classification-calibration.

Below, we assume the probability setting introduced earlier in Section~\ref{section:background-on-CC}.
\begin{definition}[Range and its convex hull]\label{definition:loss-surface}
  Let \(f: \mathbb{R}^m \to \mathbb{R}^n\) be a function.
  Denote by \( \ran(f) := \{f(x) : x \in \mathbb{R}^m\} \) the \emph{range} of \(f\), and \(\cran(f) := \mathtt{conv}(\ran(f))\) the convex hull of the range of \(f\).
  \end{definition}

  \begin{definition}[\cite{tewari2007consistency}]\label{definition: multiclass classification calibration}
  A set \(S \subseteq \mathbb{R}^k_+\) is \emph{classification-calibrated} if there exists a
  function\footnote{The function \(\theta\) is called a \emph{calibrated link} for \(S\).}
  \(\theta : \mathbb{R}^k \to [k]\) such that
  \begin{equation}
    \textstyle
    \inf \{ \langle \bfp, \bm{\zeta}\rangle :\bm{\zeta} \in S \,:\,p_{\theta(\bm{\zeta})} < \max \bfp \} > \inf_{\bm{\zeta} \in S} \langle \bfp, \bm{\zeta}\rangle
\label{equation:mcc-inequality}
\end{equation}
 for all \(\bfp \in \Delta^k\).
 A multiclass loss function \(\calL\) is \emph{classification-calibrated} if \(\cran(\mathcal{L})\) is classification-calibrated.
\end{definition}

Intuitively, Definition~\ref{definition: multiclass classification calibration} says that the lowest achievable conditional risk when predicting the wrong label (Eqn.~\eqref{equation:mcc-inequality} LHS) is still strictly larger than the
conditional Bayes risk (Eqn.~\eqref{equation:mcc-inequality} RHS).
  Define \(\underline{\smash{\argmax}} : \mathbb{R}^k \to [k]\) by
  \(
    \underline{\smash{\argmax}} (v) = \min \{i \in [k]: v_i = \max_{j \in [k]} v_j\}.
  \)
  When $\calL$ is permutation equivariant and classification-calibrated, we can take $\theta$ in Definition~\ref{definition: multiclass classification calibration} to be simply $\underline{\smash{\argmax}}$. See \citet[Lemma 4]{tewari2007consistency}.
As alluded to earlier, the significance of Definition~\ref{definition: multiclass classification calibration}
is demonstrated by the following theorem:
% which
% paraphrases
% \citet[Theorem 3]{zhang2004statistical}
% and
% one implication\footnote{\citet[Theorem 2]{tewari2007consistency} says the other implication is true as well: $\calL$-surrogate risk minimization being 01-consistent implies that $\cran(\calL)$ is classification-calibrated. However, we do not need the implication in this direction. It is nevertheless a curious question if there exists $\calL$ having the ISC property when $\cran(\calL)$ is not classification-calibrated.}
% of
% \citet[Theorem 2]{tewari2007consistency}
% when $\calL$ is a permutation equivariant loss.

\begin{theorem}[\citealt{tewari2007consistency}]\label{theorem: multiclass classification calibration}
    Let \(\calL\) be a PERM loss. Let \(\mathcal{F}\) be the set of \textcolor{black}{all} Borel functions \(\mathcal{X} \to \mathbb{R}^k\).
  If
\(\calL\) is classification-calibrated,
  then \(\calL\) has the \emph{consistency transfer property}:
  For all  sequence of function classes \(\{\mathcal{F}_n\}_n\) such that \(\mathcal{F}_n \subseteq \mathcal{F}\), \(\bigcup_n \mathcal{F}_n = \mathcal{F}\), all \(\hat f_n \in \mathcal{F}_n\) and all probability distributions \(P\) on \(\mathcal{X}\times [k]\)
  \[\textstyle
    R_{\calL,P}(\hat f_n) \overset{P}{\to}  \inf_{f} R_{\calL,P}(f)
    \quad \mbox{implies} \quad
    R_{01,P}(\underline{\smash{\argmax}} \circ\hat f_n) \overset{P}{\to}
    \inf_{h}
    R_{01,P}(h)
  \]
  where the infimums are taken over all Borel functions \(f : \mathcal{X} \to \mathbb{R}^k\) and \(h : \mathcal{X} \to [k]\), respectively.

\end{theorem}
\begin{remark}
In applications, \(\hat f_n\) is often taken to be an \(\calL\)-risk empirical minimizer over a training dataset of cardinality \(n\). However, the above property holds for \emph{any} sequence of functions $\hat f_n \in \mathcal{F}_n$.
\end{remark}

Checking that a multiclass loss \(\calL\) is classification-calibration is a non-trivial task, requiring delicate analysis of the loss function \citep{tewari2007consistency}.
Recently, \citet{wang2023classification} established a sufficient condition of classification-calibration for Gamma-Phi losses (Example~\ref{example:multiclass-exponential-loss}).
Below, we derive a sufficient conditions for general PERM losses.

\section{Regular PERM losses}
To demonstrate the utility of the relative margin form and Theorem~\ref{theorem:relative-margin-form},  we prove
in this section
{{Theorem}~\ref{theorem: nested family of regular PERM losses are CC - exposition version},
  a sufficient condition for classification-calibration (CC) of multiclass loss functions}.
 To this end, we begin with the necessary definitions.
  \begin{definition}\label{definition:coercive-functions}
    A function $f : \mathbb{R}^{n} \to \mathbb{R}$ is
    \begin{compactenum}
      \item \emph{coercive}
        if for all $c \in \mathbb{R}$, the
        set
    $
      \{\bfv \in \mathbb{R}^{n}: f(\bfv) \le c\}
    $
    (i.e., the \emph{$c$-sublevel set})
    is bounded,
  \item \emph{semi-coercive}
    if for all $c \in \mathbb{R}$ there exists $b \in \mathbb{R}$ such that
    \[
      \textstyle
      \{\bfv \in \mathbb{R}^{n}: f(\bfv) \le c\}
      \subseteq
      \{\bfv \in \mathbb{R}^n: b \le \min_{j \in [n]} v_{j} \}.
    \]
    \end{compactenum}
  \end{definition}
  The definition of a coercive function is well-studied \citep{boyd2004convex}.
  However, semi-coercivity appears to be new.
  Intuitively, a function is semi-coercive if, for all \(c \in \mathbb{R}\), its \(c\)-sublevel set is contained in a translation of the positive orthant.

\begin{definition}[Regular PERM loss]\label{definition:regular-PERM-loss}
    Let \(\calL\) be a PERM loss with template \(\TP\).
    We say that \(\calL\) is \emph{regular} if
    \(\TP\) is nonnegative, twice differentiable, strictly convex, semi-coercive,
    % the partial derivative \(\frac{\partial\TP}{\partial z_{1}} : \mathbb{R}^{k-1} \to \mathbb{R}\) is semi-bounded \yw{where is this used?},
    and
      the gradient \(\nabla_{\TP}(\bfz) \prec \vzero\) is entrywise negative for all \(\bfz \in \mathbb{R}^{k-1}\).
  \end{definition}

Below, we give a sufficient condition for a Gamma-Phi loss (Example~\ref{example:multiclass-exponential-loss}) to be a regular PERM loss. But first, we discuss the intuition behind the condition \(\nabla_{\TP}(\bfz) \prec \vzero\).
The condition \(\nabla_{\TP}(\bfz) \prec \vzero\) is reminiscent of a condition in \citet[Theorem 6]{bartlett2006convexity}, which shows that in the binary case a convex margin loss \(\TP\) is classification-calibrated if and only if \(\TP\) is differentiable at \(0\) and \(\dot \TP(0) <0\) where the ``overdot'' denotes taking derivative of a univariate function.
If we view \(-\dot\TP(\cdot)\) as a ``vector field'' on \(\mathbb{R}\), then this condition on the derivative can be stated as ``the vector field \(-\dot\TP(\cdot)\) near \(0\) points toward the positive half of the real line''.

\begin{wrapfigure}{r}{0.5\textwidth}
  \centering
  \includegraphics[width=0.29\textwidth]{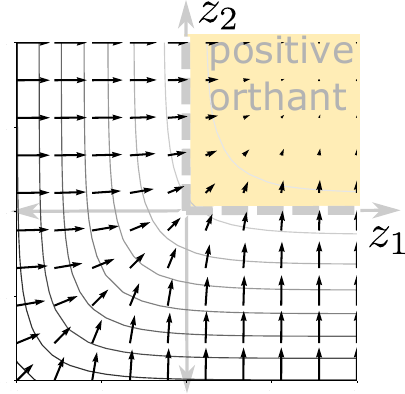}
  \caption{Vector field \(-\nabla_{\TP}(\bfz) \) where \(\TP\) is the template of the cross entropy (Example~\ref{example:cross entropy}).
The contour lines are shaded according to the value of the function \(\TP\) (darker is larger). See Figure~\ref{fig:framework}.
  }\label{fig:vector-field}
\end{wrapfigure}

The condition \(\nabla_{\TP}(\bfz) \prec \vzero\) can be thought of in a similar fashion.
See Figure~\ref{fig:vector-field} for a plot of the vector field \(-\nabla_{\TP}(\bfz)\) where \(\TP \equiv \TP^{\mathsf{CE}}\) is the template of the cross entropy (Example~\ref{example:cross entropy}), and where \(k=3\).
Note that the vectors point toward the shaded region whose interior is the positive orthant in \(\mathbb{R}^{2}\).

  However, Definition~\ref{definition:regular-PERM-loss} requires the strict negativity of the gradient for all of \(\mathbb{R}^{k-1}\), whereas the derivative of \(\TP\) is only required to be negative at \(0\) in the binary case in \citet[Theorem 6]{bartlett2006convexity}.

  Next, we define the ``truncation'' of a PERM loss \(\calL\). Intuitively, if \(k \ge 3\) is the number of classes and \(\calL\) is  \(k\)-ary, then the truncation of \(\calL\) results in a \((k-1)\)-ary loss.

  \begin{proposition}[Truncation]\label{proposition:truncation-of-a-PERM-loss}
    Assume \(k \ge 3\).
  Let \(\calL : \mathbb{R}^{k} \to \mathbb{R}^{k}_{\ge 0}\) be a PERM loss with template \(\TP: \mathbb{R}^{k-1} \to \mathbb{R}\).
  Define
  \(\con[\TP](\bfw) := \lim_{\lambda \to \infty} \TP(\lambda, \bfw)\) for all \(\bfw \in \mathbb{R}^{k-2}\).
  If \(\calL\) is regular
  ({Definition}~\ref{definition:regular-PERM-loss}), then \(\con[\TP]\) is a well-defined  symmetric function \(\mathbb{R}^{k-2} \to \mathbb{R}\) (i.e., all limits exist in \(\mathbb{R}\)), referred to as the \emph{truncation} of \(\TP\).
  \end{proposition}
  \begin{proof}
    Let \(h(\lambda) := \TP(\lambda, \bfw)\).
The condition that \(\nabla_{\TP}(\cdot)  \prec \vzero\) implies that
\(h\)
is  decreasing as a function of \(\lambda\).
Moreover, \(h\) is nonnegative since \(\TP\) is nonnegative.
Thus, \(\lim_{\lambda \to +\infty} h(\lambda)\) exists.
The symmetry of \(\con[\TP]\) follows immediately from the symmetry of \(\TP\).
  \end{proof}
  By Theorem~\ref{theorem:relative-margin-form}, the symmetric function \(\con[\TP]\) induces a unique loss function which we call the \emph{truncation} of \(\calL\).
  Note that if
\(\TP^{\mathsf{CE}}(\bfz) = \log(1 + \exp(-z_{1}) + \exp(-z_{2}))\)
is the template of the cross entropy with \(k=3\) (Example~\ref{example:cross entropy}), then
\(\con[\TP^{\mathsf{CE}}](w) = \lim_{\lambda \to \infty} \TP^{\mathsf{CE}}(\lambda, w) = \log(1+\exp(-w))\), which is just the binary cross entropy/logistic loss.

  Next, we define the \(m\)-fold iterated truncation of the template of a PERM loss.
  \begin{corollary}\label{corollary:iterated-truncation-of-a-PERM-loss}
    For each \(m \in \{0,1,\dots,k-2\}\), define \(\con^{\iter{m}}[\TP]\) to be \(m\)-fold repeated applications of \(\con\) to \(\TP\), i.e.,
    \(
\con^{\iter{m}}[\TP] :=\con[\cdots \con[\con[\TP]] \cdots]
    \)
    where \(\con\) appears \(m\)-times.
  By convention, let \(\con^{\times 0}[\TP]  = \TP\).
  Moreover, for each \(n \in \{2,\dots, k\}\), define the notational shorthand \(\TP^{(n)} := \con^{\times (k-n)}[\TP]\).
  If, for \(n \in \{2,\dots, k-1\}\), \(\TP^{(n+1)} : \mathbb{R}^{n} \to \mathbb{R}\) is a symmetric function such that the associated PERM loss, denoted \(\calL^{(n+1)}\), is a regular PERM loss, then \(\TP^{(n)} : \mathbb{R}^{n-1} \to \mathbb{R}\) is a symmetric function.
  \end{corollary}

The \(n\)-ary truncated loss captures the behavior of \(\TP\) when the {first} \(k-n\) inputs to the template \(\TP\) approach \(+\infty\).
Note that if \(\TP^{\mathsf{CE}}\) is the template of the \(k\)-ary cross entropy, then
\(\con^{\times (k-n)}[\TP^{\mathsf{CE}}]\) is the template of the \(n\)-ary cross entropy.
This follows from a calculation similar to the one right before Corollary~\ref{corollary:iterated-truncation-of-a-PERM-loss}.

Next, we define ``totally regular PERM loss'' which generalizes the above observation:
\begin{definition}[Totally regular PERM loss]\label{definition:totally-regular-PERM-loss}
  Let \(\calL : \mathbb{R}^{k} \to \mathbb{R}^{k}\) be a regular PERM loss with template \(\TP: \mathbb{R}^{k-1} \to \mathbb{R}\).
  For each \(n \in \{2,\dots, k\}\),
  let
  \(    \TP^{(n)} : \mathbb{R}^{n-1} \to \mathbb{R} \)
  be as in Corollary~\ref{corollary:iterated-truncation-of-a-PERM-loss}
  and
let \(\calL^{(n)}\) be the unique\footnote{
The existence and uniqueness is guaranteed by Theorem~\ref{theorem:relative-margin-form}.
} PERM loss associated to \(\TP^{(n)}\).
  We say that \(\calL\) is a {\emph{totally regular}} PERM loss if
  \(\calL^{(n)}\) is regular for each \(n \in \{2,\dots, k\}\).
\end{definition}

Before proving our main result, we note that Gamma Phi losses form a large family of totally regular PERM losses:
\begin{proposition}
  % \label{proposition: sufficient condition for gamma phi loss to be CC}
  \label{proposition:sufficient-condition-for-gamma-phi-losses-to-be-regular}
  Let \(\gamma,\phi : \mathbb{R} \to \mathbb{R}\) be functions and \(\calL^{\gamma,\phi}\) be the associated Gamma-Phi loss as defined  in Example~\ref{example:multiclass-exponential-loss}.
  Suppose that all of the following holds
\begin{compactenum}
  \item $\gamma$ is a twice differentiable function such that
  \(\gamma \ge 0\),
$\tfrac{d\gamma}{dt} > 0$
and
\(\tfrac{d^{2}\gamma}{dt^{2}} \ge 0\) on all of \(\mathbb{R}\),
\item $\phi$ is a twice differentiable function such that \(\tfrac{d\phi}{dt} <0 \), \(\tfrac{d^{2} \phi}{dt^{2}} > 0\) on all of \(\mathbb{R}\).
Moreover, $\phi(t) \to 0$ as $t\to +\infty$.
\end{compactenum}
Then $\calL^{\gamma,\phi}$ is totally regular.
\end{proposition}

We now state our main result:
\begin{theorem}\label{theorem: nested family of regular PERM losses are CC - exposition version}
  If \(\calL\) is totally regular, then
  \(\calL\) is classification-calibrated.
\end{theorem}

  In view of Theorem~\ref{theorem: nested family of regular PERM losses are CC - exposition version},
  the assumptions of Proposition~\ref{proposition:sufficient-condition-for-gamma-phi-losses-to-be-regular}
  turns out to be
  a sufficient condition for a Gamma-Phi loss to be classification-calibrated.
  Theorem~\ref{theorem: nested family of regular PERM losses are CC - exposition version} and Proposition~\ref{proposition:sufficient-condition-for-gamma-phi-losses-to-be-regular} together recover the result that the coherence loss \citep{zhang2009coherence} is classification-calibrated.
  However, this sufficient condition is subsumed by  a previous result in \citet[Theorem 3.3]{wang2023classification}.
  Nevertheless, that
  $\calL^{\gamma,\phi}$ is totally regular has nontrivial consequences as we will see next.

  First, new totally regular PERM losses can be constructed from existing ones:
\begin{proposition}\label{proposition:sum-of-totally-regular-loss-is-totally-regular}
  Let \(\calL\) and \(\calL' : \mathbb{R}^{k} \to \mathbb{R}^{k}\) be totally regular PERM losses and \(\lambda >0\) be a number. Then \(\lambda \calL\) and \(\calL + \calL'\) are also totally regular.
\end{proposition}
Thus,  \(\lambda \calL\) and \(\calL + \calL'\) are both classification-calibrated by Theorem~\ref{theorem: nested family of regular PERM losses are CC - exposition version}.
Next, we apply this result to sums of Gamma-Phi losses:
  \begin{example}[Sum of Gamma-Phi losses]\label{example:sum-of-Gamma-Phi-losses}
  Consider two Gamma-Phi losses each satisfying the assumptions of Proposition~\ref{proposition:sufficient-condition-for-gamma-phi-losses-to-be-regular}. For concreteness, take the cross entropy \(\calL^{\mathsf{CE}}\) and the multiclass exponential loss \(\calL^{\mathsf{Exp}}\) (See Example~\ref{example:multiclass-exponential-loss}).
  Then by Proposition~\ref{proposition:sum-of-totally-regular-loss-is-totally-regular},
  the loss \(\tfrac{1}{2}(\calL^{\mathsf{CE}} + \calL^{\mathsf{Exp}})\)
is totally regular and thus classification-calibrated by Theorem~\ref{theorem: nested family of regular PERM losses are CC - exposition version}.
Note that the sum of two Gamma-Phi losses need not be another Gamma-Phi losses.
Thus, Theorem~\ref{theorem: nested family of regular PERM losses are CC - exposition version} combined with Proposition~\ref{proposition:sum-of-totally-regular-loss-is-totally-regular} gives the first sufficient condition
for when the sum of classification-calibrated losses are again classification-calibrated, in the multiclass case\footnote{Note that the situation is much simpler in the binary case, due to the result of \citet[Theorem 6]{bartlett2006convexity}
  that
convex margin loss \(\TP\) is classification-calibrated if and only if \(\TP\) is differentiable at \(0\) and \(\dot \TP(0) <0\) where the ``overdot'' denotes taking derivative of a univariate function.
}.
\end{example}

In the next section,  we will apply Theorem~\ref{theorem: nested family of regular PERM losses are CC - exposition version} to obtain  a sufficient condition for classification-calibration of Fenchel-Young losses (Theorem~\ref{theorem: sufficient condition for FY loss to be CC - exposition version}).

% \color{red}
% % On the other hand, proof of the analogous result for Gamma-Phi loss, i.e., Theorem~\ref{theorem:Gamma-Phi-loss-is-ISC}, requires a different set of techniques introduced in the following section.
% \input{conditional_risk}
% \color{black}

\section{Fenchel-Young losses: sufficient conditions for classification-calibration}\label{section:FY-loss-main}

In this section, we consider a subset of PERM losses called the Fenchel-Young losses:
\begin{definition}[\citealt{blondel2020learning}]\label{definition:Fenchel-Young}
  Let \(\Omega: \Delta^k \to \mathbb{R}\) be a continuous function and \(\mu \in \mathbb{R}_{\ge 0}\).
  Define \(\bfc_{y}:=\mu(\vone^{(k)} - \bfe_{y}^{(k)})\).
The \emph{Fenchel-Young} loss associated to \(\Omega\) and \(\mu\)
is the loss function \(\calL^{\Omega, \mu}: \mathbb{R}^k \to \mathbb{R}^k_{\ge 0}\) whose \(y\)-th component is given by
\begin{equation}
  \textstyle
  \calL_{y}^{\Omega, \mu}(\bfv)
  :=
  % \Omega^*(\bfv+ \bfc_y) + \Omega(\bfe_y) - \langle \bfv+ \bfc_y, \bfe_y \rangle\\
  % &=
  \max_{\bfp \in \Delta^k}
  - \Omega(\bfp)
  +
  \Omega(\bfe_{y}^{(k)})
  +
\langle \bfv + \bfc_{y}, \bfp -\bfe_{y}^{(k)} \rangle.
\label{equation:Fenchel-Young}
\end{equation}
The maximization in Eqn.~\ref{equation:Fenchel-Young} is essentially the Fenchel  conjugate\footnote{Also known as the convex conjugate.} of \(-\Omega\). See
Definition~\ref{definition:convex conjugate}.
\end{definition}

It is not obvious that Fenchel-Young losses are indeed PERM losses.
  We prove this fact in
{Proposition}~\ref{proposition: FY loss is PERM} in the appendix.

\begin{definition}[Negentropy]\label{definition:negentropy}
  A function \(\Omega: \Delta^k \to \mathbb{R}\) is a \emph{negentropy} if :
  \begin{compactenum}
    \item \(\Omega\) is closed (maps closed sets to closed sets) and convex,
  \item \(\Omega\) is symmetric, i.e., \(\Omega(\mathbf{S}_{\sigma}\bfp) = \Omega(\bfp)\) for all \(\bfp \in \Delta^k\) and \(\sigma \in \mathtt{Sym}(k)\),
    \item \(-\Omega(\bfp) \ge 0\) for all \(\bfp \in \Delta^k\) and \(\Omega(\bfe_{i}^{(k)}) = 0\) for all \(i \in [k]\).
  \end{compactenum}
\end{definition}
The negative entropy is more convenient to work with due to its convexity.
We use the
term ``negentropy'' in the interest of brevity.
This term has been previously used in statistics \citep{hyvarinen2000independent} and in machine learning
\citep{mensch2019geometric}.
\begin{example}\label{example:CE-as-FY-loss}
 When \(\Omega(\bfp) := \sum_{y = 1}^{k} p_{y} \log(p_{y})\)  is the negative Shannon entropy and \(\mu = 0\),  we have that \(\calL^{\Omega,0}\) is the cross entropy
({Example}~\ref{example:cross entropy}).
\end{example}

Below, we use the notation  that
\(\tilde{\bfp} = [p_{1},\dots, p_{k-1}]^{\top}\in \mathbb{R}^{k-1}\) is the sub-vector of \(\bfp\) with the last element dropped.
Later in Proposition~\ref{proposition: FY loss is PERM}, we show that the Fenchel-Young loss Eqn.~\eqref{equation:Fenchel-Young} is a PERM loss with template
\begin{equation}
  \textstyle
\TP^{\Omega,\mu}(\bfz)
:=
\max_{\bfp \in {\Delta}^k}
-\Omega(\bfp)
+
\mu \vone^{\top} \tilde{\bfp}
- \langle \tilde{\bfp}, \bfz\rangle.
\label{equation:FY-template-simple}
\end{equation}

\begin{remark}
  \citet{blondel2020learning} allow the vector \(\bfc_{y} \in \mathbb{R}^{k}\)
  to be arbitrary, in which case the resulting loss is known as \emph{cost-sensitive Fenchel-Young} loss.
  However, known calibration results~\citep{blondel2019structured,nowak2019general} are only for \(\bfc_{y}\) as in Definition~\ref{definition:Fenchel-Young}. \end{remark}

In order to state our sufficient condition for classification-calibration (Theorem~\ref{theorem: sufficient condition for FY loss to be CC - exposition version}),
we briefly review two key notions from the theory of convex analysis and Legendre transformation following \citet[Section 26]{rockafellar1970convex}:

\begin{definition}[Convex conjugates]\label{definition:convex conjugate}
  Let \(D \subseteq \mathbb{R}^n\) be a closed convex set.
  Let $f: D \to \mathbb{R}$ be a function.
  Define $D^* := \{y \in \mathbb{R}^{n}: \sup_{x \in D} \langle y, x \rangle - f(x) < \infty \}$.
  The \emph{convex conjugate} of a function $f : D \to \mathbb{R}$ is the function
  $f^*: D^* \to \mathbb{R}$ given by
  \(    f^*(y) = \sup_{x \in D} \langle y, x \rangle - f(x)
\).
\end{definition}

\begin{definition}[Convex functions of Legendre type]\label{definition: Legendre type}
  Let $D \subseteq \mathbb{R}^n$ be a closed convex set.
  A convex function $f: D \to \mathbb{R}$ is said to be of \emph{Legendre type} if
  \begin{compactenum}
    \item\label{definition: Legendre type - domain} $C:=\mathrm{int}(D)$ is an open convex subset of $\mathbb{R}^{n}$,
    \item\label{definition: Legendre type - function} $f$ is strictly convex and differentiable on $C$,
    \item\label{definition: Legendre type - boundary}
    for all $\{x_i\}_{i=1}^{\infty} \subseteq C$ with $\lim_{i \to \infty} x_i \in \mathtt{bdry}(D)$, we have
    $\lim_{i\to \infty} \|\nabla_f(x_i)\| = +\infty$.
  \end{compactenum}
\end{definition}

For example, when $D = \Delta^k$ and $f = -H$ is the negative Shannon entropy, then $f: D \to \mathbb{R}_{\le 0}$ is of Legendre type. See the paragraph immediately following \citet[Definition 3]{blondel2020learning}.

\begin{definition}[Regular negentropies]\label{definition:regular-entropy}
  A negentropy $\Omega: \Delta^k \to \mathbb{R}$ is a \emph{regular negentropy} if
  ${\Omega}$ is of Lengedre type and twice differentiable.
\end{definition}
The term ``negentropy'' was previously used by \citet{mensch2019geometric}.
To the best of our knowledge, the term ``\emph{regular} negentropy'' is new.

\begin{proposition}[Truncation of a negentropy]\label{proposition:truncation-of-a-negentropy}
 Assume \(k \ge 3\).
 Let \(\inj: \Delta^{k-1} \to \Delta^{k}\)
 be the ``zero-padding'' operator, i.e.,
\(  \inj(\bfq)
  := [0,q_{1},\dots, q_{k-1}] \in \Delta^{k}
\)
 for all \(\bfq = [q_{1},\dots, q_{k-1}] \in  \Delta^{k-1}\).
  Let \(\Omega : \Delta^{k} \to \mathbb{R}\) be a negenetropy.
  Define the \emph{truncation} of \(\Omega\), denoted \(\con[\Omega] : \Delta^{k-1} \to \mathbb{R}\), by
  \(\con[\Omega](\cdot) := \Omega(\inj(\cdot))\).
  Then \(\con[\Omega]\) is a
  negentropy.
  \end{proposition}
  The following is analogous to the earlier iterated truncation construction in Corollary~\ref{corollary:iterated-truncation-of-a-PERM-loss}:

  % \begin{corollary}\label{corollary:iterated-truncation-of-a-negentropy}
%     In the situation of Proposition~\ref{proposition:truncation-of-a-negentropy},
%     define \(\con^{\iter{m}}[\Omega]\) to be \(m\)-fold repeated applications of \(\con\) to \(\Omega\), i.e.,
%     \(
% \con^{\iter{m}}[\Omega] :=\con[\cdots \con[\con[\Omega]] \cdots]
%     \)
%     where \(\con\) appears \(m\)-times.
%   Moreover, for each \(n \in \{2,\dots, k\}\), define the notational shorthand \(\Omega^{(n)} := \con^{\times (k-n)}[\Omega]\).
%   Then \(\Omega^{(n)} : \Delta^{n} \to \mathbb{R}\) is a negenetropy.
%   \end{corollary}
\begin{corollary}\label{corollary:iterated-truncation-of-a-negenetropy}
    For each \(m \in \{0,1,\dots,k-2\}\), define \(\con^{\iter{m}}[\Omega]\) to be \(m\)-fold repeated applications of \(\con\) to \(\Omega\), i.e.,
    \(
\con^{\iter{m}}[\Omega] :=\con[\cdots \con[\con[\Omega]] \cdots]
    \)
    where \(\con\) appears \(m\)-times.
  By convention, let \(\con^{\times 0}[\Omega]  = \Omega\).
  Moreover, for each \(n \in \{2,\dots, k\}\), define the notational shorthand \(\Omega^{(n)} := \con^{\times (k-n)}[\Omega]\).
  Then \(\Omega^{(n)} : \Delta^{n} \to \mathbb{R}\) is a negentropy.
  \end{corollary}

\begin{definition}[Totally regular negentropy]\label{definition:totally-regular-negentropy}
  Let \({\Omega} : {\Delta}^{k} \to \mathbb{R}\) be a negentropy.
  We say that $\Omega$ is a {\emph{totally regular}} negentropy if
  \(\Omega^{(n)}\) is a regular negentropy for each \(n \in \{2,\dots, k\}\).
\end{definition}

To the best of our knowledge, the definition of a totally regular negentropy is new.
As suggested by the name,
totally regular negentropies induce
Fenchel-Young losses that are totally regular.
This is a critical step\footnote{
  See the proof of
  {Theorem}~\ref{proposition: nested Delta family is nested PERM} in the appendix.
}
towards proving the main result of this section below.
The result leverages Theorem~\ref{theorem: nested family of regular PERM losses are CC - exposition version}
to
establish a sufficient condition for the classification-calibration of Fenchel-Young losses:

\begin{theorem}\label{theorem: sufficient condition for FY loss to be CC - exposition version}
  
 Let \(\Omega\) be a totally regular negentropy,
  \(\mu \in \mathbb{R}_{\ge 0}\) be fixed,
  and \(\calL^{\Omega,\mu}\) be the Fenchel-Young loss associated to \(\Omega\) and \(\mu\).
Then
\(\calL^{\Omega,\mu}\) is classification-calibrated.

\end{theorem}

In light of Theorem~\ref{theorem: multiclass classification calibration},
if \(\Omega\) satisfies the conditions of Theorem~\ref{theorem: sufficient condition for FY loss to be CC - exposition version}, then
{\(\calL^{\Omega,\mu}\) satisfies the consistency transfer property ({Definition}~\ref{definition:consistency-transfer-property}).
  Theorem~\ref{theorem: sufficient condition for FY loss to be CC - exposition version} recovers all known classification-calibration sufficient conditions~\citep{nowak2019general,blondel2019structured} which require that \(\Omega\) be \emph{strongly}-convex whereas our result only requires \emph{strict} convexity.
  The following proposition and example show
  that
  there is always a nonempty set of losses to which the result applies but the prior results don't.
  % The following proposition and example show that Fenchel-Young losses satisfying our more general sufficient condition --- but not the previous ones --- are ``abundant'' in the following sense: Given any strongly-convex negentropy, one can construct a family of strictly- but not strongly-convex negentropies.

\begin{proposition}\label{proposition:non-strongly-convex-negentropy}
Let \(\Omega : \Delta^{k} \to \mathbb{R}\) be
a totally regular negentropy and
\(\sqfunc : \mathbb{R}_{\ge 0} \to \mathbb{R}_{\ge 0}\)
be convex, twice differentiable and strictly increasing.
  Let \(\bfu := (1/k) \vone^{(k)}\in \Delta^{k}\) be the uniform probability vector.
Define \(\Theta : \Delta^{k} \to \mathbb{R}\) by
\(  \Theta(\bfp) := \sqfunc({\Omega}(\bfp) - {\Omega}(\bfu)) - \sqfunc(-{\Omega}(\bfu))
 \) for \( \bfp \in \Delta^{k}
\).
Then \(\Theta\) is a totally regular negentropy.
 Furthermore, if $g'(0) = 0$, then $\Theta$ is not strongly-convex.
\end{proposition}

\begin{example}\label{example:not-strongly-convex-totally-regular-negentropy}
For a concrete application of {Proposition}~\ref{proposition:non-strongly-convex-negentropy}, take \(\Omega\) to be the negative Shannon entropy (Example~\ref{example:CE-as-FY-loss}) and $g(x) = x^{2}$ the square function.
Then \(\Theta\) as defined in {Proposition}~\ref{proposition:non-strongly-convex-negentropy}
is a totally regular negenetropy that is not strongly-convex.
Moreover, \(\calL^{\Theta,\mu}\) is classification-calibration for any \(\mu \in \mathbb{R}_{\ge 0}\)
by
{Theorem}~\ref{theorem: sufficient condition for FY loss to be CC - exposition version}.
\end{example}
Note that previous sufficient conditions
\citep{nowak2019general,blondel2019structured} do not cover our example above.

\section{Discussion and Future Directions}

We proposed the relative margin form as a multiclass extension to the popular (binary) margin loss framework.
We characterized the set of losses that can be expressed in the relative margin form --- the PERM losses.
Our Theorem~\ref{theorem: nested family of regular PERM losses are CC - exposition version} generalizes to the multiclass case the seminal result of \cite[Theorem 6]{bartlett2006convexity}.
Central to our analysis is the matrix label code ({Definition}~\ref{definition:matrix-label-code}) which we expect to be useful beyond the scope here.
Utilizing our framework, we extended the existing set of known sufficient conditions for the classification-calibration of Fenchel-Young losses.
We expect the framework to be useful for research on multiclass classification, e.g., for \emph{\(\mathcal{H}\)-consistency}  \citep{long2013consistency,zhang2020bayes,awasthi2022multi}.

\myheader{Characterization of classification-calibration}
 One weakness of our work is that, compared to the characterization of CC in the binary case \citep[Theorem 6]{bartlett2006convexity}, we are only able to prove an analogous result (Theorem~\ref{theorem: nested family of regular PERM losses are CC - exposition version}) in the multiclass case under additional assumptions.
  Future work will investigate whether this can be improved.

\myheader{Cost-weighted classification}
Our results are limited to the traditional notion of classification-calibration under the assumption that the cost of misclassifying all classes are the same.
The more general notion of \(\mathbf{c}\)-classification-calibration described in \cite{williamson2016composite} allows the cost of misclassification to be class-dependent according to some class-specific weight vector \(\mathbf{c}\).
An interesting future direction is to consider
a setting where a general \(\mathbf{c}\)-classification-calibration sufficient condition similar to our
Theorem~\ref{theorem: nested family of regular PERM losses are CC - exposition version} can be derived for non-symmetric losses.
More specifically, given a classification-calibrated PERM loss, is there a principled way to modify the loss to make it \(\mathbf{c}\)-classification calibrated?

\myheader{Combining losses}  \cite{hui2023cut} demonstrate the effectiveness of summing existing losses (the cross entropy and the square loss) to design new losses for classification.
Unfortunately, our result in Proposition~\ref{proposition:sum-of-totally-regular-loss-is-totally-regular}
does not apply to their new loss since the square loss is not totally regular.
One approach may be to connect our technique with that of \citet[\S 6]{williamson2023geometry} for studying the geometric properties of sums of losses.

% \noindent\textbf{\(\mathcal{H}\)-consistency}.

\myheader{Beyond classification}
While this work is concerned with classification-calibration, there are many works on calibration for  more general discrete supervised learning tasks.
% \citet{steinwart2007compare} introduced the extension of loss calibration-theory to cost-sensitive classification, regression and unsupervised learning tasks such as density estimation.
\citet{ramaswamy2016convex} developed theory for multiclass classification with abstain option and, more generally, losses defined over finite sets i.e., discrete losses.
\citet{finocchiaro2019embedding} showed that there exists polyhedral losses that are calibrated with respect to arbitrary discrete losses.
An interesting question is if the PERM loss framework can useful for analyzing such surrogate losses.
Beyond classification-calibration, a potentially fruit direction is to employ the PERM loss framework to study optimization and learning theory.

\section*{Acknowledgments}
The authors were supported in part by the National Science Foundation under awards 1838179 and 2008074, and by the Department of Defense, Defense Threat Reduction Agency under award HDTRA1-20-2-0002.
YW is supported by the Eric and Wendy Schmidt AI in Science Postdoctoral Fellowship, a Schmidt Futures program.

\bibliography{references}

\appendix

\section{Notational conventions and definitions}\label{section:appendix:full-notations}
% Let $\Delta^k_{\mathtt{desc}} = \{ \bfp \in \Delta^k : p_1 \ge \cdots \ge p_k\}$.
% For $i,j \in [k]$, let $\tau_{ij} \in \mathtt{Sym}(k)$ be the element that swaps $i$ and $j$.

For the reader's convenience, we tabulate the notational convention used in this work in Table~\ref{table:conventions}.
Moreover, we tabulate the mathematical objects defined in Table~\ref{table:definitions}.
Finally, we require the following additional notations:

\noindent\textbf{Permutations}.
A bijection from $[k]$ to itself is called a \emph{permutation} (on $[k]$).
Recall we denote by $\mathtt{Sym}(k)$ the set of all permutations on $[k]$.
Two permutations \(\sigma,\sigma' \in \mathtt{Sym}(k)\) can be composed resulting in another permtuation \(\sigma \sigma' \in \mathtt{Sym}(k)\) defined by function composition: for \(y \in [k]\),  we have \(\sigma \sigma'(y) := \sigma(\sigma'(y))\).
For $i,j \in [k]$, let $\tsp_{(i,j)} \in \mathtt{Sym}(k)$ denote the \emph{transposition} which swaps $i$ and $j$, leaving all other elements unchanged.
More precisely, $\tsp_{(i,j)}(i) = j$, $\tsp_{(i,j)}(j) = i$ and $\tsp_{(i,j)}(y) = y$ for $y \in [k] \setminus \{i,j\}$.
Define the notational shorthand $\tsp_{i} := \tsp_{(k,i)}$, the transposition that swaps $k$ and $i$.

\noindent\textbf{Permutation matrices}.
Recall that for each $\sigma \in \mathtt{Sym}(k)$, $\mathbf{S}_{\sigma} \in \mathbb{R}^{k\times k}$ denotes the permutation matrix corresponding to $\sigma$.
In other words, if $\bfv \in \mathbb{R}^{k}$ is a vector, then $[\mathbf{S}_{\sigma} \bfv]_{j} = [ \bfv ]_{\sigma(j)}= v_{\sigma(j)}$.
Note that if $\sigma, \sigma' \in \mathtt{Sym}(k)$, then $\mathbf{S}_{\sigma \sigma'} = \mathbf{S}_{\sigma'} \mathbf{S}_{\sigma}$ where the order of compositions are reversed (Lemma~\ref{lemma:matrix-permutation-is-order-reversing}).
Define the notational shorthand $\Tsp_{(i,j)} := \mathbf{S}_{\tsp_{(i,j)}}$ the matrix corresponding to the transposition of $i$ and $j$.
Likewise, define $\Tsp_{i} := \Tsp_{(k,i)}$.
When \(k\) is ambiguous, we disambiguate it with the superscript notation and write \(\mathbf{S}_{(i,j)}^{(k)}\) and \(\mathbf{T}_{(i,j)}^{(k)}\).

\noindent\textbf{Topology}.
Let $S$ be a subset of a topological space. Let $\interior(S)$ and $\mathtt{bdry}(S)$ denote the interior and the boundary of the set $S$, respectively.
See Table~\ref{table:conventions} for the full list of symbols.

\begin{table}[ht]
  \centering
\begin{tabular}{@{}lll@{}}
\toprule
Mathematical object   & Description of notation                     & Example of notation                  \\ \midrule\midrule
Set of all permutations           & &\(\mathtt{Sym}(k)\)            \\
Permutations           & Lower case sigma or tau      & $\sigma, \tau$           \\
  Transpositions        & tau with subscripts & $\tau_{(i,j)}$           \\
  \midrule
Vector                & Bold lower case              & $\bfv,\bfw$             \\
Entries of vector     & Normal font lower case       & $v_1,v_{2},\dots$                   \\
  \midrule
Special vector                & Blackboard font &\\
All zeros/ones vector in $\mathbb{R}^{n}$ &                              & $\vzero^{(n)}, \vone^{(n)}$ \\
$i$-th elem.\ basis vector in $\mathbb{R}^{n}$ &                              & $\bfe^{(n)}_{i}$ \\
  \midrule
  Matrix                & Bold upper case              & $\mathbf{A}$             \\
  $j$-th Column & & $[ \mathbf{A} ]_{:j}$             \\
  $n\times n$ Identity                & & \(\mathbf{I}_{n}\)             \\
  % \midrule
  % Loss & Caligraphic font & $\calL$ &
  %                                     Definition~\ref{definition:multiclass-loss-functions}
  %                                     \\
  % Reduced loss & & $\LL$ &
  %                                     Definition~\ref{definition:PERM-loss}
  %                                     \\
  \bottomrule
\end{tabular}
\caption{Notation convention used throughout this work.} \label{table:conventions}
\end{table}

\begin{table}[ht]
  \centering
\begin{tabular}{@{}lll@{}}
\toprule
Symbol & Description &Defined in...\\ \midrule\midrule
  \(\calL\) & Multiclass loss function &
{Definition}~\ref{definition:PERM-loss}
  \\
  \(\LL\) & Reduced form of \(\calL\)&
{Definition}~\ref{definition:PERM-loss} item~\ref{definition:PERM-loss-reduced-form} \\
  \(\TP\) & Template of \(\calL\)&
{Definition}~\ref{definition:PERM-loss} item~\ref{definition:PERM-loss-template} \\
  \(\Omega\) & Negentropy  & Definition~\ref{definition:negentropy} \\
  \(\calL^{\Omega,\mu}\) & Fenchel-Young loss assoc.\ to \(\Omega,\mu\)&
{Definition}~\ref{definition:Fenchel-Young}
  \\
  \(\con[\TP]\)& truncation  of \(\TP\)& Proposition~\ref{proposition:truncation-of-a-PERM-loss} \\
%   \(\con[\calL]\)&&
% {Definition}~\ref{definition:totally-regular-PERM-loss}
%   \\
  \(\TP^{(n)}\) & \(\con\) applied \(k-n\) times to \(\TP\)&
{Corollary}~\ref{corollary:iterated-truncation-of-a-PERM-loss}
  \\
\(\calL^{(n)}\)& PERM loss corresponding to \(\TP^{(n)}\)&
{Definition}~\ref{definition:totally-regular-PERM-loss}
  \\
  \(\con[\LL]\)&truncation of \(\LL\) &
{Definition}~\ref{definition:truncation-of-the-reduced-form}
  \\
  \(\inj\)&
 ``zero-padding'' operator
                     &
{Proposition}~\ref{proposition:truncation-of-a-negentropy}
  \\
  \(\con[\Omega]\)& truncation of \(\Omega\) &
{Proposition}~\ref{proposition:truncation-of-a-negentropy}
  \\
  \(\Omega^{(n)}\) & \(\con\) applied \(k-n\) times to \(\Omega\) &
{Corollary}~\ref{corollary:iterated-truncation-of-a-negenetropy}
  \\
  \(\bmpi\) &relative margins conversion matrix&
{Definition}~\ref{definition:relative-marginalization-mapping}
  \\
  \(\bmPi_{\sigma}\) & \(\bmpi \mathbf{S}_{\sigma} \bmpi^{\dagger}\)&
{Definition}~\ref{definition:irreducibility-representation}
  \\
  \(\ico\) & matrix label code &
{Definition}~\ref{definition:matrix-label-code}
  \\
  \(\Tsp\) &Transposition matrix&
Section~\ref{section:appendix:full-notations}
\\
  \(\tsp\) &Transposition permutation&
Section~\ref{section:appendix:full-notations}
              \\
  \(\bfA(\bfz)\) & gradient of \(\LL(\bfz)\)&
{Lemma}~\ref{lemma:Az-matrix-is-nonsing-M-matrix}
  \\
  \(\drop\) (resp.\ \(\proj\)) & drop last (resp.\ first) coordinate &
{Lemma}~\ref{lemma:interior-projection-surjectivity}
                                                     (resp.\ ~\ref{corollary:interior-projection-surjectivity})
                                  \\
  \(\ran(\cdot)\) and \(\cran(\cdot)\) & range of convex hull of range &
{Definition}~\ref{definition:loss-surface}
  \\
        \(\cran^{\bullet}(f)\)
        and
  \(\cran^{\circ}(f)\)
       &closure and interior of \(\cran(f)\)
                     &
     {Definition}~\ref{definition: loss surface - alternative characterization}
     \\
  \(    \mathcal{N}(\bmzeta; S)\)
       &Positive normals&
  {Definition}~\ref{definition:admissible-sets}
  \\
     \bottomrule
\end{tabular}
\caption{Mathematical objects and where they are defined.} \label{table:definitions}
\end{table}

\section{Properties of the matrix label code}\label{section:appendix:matrix-label-code}

As mentioned in the main text, the definition of matrix label code has already been introduced in \citet[Definition S3.14]{wang2020weston}.
However, that work, the matrix label code is used to analyze a particular optimization problem resulting from the Weston-Watkins SVM \citep{weston1998multi}.
Here, we develop a comprehensive theory for using matrix label code with PERM loss.
\iftoggle{arxiv}{}{
\textcolor{black}{
  With the exception of
{Lemma}~\ref{lemma:subtraction-matrix-label-code},
  all results in this section  have already appeared in  \citet{wang2020weston,wang2021exact} using slightly different notation.
  For the reader's convenience, the proofs are omitted here but  are all included in the arXiv version of this manuscript \citep{wang2023unified}.
}
}

    First, we note that
for $y \in \{1,\dots, k-1\}$, the matrix $\ico^{(k)}_y$ acts on
 a vector $\bfz = (z_1,\dots, z_{k-1})^{\top} \in \mathbb{R}^{k-1}$ by
  \begin{equation}\label{equation: action of rho i plus 1}
    \forall j \in [k-1], \,
    [\ico^{(k)}_{y}\bfz]_{j} :=
  \begin{cases}
    z_j - z_{y} &: j \ne y\\
    -z_{y} &: j = y.
  \end{cases}
\end{equation}
\begin{remark}
  Throughout this section, we consider \(k\) fixed and write \(\ico_{1},\cdots, \ico_{k}\) instead of \(\ico_{1}^{(k)},\cdots, \ico_{k}^{(k)}\).
Moreover, whenever \(\mathbf{T}_{y}\) does not have a superscript, we implicitly assume that \(\mathbf{T}_{y} = \mathbf{T}_{y}^{(k)}\).
  In some results, we will work with \((k-1)\times (k-1)\) permutation matrices in which case we will write, for example,
  \(\mathbf{T}_{(y,y')}^{(k-1)}\) where \(y,y' \in [k-1]\).
\end{remark}
\begin{lemma}\label{lemma: rho pi sigma relation}
For all $y \in [k]$, we have
\(\bmpi\Tsp_{y} = \ico_{y} \bmpi.\)
    In particular, for all $i > 1$ and $j \in [k-1]$, we have
    \[
  [\ico_y \bmpi \bfv]_{j}
  =
  \begin{cases}
    v_y - v_{j} &: y \ne j
    \\
    v_y - v_k &: y = j.
  \end{cases}
    \]
\end{lemma}

\iftoggle{arxiv}{\begin{proof}
If $y = k$, then $\Tsp_y$ and $\ico_y$ are both identity matrices and there is nothing to show. Otherwise, suppose that $y < k$.
Consider $\bfv \in \mathbb{R}^k$.
We first calculate $\bmpi\Tsp_{y}\bfv$.
For each $j \in [k-1]$, we have
\begin{equation}
  \label{equation: commuting relations step 1}
  [\bmpi \Tsp_y \bfv]_{j}
  =
  [\Tsp_y\bfv]_k -
  [\Tsp_y\bfv]_{j}
  =v_y - v_{\tsp_y(j)}
  =
  \begin{cases}
    v_y - v_{j} &: y \ne j
    \\
    v_y - v_k &: y = j.
  \end{cases}
\end{equation}
Note that the first equality follows from the definition of \(\bmpi\) (Definition~\ref{definition:relative-marginalization-mapping}).
Next, we compute
    $\ico_{y} \bmpi \bfv$.
    Using Eqn.~\eqref{equation: action of rho i plus 1}, we have for each $j \in [k-1]$ that
\[
  [\ico_y \bmpi \bfv]_j =
  \begin{cases}
    [\bmpi \bfv]_j - [\bmpi \bfv]_{y} &: j \ne y \\
    -[\bmpi \bfv]_{y} &: j =y.
  \end{cases}
\]
For the $j \ne y$ case, we have
\(    [\bmpi \bfv]_j - [\bmpi \bfv]_{y}
    =
    (v_k - v_{j} )
    -
    (v_k - v_{y})
    =
    v_y - v_{j}
\).
For the $y = j$ case, we have
\(    -[\bmpi \bfv]_{y}
    = -(v_k - v_{y})
    =v_y - v_k
\).
  Thus,
  $[\bmpi \Tsp_y \bfv]_{j}
  =
  [\ico_y \bmpi \bfv]_j
  $
  for all $j$, which implies that
  $\bmpi \Tsp_y \bfv = \ico_y \bmpi \bfv $. Since $\bfv$ was arbitrary, we have $\bmpi \Tsp_y = \ico_y \bmpi$.
\end{proof}}{}

\begin{remark}\label{remark:pseudo-inverse-injectivity}
Let \(\bmpi^\dagger\) denote the Moore-Penrose inverse of \(\bmpi\).
Since \(\bmpi\) contains a copy of the identity matrix, \(\bmpi\) has full column rank. Hence, \(\bmpi \bmpi^\dagger\) is the identity.
\end{remark}
\begin{definition}\label{definition:irreducibility-representation}
  For each \(\sigma \in \mathtt{Sym}(k)\), define the \((k-1)\times (k-1)\) matrix \(\bmPi_{\sigma} := \bmpi \mathbf{S}_{\sigma} \bmpi^{\dagger}\).
\end{definition}
% Define a mapping $\bmPi : \mathtt{Sym}(k) \to \mathbb{R}^{(k-1)\times (k-1)}$ by
% \[
%   \bmPi(\sigma) := \bmpi \mathbf{S}_{\sigma} \bmpi^\dagger.
% \]
\begin{lemma}\label{lemma: rho i is Pi sigma i}
  For all $y \in [k]$, we have $\bmPi_{\tsp_y} = \ico_y$.
\end{lemma}

\iftoggle{arxiv}
{
  \begin{proof}
    From Lemma~\ref{lemma: rho pi sigma relation}, we have $\ico_y \bmpi = \bmpi \Tsp_y$ which implies that
    $
    \ico_y = \ico_y \bmpi \bmpi^{\dagger} = \bmpi \Tsp_y \bmpi^{\dagger} = \bmPi_{\tsp_y}
    $.
  \end{proof}
}{}

\begin{lemma}\label{lemma:matrix-permutation-is-order-reversing}
If \(\sigma,\sigma' \in \mathtt{Sym}(k)\), then \(\mathbf{S}_{\sigma\sigma'} = \mathbf{S}_{\sigma'} \mathbf{S}_{\sigma}\).
\end{lemma}
\iftoggle{arxiv}{\begin{proof}
  Let \(\bfv \in \mathbb{R}^{k}\) and \(y \in [k]\) be arbitrary.
  By definition, we have  \([\mathbf{S}_{\sigma \sigma'}\bfv]_{y} = v_{\sigma(\sigma'(y))} \).
  On the other hand,
  \([\mathbf{S}_{\sigma'} (\mathbf{S}_{\sigma}\bfv)]_{y} =
[ \mathbf{S}_{\sigma}\bfv]_{\sigma'(y)}
=
[ \bfv]_{\sigma(\sigma'(y))}
=
v_{\sigma(\sigma'(y))}
  \). Thus, we are done.
\end{proof}}{}

\begin{lemma}\label{lemma: Pi is a group hom}
  For all \(\sigma, \sigma' \in \mathtt{Sym}(k)\), we have \(\bmPi_{\sigma \sigma'} = \bmPi_{\sigma'} \bmPi_{\sigma}\).
\end{lemma}
\iftoggle{arxiv}{\begin{proof}
  Note that \(\bmPi_{\sigma \sigma'} =
\bmpi \mathbf{S}_{\sigma \sigma'} \bmpi^{\dagger}
=
\bmpi \mathbf{S}_{\sigma'} \mathbf{S}_{\sigma} \bmpi^{\dagger}
  \)
  by Lemma~\ref{lemma:matrix-permutation-is-order-reversing}.
Thus, it suffices to show
 \(
    \bmpi \mathbf{S}_{\sigma'} \mathbf{S}_{\sigma} \bmpi^\dagger = \bmpi \mathbf{S}_{ \sigma' } \bmpi^\dagger \bmpi \mathbf{S}_{\sigma} \bmpi^\dagger
 \).
 First, by
 the definition of \(\bmpi\)
the nullspace of \(\bmpi\) is
\(
    \ker (\bmpi) = \{ \bfv \in \mathbb{R}^{k} : v_1 = v_2 = \cdots = v_k\}
\).
Next, let \(\mathcal{V} \subseteq \mathbb{R}^k\) be the subspace
  \(\{\bfv \in \mathbb{R}^k : v_1 + \cdots + v_k = 0\}
\).
Then
    \(\{ \bfv \in \mathbb{R}^{k} : v_1 = v_2 = \cdots = v_k\} = \mathcal{V}^{\perp}\) is the subspace of vectors orthogonal to \(\mathcal{V}\).
    A fundamental result in linear algebra states that
    \(      \ran(\bmpi)^{\dagger} = \ran(\bmpi)^{\top} = (\ker (\bmpi))^{\perp}.
\)
    Thus, $\ran(\bmpi)^{\dagger} = \mathcal{V}$.
    Taken together, if we let $\bfP := \bmpi^\dagger \bmpi \in \mathbb{R}^{k\times k}$, then $\bfP$ is a projection matrix on $\mathcal{V}$.
  Thus, $\bfP\bfv = \bfv$ for all $\bfv \in \mathcal{V}$.
  Since \(\mathbf{S}_{\sigma}\) is a permutation matrix, we have \(\mathbf{S}_{\sigma}(\mathcal{V}) \subseteq \mathcal{V}\).
  Thus, $\ran ( \mathbf{S}_{\sigma} \bmpi^{\dagger} ) \subseteq \mathcal{V}$ which implies that
  $\mathbf{P}  \mathbf{S}_{\sigma} \bmpi^{\dagger}  =  \mathbf{S}_{\sigma} \bmpi^{\dagger}$.
  This proves
 \(
 \bmpi \mathbf{S}_{\sigma'} \mathbf{S}_{\sigma} \bmpi^\dagger =
 \bmpi \mathbf{S}_{ \sigma' } \mathbf{P} \mathbf{S}_{\sigma} \bmpi^\dagger=
 \bmpi \mathbf{S}_{ \sigma' } \bmpi^\dagger \bmpi \mathbf{S}_{\sigma} \bmpi^\dagger
    \)
    as desired.
\end{proof}}{}

\begin{lemma}\label{lemma: rho is involutional}
  For all $y \in [k]$, $\ico_y^2$ is the identity.
\end{lemma}
\iftoggle{arxiv}{\begin{proof}
  By Lemma \ref{lemma: rho i is Pi sigma i} and Lemma \ref{lemma: Pi is a group hom}, we have
  \(    \ico_y^2 = \bmPi_{\tsp_y} \bmPi_{\tsp_y} = \bmPi_{\tsp_y^2}
\).
  Since $\tsp_y$ is a transposition, $\tsp_y^2$ is the identity. Thus, $\ico_y^2$ is also the identity.
\end{proof}}{}
\begin{lemma}\label{lemma: transpositional identity}
  Let $y_1,y_2 \in [k-1]$ be distinct, i.e., \(y_{1}\ne y_{2}\). Then
  $\tsp_{y_1} \tsp_{y_2} \tsp_{y_1} = \tsp_{(y_1,y_2)}$, as elements of \(\mathtt{Sym}(k)\). Moreover,
  $\Tsp_{y_1}\Tsp_{y_{2}}\Tsp_{y_{1}} = \Tsp_{(y_{1},y_{2})}$, as elements of \(\mathbb{R}^{k\times k}\).
\end{lemma}
\iftoggle{arxiv}{\begin{proof}
  This is simply an exhaustive case-by-case proof over all inputs $j \in [k]$.
  First, let $j=k$. Then $\tsp_{(y_1,y_2)}(k) = k$ since $k \not \in \{y_1,y_2\}$. On the other hand
  $
    \tsp_{y_1} \tsp_{y_2} \tsp_{y_1} (k)
    =
    \tsp_{y_1} \tsp_{y_2} (y_1)
    =
    \tsp_{y_1} (y_1)
    =k.
  $
  Now, let $j \in [k-1]$.  If $j \not \in \{y_1,y_2\}$, then $\tsp_{(y_1,y_2)}(j) = j$ and
  $
    \tsp_{y_1} \tsp_{y_2} \tsp_{y_1} (j)
    =
    \tsp_{y_1} \tsp_{y_2} (j)
    =
    \tsp_{y_1} (j)
    =j.
  $
  If $j = y_1$, then $\tsp_{(y_1,y_2)}(y_1) = y_2$ and
  $
    \tsp_{y_1} \tsp_{y_2} \tsp_{y_1} (y_1)
    =
    \tsp_{y_1} \tsp_{y_2} (1)
    =
    \tsp_{y_1} (y_2)
    =y_2.
  $
  If $j = y_2$, then $\tsp_{(y_1,y_2)}(y_2) = y_1$ and
  $
    \tsp_{y_1} \tsp_{y_2} \tsp_{y_1} (y_2)
    =
    \tsp_{y_1} \tsp_{y_2} (y_2)
    =
    \tsp_{y_1} (1)
    =y_1.
  $
  The ``Moreover'' part
follows analogously.
\end{proof}}{}
\begin{corollary}\label{corollary: generator}
  Every $\sigma \in \mathtt{Sym}(k)$ can be written as a product $\sigma = \tsp_{y_1} \tsp_{y_2} \cdots  \tsp_{y_n}$ for some integer \(n\ge 0\) and \(y_{i} \in [k-1]\) for each \(i \in [n]\).
\end{corollary}
\iftoggle{arxiv}{\begin{proof}
  A standard result in group theory\footnote{See \citet[Prop.\ 2.35]{rotman2006first}, for instance.} states that the set of \emph{all} transpositions \(\mathcal{T}\) generates \(\mathtt{Sym}(k)\).
  We prove that the proper {sub}set \(\mathcal{T}' := \{\tsp_y: y \in [k-1]\} \subset \mathcal{T}\) generates the group \(\mathtt{Sym}(k)\).
  By Lemma~\ref{lemma: transpositional identity}, transpositions between labels in \([k-1]\) can be generated by \(\mathcal{T}'\).
  Furthermore, \(\tsp_y = \tsp_{(k,y)}\) by definition, so transposition between \(k\) and elements of \([k-1]\) can be generated by \(\mathcal{T}'\) as well.
  Hence, all of \(\mathcal{T}\) can be generated by \(\mathcal{T}'\).
\end{proof}}{}

% \begin{lemma}\label{lemma: Pi commutes with pi}
%   For all $\sigma \in \mathtt{Sym}(k)$, we have
% $
%   \bmPi(\sigma) \bmpi = \bmpi \mathbf{S}_{\sigma}.
% $
% \end{lemma}
% \begin{proof}
%   By Corollary \ref{corollary: generator} and Lemma \ref{lemma: Pi is a group hom}, it suffices to show that $\bmPi(\tsp_i) \bmpi = \bmpi \Tsp_i$.
%   To see this, recall that Lemma \ref{lemma: rho i is Pi sigma i} says that $\bmPi(\tsp_i) = \ico_i$ and so
%   $
%     \bmPi(\tsp_i) \bmpi =
%     \ico_i \bmpi = \bmpi \Tsp_i
%     $ by Lemma \ref{lemma: rho pi sigma relation}.
% \end{proof}
\begin{lemma}\label{lemma: commuting relations with transposition}
  For all \(y_{1},y_{2} \in [k-1]\), we have
  \( \Tsp_{(y_{1},y_{2})}^{(k-1)} \bmpi = \bmpi \Tsp_{(y_{1},y_{2})}^{(k)}\).
\end{lemma}
\iftoggle{arxiv}{\begin{proof}
Now, let \(\bfv \in \mathbb{R}^k\) be arbitrary.
Then, by definition, for all \(j \in [k-1]\), we have
\(  [\bmpi \Tsp_{(y_1,y_2)}^{(k)}\bfv]_{j}
  =
  [\Tsp_{(y_1,y_2)}^{(k)}\bfv]_{k}
  -
  [\Tsp_{(y_1,y_2)}^{(k)}\bfv]_{j}
  =
  v_k - v_{\tsp_{(y_1,y_2)}(j)}.
\)
On the other hand,
\(  [\Tsp_{(y_1,y_2)}^{(k-1)} \bmpi \bfv]_{j}
  =
  [\bmpi \bfv]_{\tsp_{(y_1,y_2)}(j)}
  =
  v_k - v_{\tsp_{(y_1,y_2)}(j)}
\)
which implies
\(
\Tsp_{(y_1,y_2)}^{(k-1)} \bmpi
=\bmpi \Tsp_{(y_1,y_2)}^{(k)}\).
\end{proof}}{}

\begin{lemma}\label{lemma: projective transposition}
  Let \(y_1,y_2 \in [k-1]\) be distinct. Then \(\Tsp_{(y_1,y_2  )}^{(k-1)} = \ico_{y_1}\ico_{y_2} \ico_{y_1}.\)
\end{lemma}
\iftoggle{arxiv}{\begin{proof}
  First, we note that
\begin{align*}
  \ico_{y_1}\ico_{y_2} \ico_{y_1} =
  \bmPi_{\tsp_{y_1}}\bmPi_{\tsp_{y_2}}\bmPi_{ \tsp_{y_1}}
  % \quad \because \mbox{Lemma \ref{lemma: rho i is Pi sigma i}}
                                  &=
  \bmPi_{\tsp_{y_1}\tsp_{y_2} \tsp_{y_1}}
  \quad \because \mbox{Lemma \ref{lemma: Pi is a group hom}}
  \\
                                  &=
  \bmPi_{\tsp_{(y_1,y_2)}}
  \quad \because \mbox{Lemma \ref{lemma: transpositional identity}.}
\end{align*}

% Now, let \(\bfv \in \mathbb{R}^k\) be arbitrary.
% Then, by definition, for all \(j \in [k-1]\), we have
% \[
%   [\bmpi \Tsp_{(y_1,y_2)}^{(k)}\bfv]_{j}
%   =
%   [\Tsp_{(y_1,y_2)}^{(k)}\bfv]_{k}
%   -
%   [\Tsp_{(y_1,y_2)}^{(k)}\bfv]_{j}
%   =
%   v_k - v_{\tsp_{(y_1,y_2)}(j)}.
% \]
% On the other hand,
% \[
%   [\Tsp_{(y_1,y_2)}^{(k-1)} \bmpi \bfv]_{j}
%   =
%   [\bmpi \bfv]_{\tsp_{(y_1,y_2)}(j)}
%   =
%   v_k - v_{\tsp_{(y_1,y_2)}(j)}
% \]
% which proves that
% \(
% \Tsp_{(y_1,y_2)}^{(k-1)} \bmpi
% =\bmpi \Tsp_{(y_1,y_2)}^{(k)}\).
Next, from Lemma~\ref{lemma: commuting relations with transposition}, we have
  \( \Tsp_{(y_{1},y_{2})}^{(k-1)} \bmpi = \bmpi \Tsp_{(y_{1},y_{2})}^{(k)}\).
Multiplying both sides
on the right by \(\bmpi^{\dagger}\), we have
\begin{equation}
\Tsp_{(y_1,y_2)}^{(k-1)} \bmpi \bmpi^{\dagger}
=
\bmpi \Tsp_{(y_1,y_2)}^{(k)}\bmpi^{\dagger}.
\label{equation:projective-transposition-helper-1}
\end{equation}
To conclude, we have
  \(    \Tsp_{(y_1,y_2)}^{(k-1)}
    \overset{\textit{1}}{=}
\Tsp_{(y_1,y_2)}^{(k-1)} \bmpi \bmpi^{\dagger}
    \overset{\textit{2}}{=}
\bmpi \Tsp_{(y_1,y_2)}^{(k)}\bmpi^{\dagger}
    \overset{\textit{3}}{=}
\bmPi_{\tsp_{(y_1,y_2)}}\)
where
\emph{1} follows from Remark~\ref{remark:pseudo-inverse-injectivity},
\emph{2}
is Eqn.~\eqref{equation:projective-transposition-helper-1}
and \emph{3} follows from the definition of \(\bmPi\).
\end{proof}}{}

\begin{lemma}\label{lemma: permutation lemma}
  Let \(y, j \in [k]\). Then
  \begin{equation}
    \label{equation: permutation lemma main}
    \ico_{\tsp_y(j)} =
    \begin{cases}
      \Tsp_{(j,y)}^{(k-1)} \ico_j \ico_y &: \mbox{both } y \mbox{ and }j \in [k-1] \\
      \ico_j \ico_y &: \mbox{otherwise}.
    \end{cases}
  \end{equation}
\end{lemma}
\iftoggle{arxiv}{\begin{proof}
  First we consider the case when \(y = j\). Then \(\tsp_y(j) = \tsp_y(y) =k\).
  Hence, the left hand side of Eqn.~\eqref{equation: permutation lemma main}
  is the identity by definition.
  For the right hand side of Eqn.~\eqref{equation: permutation lemma main}, we recall that \(\Tsp_{(y,j)}\) when \(y=j\) is equal to the identity by definition.
  Furthermore, \(\ico_j \ico_y = \ico_y^2\) is also the identity by Lemma~\ref{lemma: rho is involutional}.
  Thus, Eqn.~\eqref{equation: permutation lemma main} holds when \(y=j\).

  Below, we focus on the case when $y \ne j$.
  We will refer to the two conditions of
  the right hand side of
    Eqn.~\eqref{equation: permutation lemma main} as ``sub-cases''.
  First, the $y,j \in [k-1]$ sub-case. We  have $\tsp_y(j) = j$, thus $\ico_{\tsp_y(j)} = \ico_j$. Now, for the right hand side of Eqn.~\ref{equation: permutation lemma main}, we have
  \begin{align*}
    \Tsp_{(j,y)}^{(k-1)}\ico_j \ico_y
    &=
    (\ico_j \ico_y \ico_j) \ico_j \ico_y
    \quad \because \mbox{Lemma \ref{lemma: projective transposition}}
    \\
    &=
    \ico_j \ico_y \ico_y
    \quad \because \mbox{Lemma \ref{lemma: rho is involutional} \( \implies \ico_{j}^{2}\) is the identity}
    \\
    & =
    \ico_j
    \quad \because \mbox{Lemma \ref{lemma: rho is involutional} \( \implies \ico_{y}^{2}\) is the identity.}
  \end{align*}

  Next, consider the second sub-case of Eqn.~\eqref{equation: permutation lemma main} splits into two ``sub-sub-cases'':   \emph{(a)} $j = k$ and $y < k$
\emph{(b)} $y = k$ and \(j \in [k]\) is arbitrary.
  For (a), note that $\tsp_{y}(j) =\tsp_{y}(k) = y$.
  So $\ico_{\tsp_y(j)} = \ico_y$. The right hand side of Eqn.~\eqref{equation: permutation lemma main}, we have $\ico_j \ico_y = \ico_y$ since $\ico_j = \ico_k$ is the identity by definition.
  For (b), the left hand side of Eqn.~\eqref{equation: permutation lemma main}  \(\ico_{\tsp_{y}(j)} = \ico_{\tsp_{k}(j)} = \ico_{j}\).
  Moreover, the right hand side of Eqn.~\eqref{equation: permutation lemma main}
  $\ico_j \ico_y = \ico_{j}\ico_k = \ico_{j}$, as desired.
\end{proof}}{}

\begin{proposition}\label{proposition:commuting-relations-full-perm}
  For an arbitrary $\sigma \in \mathtt{Sym}(k-1)$, let $\sigma' \in \mathtt{Sym}(k)$ denote ``the permutation that extends \(\sigma\) to \([k]\)'', i.e.,
  $\sigma'(k) := k$ and $\sigma'(y) := \sigma(y)$ for $y \in [k-1]$.
    Then we have $\mathbf{S}_{\sigma} \bmpi = \bmpi \mathbf{S}_{\sigma'}$.
\end{proposition}
\iftoggle{arxiv}{\begin{proof}
  First, recall that \(\mathbf{S}_{\sigma} \in \mathbb{R}^{(k-1)\times (k-1)}\)
  and \(\mathbf{S}_{\sigma'} \in \mathbb{R}^{k\times k}\) are the permutation matrices corresponding to \(\sigma\) and \(\sigma'\), respectively.
As mentioned in the proof of Corollary~\ref{corollary: generator}, a standard result in group theory states that
  the set of transpositions on \([k-1]\) generates \(\mathtt{Sym}(k-1)\) as a group. Thus, there exists an integer \(l \ge 0\) and indices \(y_{1},j_{1},\dots, y_{l},j_{l} \in [k-1]\) such that
$\sigma = \tsp_{(y_{1},j_{1})} \tsp_{(y_{2},j_{2})}\cdots \tsp_{(y_{l},j_{l})}$.
By Lemma~\ref{lemma:matrix-permutation-is-order-reversing}, we have
\[
  \mathbf{S}_{\sigma} =
  \Tsp_{(y_{l},j_{l})}^{(k-1)} \cdots \Tsp_{(y_{2},j_{2})}^{(k-1)}\Tsp_{(y_{1},j_{1})}^{(k-1)}
  \quad \mbox{and}
  \quad
    \mathbf{S}_{\sigma'} =
\Tsp_{(y_{l},j_{l})}^{(k)} \cdots \Tsp_{(y_{2},j_{2})}^{(k)}\Tsp_{(y_{1},j_{1})}^{(k)}
\]
By Lemma~\ref{lemma: commuting relations with transposition}, we have
\begin{align*}
  \mathbf{S}_{\sigma} \bmpi =
\Tsp_{(y_{l},j_{l})}^{(k-1)} \cdots \Tsp_{(y_{2},j_{2})}^{(k-1)}\Tsp_{(y_{1},j_{1})}^{(k-1)}
\bmpi
  =
\bmpi
\Tsp_{(y_{l},j_{l})}^{(k)} \cdots \Tsp_{(y_{2},j_{2})}^{(k)}\Tsp_{(y_{1},j_{1})}^{(k)}
  =
\bmpi
    \mathbf{S}_{\sigma'}
\end{align*}
as desired.
% \begin{align*}
%   \mathbf{S}_{\sigma} \bmpi &=
% \Tsp_{(y_{1},j_{1})} \Tsp_{(y_{2},j_{2})}\cdots \Tsp_{(y_{n},j_{n})}
% \bmpi
%                               \\
%   &=
% \bmpi
% \Tsp_{(y_{1},j_{1})} \Tsp_{(y_{2},j_{2})}\cdots \Tsp_{(y_{n},j_{n})}
%   \\
%   &=
% \bmpi
%     \mathbf{S}_{\sigma'}
% \end{align*}
\end{proof}}{}

\begin{lemma}\label{lemma:subtraction-matrix-label-code}
  Let \(\vone := \vone^{(k-1)}\) denote\footnote{We drop the superscript \((k-1)\) for convenience.} the \((k-1)\)-dimensional (column) vector of all ones.
  Let \(\bfe_{y} := \bfe_{y}^{(k-1)}\) denote the \((k-1)\)-dimensional \(y\)-th elementary basis (column) vector.
  Let \(y,j \in [k]\) be such that \(y \ne j\). Then we have the following identities
  \begin{align}
    \ico_{y} - \Tsp_{(y,j)} \ico_{j} = (\vone + \bfe_{j})(\bfe_{j} - \bfe_{y})^{\top} \qquad&\mbox{if \(y,j\in[k-1]\),}
\label{equation:subtraction-matrix-label-code}
    \\
      \ico_{y} - \ico_{j} = (\vone + \bfe_{j})(\bfe_{j})^{\top}\qquad&\mbox{
  if \(y  = k\),}
\label{equation:subtraction-matrix-label-code-special-case-1}
    \\
    \ico_{y} - \ico_{j} = -(\vone + \bfe_{y})(\bfe_{y})^{\top}\qquad&\mbox{if \(j = k\).}
\label{equation:subtraction-matrix-label-code-special-case-2}
  \end{align}
\end{lemma}

\iftoggle{arxiv}{\begin{proof}
  First, we recall the notation that if \(\mathbf{A}\) is a matrix and \(l\) is a column index, then \([\mathbf{A}]_{:l}\) denote the \(l\)-th column of \(\mathbf{A}\).

  We prove \cref{equation:subtraction-matrix-label-code}.
  Let \(l \in [k-1] \setminus \{y,j\}\). Then both \([\ico_{y}]_{:l}\) and
  \([\Tsp_{(y,j)} \ico_{j}]_{:l} = \bfe_{l}\).
  Thus, only the \(y\)-th and \(j\)-th columns of
  \(\ico_{y} - \Tsp_{(y,j)} \ico_{j} \)
 are possibly nonzero.
We claim that the \(y\)-th and \(j\)-th columns are
\(-(\vone + \bfe_{j})\)
and
\(\vone + \bfe_{j}\), respectively.
To prove the claim, first by definition we have \([\ico_{y}]_{:y} =-\vone\).
On the other hand, \([\ico_{j}]_{:y} =\bfe_{y}\) and so
\( [\Tsp_{(y,j)} \ico_{j}]_{:y}= \bfe_{j}\).
Hence, the \(y\)-th column of
\(\ico_{y} - \Tsp_{(y,j)} \ico_{j} \)
is
\(-(\vone + \bfe_{j})\).
Furthermore,
\([\ico_{y}]_{:j} =\bfe_{j}\)
and
\([\ico_{j}]_{:j} = -\vone\).
Thus,
\([\Tsp_{(y,j)}\ico_{j}]_{:j} = -\vone\) as well.
This proves that the \(j\)-th column of
\(\ico_{y} - \Tsp_{(y,j)} \ico_{j} \)
 is equal to
\(\vone + \bfe_{j}\). Thus our claim holds and \cref{equation:subtraction-matrix-label-code} follows.

Next, we consider the case when \(y = k\).
Let \(l \in [k-1] \setminus \{j\}\) be an index not equal to \(j\).
Then both \([\ico_{y}]_{:l}\) and \([\ico_{j}]_{:l} = \bfe_{l}\).
Thus, only the \(j\)-th column of \(\ico_{y} - \ico_{j}\) is possibly nonzero. Now, since
\([\ico_{j}]_{:j} = -\vone\),
\cref{equation:subtraction-matrix-label-code-special-case-1}
follows immediately.

Finally, \cref{equation:subtraction-matrix-label-code-special-case-2}
follows  from
\cref{equation:subtraction-matrix-label-code-special-case-1}
by swapping \(y\) and \(j\).
\end{proof}}{}

\section{Proof of {Theorem}~\ref{theorem:relative-margin-form}}\label{section:appendix:proof-of-theorem-relative-margin-form}

 We first recall the definition of a template from
{Theorem}~\ref{theorem:relative-margin-form}. For the reader's convenience, we restate it as a definition below:

\begin{lemma}\label{lemma:relative-margin-form-special-case}
  Let \(\calL\) be a PERM loss with template \(\TP\). Let \(\bmpi\) be as in Definition~\ref{definition:relative-marginalization-mapping}. Then for all \(\bfv \in \mathbb{R}^{k}\) we have
  \(\calL_{k}(\bfv):=[\calL(\bfv)]_{k} = \TP(\bmpi \bfv)\).
\end{lemma}
\begin{proof}
Recall from Definition~\ref{definition:relative-marginalization-mapping}
that
\(  \bmpi\bfv  = (v_k-v_1, v_k - v_2,\dots, v_k - v_{k-1})^{\top}
\).
Thus, the result now follows immediately from Definition~\ref{definition:PERM-loss}.
\end{proof}

For the reader's convenience, the following is a restatement of Theorem~\ref{theorem:relative-margin-form}.

\begin{theorem}\label{theorem:appendix:relative-margin-form}

    Let $\calL: \mathbb{R}^k \to \mathbb{R}^k$ be a PERM loss with template $\TP$,
    and let
\(\bfv \in \mathbb{R}^{k}\) and \(y \in [k]\) be arbitrary.
    Then
    \(\TP\) is a symmetric function
    and
    \(\calL\) can be expressed as
  \begin{equation}
\calL_y(\bfv) =
  \TP(\ico_y \bmpi\bfv).
    \label{equation:appendix:relative-margin-form}
\end{equation}
Conversely, given a symmetric function \(\TP : \mathbb{R}^{k-1} \to \mathbb{R}\), the function \(\calL = (\calL_{1},\dots, \calL_{k}) : \mathbb{R}^{k} \to \mathbb{R}^{k}\)  defined componentwise via \cref{equation:appendix:relative-margin-form} is a PERM loss with template \(\TP\).
  \end{theorem}

\begin{proof}[Proof of Theorem~\ref{theorem:appendix:relative-margin-form}]
    We begin with  the ``\(\implies\)'' direction of the proof.
    First, we check that
    $\TP: \mathbb{R}^{k-1} \to \mathbb{R}$ is symmetric, i.e.,
    $\TP(\mathbf{S}_{\sigma} \bfz) = \TP(\bfz)$
    for
    all $\bfz \in \mathbb{R}^{k-1}$ and all $\sigma \in \mathtt{Sym}(k-1)$.
    Below, fix such \(\bfz\) and \(\sigma\).

    Pick $\bfv \in \mathbb{R}^{k}$ such that $\bmpi \bfv = \bfz$. For instance, let \(\bfv :=
    \begin{bmatrix}
      -\bfz^{\top} & 0
    \end{bmatrix}^{\top}
    \).
    Define $\sigma' \in \mathtt{Sym}(k)$ to be the permutation that extends \(\sigma\) to \([k]\) as in
Proposition~\ref{proposition:commuting-relations-full-perm}. Then
\begin{align*}
    \TP(\mathbf{S}_{\sigma}\bfz)
  =
    \TP(\mathbf{S}_{\sigma} \bmpi \bfv)
&  =
    \TP(\bmpi \mathbf{S}_{\sigma'}  \bfv) \quad \because \mbox{Proposition~\ref{proposition:commuting-relations-full-perm}}
  \\&
    =
  [ \calL(\mathbf{S}_{\sigma'}\bfv) ]_{k}
  \quad \because \mbox{Definition~\ref{definition:PERM-loss} of the template \(\TP\) of a PERM loss}
  \\&
    =
  [\mathbf{S}_{\sigma'}\calL(\bfv)]_k
  \quad \because \mbox{\(\calL\) is permutation equivariant (Definition~\ref{definition:PERM-loss})}
  \\
  &
  =
  [\calL(\bfv)]_{\sigma'(k)}
  \quad \because \mbox{\(\mathbf{S}_{\sigma'}\) is the permutation matrix of \(\sigma'\)}
  \\
  &
  =
  [\calL(\bfv)]_{k}
  \quad \because \mbox{Definition of \(\sigma'\) in Proposition~\ref{proposition:commuting-relations-full-perm}}
  \\
  &
  =
  \TP(\bmpi \bfv).
  \quad \because \mbox{Lemma~\ref{lemma:relative-margin-form-special-case}}
\end{align*}
Since by definition
\(
\TP(\bmpi \bfv)
=
\TP(\bfz)
\), we have by the above that
    $\TP(\mathbf{S}_{\sigma} \bfz) = \TP(\bfz)$ for all \(\bfz\).
    This proves that $\TP$ is symmetric. Next, we prove
    \cref{equation:appendix:relative-margin-form}, i.e., \(   [\calL(\bfv)]_y
  =
  \TP(\ico_{y} \bmpi \bfv)
 \) for all \( y \in [k]\).  Again, this is a straight forward computation:
\begin{align*}
  [\calL(\bfv)]_y
  &=
  [\calL(\bfv)]_{\tsp_{y}(k)}
    =
  [\Tsp_y \calL(\bfv)]_k
    \qquad \because \mbox{definitions of \(\tsp_{y}\) and \(\Tsp_{y}\)}
    \\
  &=
  [\calL(\Tsp_y \bfv)]_k
    \qquad \because \mbox{\(\calL\) is permutation equivariant}
    \\&
  =
  \TP(\bmpi \Tsp_y \bfv)
  \qquad \because \mbox{Lemma~\ref{lemma:relative-margin-form-special-case}}
  \\& =
  \TP(\ico_y \bmpi \bfv)
  \qquad \because \mbox{Lemma~\ref{lemma: rho pi sigma relation}.}
\end{align*}
This proves \cref{equation:appendix:relative-margin-form} and thus the ``\(\implies\)'' direction follows.
\iftoggle{arxiv}
{
  Now, we prove ``\(\impliedby\)'' direction, i.e., the ``Conversely'' part of the theorem.
  First, we argue that \(\calL\) defined via \cref{equation:appendix:relative-margin-form}  is relative margin-based. Define \(\LL_{y} : \mathbb{R}^{k-1} \to \mathbb{R}\) by \(\LL_{y}(\bfz) := \TP(\ico_{y} \bfz )\) for all \(\bfz \in \mathbb{R}^{k-1}\). Then \cref{equation:appendix:relative-margin-form}
  immediately implies \cref{equation:relative-margin-form-with-reduced-form} in item 2 of Definition~\ref{definition:PERM-loss}.

  Next, we check that $\calL$ is permutation equivariant.
  Let $\bfv \in \mathbb{R}^k$ be arbitrary.
  By Lemma~\ref{lemma:transposition-suffices-for-permutation} below, it suffices to prove
  the claim that $\calL(\Tsp_y \bfv) = \Tsp_y \calL(\bfv)$ for all $y \in [k]$.
  In other words, for all \(y,j \in [k]\) we have
  $[\calL(\Tsp_y \bfv)]_{j} = [\Tsp_y \calL(\bfv)]_{j}$.
  To see this, we have
  \begin{align*}
    [\Tsp_y \calL(\bfv)]_j &=
                             [\calL(\bfv)]_{\tsp_y(j)} \quad \because \mbox{Definition of \(\Tsp_{y}\)}
    \\ &=
         \TP(\ico_{\tsp_y(j)} \bmpi \bfv)
         \quad \because \mbox{Definition of $\calL$}
    \\&{=}
    \begin{cases}
      \TP(\Tsp_{(j,y)} \ico_j \ico_y \bmpi \bfv) &: y,j \in [k-1]\\
      \TP(\ico_j \ico_y \bmpi \bfv) &: \mbox{otherwise}
    \end{cases}
    \quad \because
    \mbox{Lemma~\ref{lemma: permutation lemma}}
    \\&=
    \TP(\ico_j \ico_y \bmpi \bfv)
    \quad \because \mbox{$\TP$ is symmetric}
    \\&=
    \TP(\ico_j \bmpi \Tsp_y \bfv)
    \quad \because \mbox{Lemma~\ref{lemma: rho pi sigma relation}}
    \\&=
    [\calL(\Tsp_y \bfv)]_j
    \quad \because \mbox{Definition of $\calL$}
  \end{align*}
  This proves that $\Tsp_y \calL(\bfv) = \calL(\Tsp_y \bfv)$ for all \(y \in [k]\).
}
{
    \textcolor{black}{
  The proof of the  ``\(\impliedby\)'' direction, i.e., the ``Conversely'' part of the theorem, proceeds similarly.
  For the reader's convenience, it is omitted here but is included in the arXiv version of this manuscript \citep{wang2023unified}.
  }
}
\end{proof}

\begin{lemma}\label{lemma:transposition-suffices-for-permutation}
  Let \(\calL : \mathbb{R}^{k} \to \mathbb{R}^{k}\)  be a function  such that
  \(\Tsp_{y} \calL(\bfv) = \calL(\Tsp_{y}\bfv)\) for all \(y \in [k]\) and \(\bfv \in \mathbb{R}^{k}\). Then
  \(\calL\) is permutation-equivariant, i.e., \(  \calL(\mathbf{S}_{\sigma} \bfv)
  =
  \mathbf{S}_{\sigma} \calL(\bfv)
  \) for all \(y \in [k]\) and \(\bfv \in \mathbb{R}^{k}\).
\end{lemma}

\begin{proof}
For an arbitrary $\sigma \in \mathtt{Sym}(k)$, write $\sigma = \tsp_{y_1} \tsp_{y_2} \cdots  \tsp_{y_n}$ as a product of transposition where \(n\ge 0\) is an integer and \(y_{i} \in [k-1]\) for each \(i \in [n]\). This is possible because of Corollary~\ref{corollary: generator}.
By Lemma~\ref{lemma:matrix-permutation-is-order-reversing}, we have
\(\mathbf{S}_{\sigma} = \Tsp_{y_n}\cdots \Tsp_{y_{2}}\Tsp_{y_1}\).
Thus
\[
  \calL(\mathbf{S}_{\sigma} \bfv)
  =
  \calL(
\Tsp_{y_n}\cdots \Tsp_{y_1}
\bfv)
  =
  \Tsp_{y_n}\calL(\Tsp_{y_{n-1}}\cdots \Tsp_{y_{1}}\bfv)
  =
  \cdots
  =
  \Tsp_{y_{n}}\cdots \Tsp_{y_{1}}\calL(\bfv)
  =
  \mathbf{S}_{\sigma} \calL(\bfv)
\]
where ``\(=\cdots =\)'' denotes repeated application of $\calL(\Tsp_y \bfv) = \Tsp_y \calL(\bfv)$ for all $y \in [k]$.
\end{proof}

      % The following notation will be useful:
% \begin{definition}\label{definition:reduced-form-of-PERM-loss}
%   Let $\calL$ be a PERM loss and \(\TP\) be its template (Definition~\ref{definition:PERM-loss}).
%   Define the \emph{reduced form} of \(\calL\)
%   as the vector-valued function
% $\LL = (\LL_1, \dots, \LL_k) : \mathbb{R}^{k-1} \to \mathbb{R}^k$
% where
% \begin{equation}
% \LL_{y}(\bfz) := \TP( \ico_{y} \bfz ), \quad \mbox{for \(y \in [k]\) and \(\bfz \in \mathbb{R}^{k-1}\).}
% \label{equation:relative-margin-form-summarized}
% \end{equation}
% \end{definition}

\begin{remark}\label{remark:relative-margin-form-summarized}
For convenience,  we now summarize the relationship between \(\calL\), \(\LL\) and \(\TP\) from Definition~\ref{definition:PERM-loss} in terms of the input to \(\TP\), typically denoted below by \(\bfz \in \mathbb{R}^{k-1}\).
Let
\(y \in [k]\) and \(\bfz \in \mathbb{R}^{k-1}\) be arbitrary.
Define $\bfv \in \mathbb{R}^{k}$ be such that \(\bmpi\bfv = \bfz\), where \(\bmpi\) is as defined in Definition~\ref{definition:relative-marginalization-mapping}.
     Since \(\bmpi \bmpi^{\dagger}\) is the identity (Remark~\ref{remark:pseudo-inverse-injectivity}), we can for instance let \(\bfv =
    \bmpi^{\dagger} \bfz
    \).
    Then we have
\begin{equation}
  \LL_{y}(\bfz)
  =
  [\calL(\bfv)]_y
  =\TP(\ico_{y} \bfz).
\label{equation:relative-margin-form-summarized}
\end{equation}
The first equality is
\cref{equation:relative-margin-form-with-reduced-form} (rewritten using \(\bmpi \bfv = \bfz\)) and the second equality is
    \cref{equation:appendix:relative-margin-form}.
    % The second equality is simply the definition of \(\LL\) given above.
    % The first equality is
    % \cref{equation:appendix:relative-margin-form}.
    A useful corollary of the above identity is the following:
    \begin{equation}
      \label{equation:permutation-equivariance-for-reduced-loss}
  \mathbf{S}_{\sigma} \LL(\bfz)
   =
    \LL(\bmPi_{\sigma}\bfz), \quad \mbox{for all \(\sigma \in \mathtt{Sym}(k)\).}
    \end{equation}
    To prove
      \cref{equation:permutation-equivariance-for-reduced-loss},
 we let  \(\bfv := \bmpi^{\dagger}\bfz\) as discussed earlier. Then
  \[\mathbf{S}_{\sigma} \LL(\bfz)
   \overset{\mathit{1}}{=}
    \mathbf{S}_{\sigma}\calL(\bfv)
   \overset{\mathit{2}}{=}
    \calL(\mathbf{S}_{\sigma}\bfv)
   \overset{\mathit{3}}{=}
    \LL(\bmpi\mathbf{S}_{\sigma}\bfv)
   \overset{\mathit{4}}{=}
    \LL(\bmpi\mathbf{S}_{\sigma}\bmpi^{\dagger}\bfz)
   \overset{\mathit{5}}{=}
    \LL(\bmPi_{\sigma}\bfz).
  \]
  For 1 and 3, we used \cref{equation:appendix:relative-margin-form}. For 2, we used permutation equivariance of \(\calL\). For 4, we used the definition of \(\bfv\).
  For 5, we used Definition~\ref{definition:irreducibility-representation}
of \(\bmPi_{\sigma}\).
\end{remark}

  \section{Well-incentivized losses}
  This section discusses the ``well-incentivized'' property of a loss that should be satisfied for a multiclass classification loss to be useful.
  This will be made clear after
{Remark}~\ref{remark:well-incentivized-loss-and-min-max-identity}.

\begin{definition}[Well-incentivized losses]
  % [Properties of multiclass loss function]
  Let \(\calL\) be a multiclass loss function (Definition~\ref{definition:PERM-loss}).
  We say that \(\calL\) is \emph{well-incentivized} (resp.\ \emph{strictly well-incentivized}) if
  for all \(\bfv \in \mathbb{R}^k\) and distinct \(y,j \in [k]\), \(v_{j} \le v_{y}\) (resp.\  \(v_{j} < v_{y}\)) implies \(\calL_{j}(\bfv) \ge \calL_{y}(\bfv)\)
(resp.\ \(\calL_{j}(\bfv) > \calL_{y}(\bfv)\)).
\end{definition}

\begin{remark}\label{remark:well-incentivized-loss-and-min-max-identity}
  Suppose that \(\calL\) is well-incentivized. Let \(\bfv \in \mathbb{R}^{k}\) and \(y \in \argmax_{j \in [k]} v_{j}\).
  Then \(\calL_{y}(\bfv) = \min_{j \in [k]} \calL_{j}(\bfv)\).
\end{remark}
{Remark}~\ref{remark:well-incentivized-loss-and-min-max-identity} implies that for well-incentivized losses, the \(\argmax\) predictor is correct, i.e., \(\calL_{y}(\bfv)\) is minimized when the class score \([\bfv]_{y}\) is the highest.
Next, we define a condition on the template \(\TP\) such that the corresponding PERM loss is well-incentivized:

\begin{definition}\label{definition:monotone-functions}
    A function $f : \mathbb{R}^{n} \to \mathbb{R}$ is
    \emph{non-increasing} (resp.\ \emph{decreasing})
    if for all \(\bfv,\bfw \in \mathbb{R}^{n}\) such that
    \(\bfv \succeq \bfw\) (resp.\ \(\bfv \succ \bfw\))
    we have
    \(f(\bfv) \le f(\bfw)\) (resp.\ \(f(\bfv) < f(\bfw)\)).
  \end{definition}

  \begin{proposition}\label{proposition:monotone-function}
    Let \(f: \mathbb{R}^{n} \to \mathbb{R}\) be a continuously differentiable function.
    If \(\nabla_{f}(\cdot) \preceq \vzero\) (resp.\ \(\nabla_{f}(\cdot) \prec \vzero\))
    then \(f\) is non-increasing (resp.\ decreasing).
  \end{proposition}

\begin{proposition}\label{proposition:sufficient-condition-for-well-incentivized-ness}
    Let $\calL$ be a PERM loss whose template \(\TP\) is non-increasing (resp.\ decreasing).
    Then \(\calL\) is (resp.\ strictly) well-incentivized.
  \end{proposition}
\iftoggle{arxiv}{\begin{proof}[Proof of Proposition~\ref{proposition:monotone-function}]
   Let \(\bfv,\bfw \in \mathbb{R}^{n}\)  be such that \(\bfv \succeq \bfw\).
   Define \(\gamma : \mathbb{R} \to \mathbb{R}^{n}\) by \(\gamma(t) := (1-t) \bfv + t \bfw \).
   Our goal is to show that \(f(\bfv) \le f(\bfw)\).
    The ``strict'' case will be addressed later.
   By the Fundamental Theorem of Calculus,
   \[
     f(\bfw) - f(\bfv) = f(\gamma(1)) - f(\gamma(0)) =  \int_{0}^{1} \frac{d
f\circ \gamma
     }{dt}(t) dt
=
\int_{0}^{1} \nabla_{f}(\gamma(t))^{\top} \gamma'(t) dt
   \]
   where  we used the
   chain rule for curves \cref{equation:chain-rule-for-curves}
   for the last equality.
   Now, note that \(\nabla_{f}(\cdot) \preceq \vzero\) by assumption. Moreover, \(\gamma'(t) = \bfw - \bfv \preceq \vzero\).
   Therefore, \(\nabla_{f}(\gamma(t))^{\top} \gamma'(t) \ge 0\) for all \(t\).
The proof for the ``strict'' case follows by replacing every ``\(\succeq\)'' /
``\(\ge\)''
with ``\(\succ\)''/ ``\(>\)''.
\end{proof}
\begin{proof}[Proof of Proposition~\ref{proposition:sufficient-condition-for-well-incentivized-ness}]
    Below, let \(y,j\in[k]\) and \(\bfv \in \mathbb{R}^{k}\) be such that
    suppose that \(v_{y} \ge v_{j}\).
    Our goal is to show that \(\calL_{y}(\bfv) = \TP(\ico_{y}\bfz) \le \TP(\ico_{j} \bfz) = \calL_{j}(\bfv)\).
    The ``strict'' case will be addressed later.

    Let \(\bfz := \bmpi \bfv\).
    First, we consider the case when \(j = k\) or \(y = k\).
    If \(j=k\), then
    \[
      \ico_{y} \bfz
    - \ico_{j}\bfz
    =
      (\ico_{y}
    - \ico_{j})\bfz
    \overset{
\cref{equation:subtraction-matrix-label-code-special-case-2}
    }{=}
-(\vone + \bfe_{y})(\bfe_{y})^{\top} \bfz
    =
    -(\vone + \bfe_{y}) z_{y}  \succeq \vzero
\]
    where ``\(\succeq\)'' holds since \(z_{y} = v_{j} - v_{y} \le 0\).
    If \(y=k\), then applying a similar argument as in the previous case with
\cref{equation:subtraction-matrix-label-code-special-case-1}
yields
    \(
    \ico_{y} \bfz
    - \ico_{j}\bfz
    =
    (\vone + \bfe_{j}) z_{j} \succeq \vzero
    \).
    Thus, when \(y = k\) or \(j=k\), we have
    \(      \ico_{y}\bfz \succeq \ico_{j} \bfz
    \) which implies that
    \(\calL_{y}(\bfv) = \TP(\ico_{y}\bfz) \le \TP(\ico_{j} \bfz) = \calL_{j}(\bfv)\).
    Note that here, we used
    \cref{equation:appendix:relative-margin-form}
    for the first and last equalities.

    Finally, suppose that \(y,j \in [k-1]\), then
    \[
      \ico_{y} \bfz
    - \Tsp_{(y,j)}\ico_{j}\bfz
    =
(      \ico_{y}
    - \Tsp_{(y,j)}\ico_{j})\bfz
\overset{\cref{equation:subtraction-matrix-label-code}}{=} (\vone + \bfe_{j})(\bfe_{j} - \bfe_{y})^{\top} \bfz
    =
    (\vone + \bfe_{j}) (z_{j} - z_{y}) \succeq \vzero.
  \]
  Note that for the ``\(\succeq\)'', we used the fact that
    \(z_{j} - z_{y} = (v_{k} - v_{j}) - (v_{k}-v_{y}) = v_{y} - v_{j} \ge 0\).
If \(y ,j \in [k-1]\), then
    \(      \ico_{y}\bfz \succeq \Tsp_{(y,j)}\ico_{j} \bfz
    \) which implies that
    \(\calL_{y}(\bfv) = \TP(\ico_{y}\bfz) \le \TP(\Tsp_{(y,j)}\ico_{j} \bfz)
    =
\TP(\ico_{j} \bfz)
    = \calL_{j}(\bfv)\). Note that \(\TP(\Tsp_{(y,j)}\ico_{j} \bfz)
    =
\TP(\ico_{j} \bfz)\) holds because \(\TP\) is a symmetric function.
The proof for the strict case follows by replacing every ``\(\succeq\)'' /
``\(\ge\)''
with ``\(\succ\)''/ ``\(>\)''.
\end{proof}}{
    \textcolor{black}{
      The proofs of
Propositions~\ref{proposition:monotone-function}
     and
     \ref{proposition:sufficient-condition-for-well-incentivized-ness}
    are straightforward calculations involving the Fundamental Theorem of Calculus.
  For the reader's convenience, they are omitted here but  are included in the arXiv version of this manuscript \citep{wang2023unified}.
  }
}

  \section{Proof of Proposition~\ref{proposition:sufficient-condition-for-gamma-phi-losses-to-be-regular}}

  \iftoggle{arxiv}
  {
    For notational simplicity, we write \(\dot{\phi}\) and \(\ddot{\phi}\) for the first and second derivative of \(\phi\) and likewise for \(\gamma\).
    Let \(\TP\) be the template of \(\calL^{\gamma,\phi}\), which is given explicitly in Example~\ref{example:multiclass-exponential-loss}.
    We first show that $\nabla \TP (\bfz)\prec \vzero$ for all \(\bfz \in \mathbb{R}^{k-1}\).
    Equivalently, we show that
    \(
    \frac{\partial \TP}{\partial z_i}
    (\bfz)
    < 0
    \)
    for each \(i \in [k-1]\).
    Using the chain rule, we get
    \(  \frac{\partial \TP}{\partial z_i}
    (\bfz)
    =
    \dot{\gamma}\left(\sum_{j \in [k-1]} \phi(z_j)\right)
    \dot{\phi}(z_i)
    \)
    which is negative by the assumptions of  Proposition~\ref{proposition:sufficient-condition-for-gamma-phi-losses-to-be-regular}.

    Next, we show that \(\TP\) is strictly convex. To this end, it suffices to  show that the Hessian is positive definite  for all input \(\bfz \in \mathbb{R}^{k-1}\) \citep[\S 3.1.4]{boyd2004convex}.
    We being by calculating the mixed partial derivative for \(i,l \in [k-1]\):
    \[ \textstyle  \frac{\partial^2 \TP}{ \partial z_l \partial z_i} (z)
      =
      \ddot{\gamma}\left(\sum_{j \in [k-1]} \phi(z_j)\right)
      \dot{\phi}(z_l)
      \dot{\phi}(z_i)
      +
      \delta_{il}
      \dot{\gamma}\left(\sum_{j \in [k-1]} \phi(z_j)\right)
      \ddot{\phi}(z_i)\]
    where
    \(
    \delta_{il} = 1
    \)
    if \(i = l\) and
    \(
    \delta_{il} = 0
    \)
    otherwise.
    Abusing notation, we write
    \(\dot{\phi}(\bfz)\) to denote the (column) vector
    \(
    \begin{bmatrix}
      \dot{\phi}(z_1)
      &
        \cdots
      &
        \dot{\phi}(z_{k-1})
    \end{bmatrix}^{\top}
    \) and define
    \(\ddot{\phi}(\bfz)\) likewise.
    Thus, the Hessian of \(\TP\) at \(\bfz\) is given by
    \[
      \textstyle
      \ddot{\gamma}\left(\sum_{j \in [k-1]} \phi(z_j)\right)
      \dot{\phi}(\bfz)
      \dot{\phi}(\bfz)^{\top}
      +
      \dot{\gamma}\left(\sum_{j \in [k-1]} \phi(z_j)\right)
      \mathrm{diag}(\ddot{\phi}(\bfz)).
    \]
    The assumption on \({\phi}\) implies that
    \(
    \mathrm{diag}(\ddot{\phi}(\bfz))
    \)
    is a positive definite matrix.
    Moreover, \(
    \dot{\phi}(\bfz)
    \dot{\phi}(\bfz)^{\top}
    \) is a positive definite matrix since it is the outer product of a vector with itself.
    By the assumption on \(\gamma\), the scalar-valued coefficients
    \(    \ddot{\gamma}\left(\sum_{j \in [k-1]} \phi(z_j)\right)
    \)
    and
    \(  \dot{\gamma}\left(\sum_{j \in [k-1]} \phi(z_j)\right)
    \)
    are nonnegative and positive, respectively.
    Therefore, the Hessian of \(\TP\) is positive definite.

    Next, we check that $\TP$ is semi-coercive. Winding Definition~\ref{definition:monotone-functions}, this means that for all $c \in \mathbb{R}$, there exists $b \in \mathbb{R}$ such that
    \begin{equation}
      \label{equation: pre-coercive checking}
      \{\bfz \in \mathbb{R}^{k-1}: \TP(\bfz) \le c\}
      \subseteq
      \{\bfz \in \mathbb{R}^{k-1}: \min_{j \in [k-1]} z_{j}  \ge b\}.
    \end{equation}
    Note that \(\gamma\) and \(\phi\) are injective by our assumptions in Proposition~\ref{proposition:sufficient-condition-for-gamma-phi-losses-to-be-regular}.
    Next, note that
    \begin{align*}
      \textstyle
      \TP(\bfz) =
      \gamma\left(\sum_{j \in [k-1]} \phi(z_j)\right)
      \le c
      \iff &
             \sum_{j \in [k-1]} \phi(z_j)
             \le \gamma^{-1}(c)
      \\ \implies &
                    \max \left\{\phi(z_j): j \in [k-1]\right\}
                    \le \gamma^{-1}(c)
      \\ \iff &
                \phi\left(\min \left\{z_j: j \in [k-1]\right\}\right)
                \le \gamma^{-1}(c)
      \\ \iff &
                \phi(\min z)
                \le \gamma^{-1}(c)
      \\ \iff &
                \min z \ge \phi^{-1}(\gamma^{-1}(c)).
    \end{align*}
    If we let $\phi^{-1}(\gamma^{-1}(c))=:b$, then we get \cref{equation: pre-coercive checking}.
    Thus, we have proven that \(\calL^{\gamma,\phi}\)  is regular.
    Now, by the assumption that $\phi(t) \to 0$ as $t\to +\infty$, we immediately get that
    \(\calL^{\gamma,\phi}\) is totally regular since all truncations of \(\calL^{\gamma,\phi}\)
    are also Gamma-Phi losses with the same \(\gamma,\phi\).
  }
  {
    \textcolor{black}{
      The proof is a straightforward calculation involving the chain rule and properties of the functions \(\phi\) and \(\gamma\).
  For the reader's convenience, it is omitted here but is included in the arXiv version of this manuscript \citep{wang2023unified}.
  }
  }

    \section{Proof of Proposition~\ref{proposition:sum-of-totally-regular-loss-is-totally-regular}}
  We first prove that \(\calL + \calL'\) is regular.
  Let \(\TP\) and \(\TP'\) be the templates of \(\calL\) and \(\calL'\) respectively.
  Then \(\calL + \calL'\) has template \(\TP + \TP'\).
  All conditions in Definition~\ref{definition:regular-PERM-loss} clearly holds, except for the semi-coercive property which we now check.
  Let \(c,b \in \mathbb{R}\) be arbitrary and satisfy
  \[
      \{\bfz \in \mathbb{R}^{n}: \TP(\bfz)  \le c\}
      \subseteq
      \{\bfz \in \mathbb{R}^n: b \le \min_{j \in [n]} z_{j} \}.
  \]
  Since \(\TP\) and \(\TP'\) are both non-negative, we have
  \(\TP(\bfz) + \TP'(\bfz) \le c\) implies
  \(\TP(\bfz)  \le c\).
  Thus
\[
      \textstyle
      \{\bfz \in \mathbb{R}^{n}: \TP(\bfz) + \TP'(\bfz) \le c\}
      \subseteq
      \{\bfz \in \mathbb{R}^{n}: \TP(\bfz)  \le c\}
      \subseteq
      \{\bfz \in \mathbb{R}^n: b \le \min_{j \in [n]} z_{j} \},
    \]
    which shows that \(\TP + \TP'\) is semi-coercive, as desired.
    The fact that \(\calL+\calL'\) is \emph{totally} regular follows immediately from the definition of the truncation
in
  {Proposition}~\ref{proposition:truncation-of-a-PERM-loss}.
  That \(\lambda \calL\) is totally regular is completely straightforward and thus we omit its proof.

      \section{Properties of Regular PERM losses}\label{section:properties-of-regular-PERM-loss}

      In this section, we will prove several key properties of regular PERM losses which were introduced in {Definition}~\ref{definition:regular-PERM-loss}.
      Recall that
      a regular PERM loss has a template \(\TP\) such that
    \(\TP\) is nonnegative, twice differentiable, strictly convex, semi-coercive,
    and
      the gradient \(\nabla_{\TP}(\bfz) \prec \vzero\) is entrywise negative for all \(\bfz \in \mathbb{R}^{k-1}\).
      Section~\ref{section:semi-coercivity} focuses on consequences of the semi-coercivity condition.
      Sections~\ref{section:link-function} and \ref{section:geometry-of-loss-surface} focus on consequences of the other aforementioned conditions.
      Finally, Section~\ref{section:proof theorem totally regular PERM loss is CC} presents the proof of our main result Theorem~\ref{theorem: nested family of regular PERM losses are CC - exposition version}.
      \iftoggle{arxiv}{
      Below, we give an overview of which results explicitly use which conditions:
      \begin{compactenum}
        \item Lemma~\ref{lemma:semi-coercivity-of-template-implies-bounded-sublevel-set}
        ---
        Semi-coercivity of \(\TP\)

        \item
Proposition~\ref{proposition:semi-coercive-template-implies-coercive-conditional-risk} ---
Nonnegativity of \(\TP\).
\item
Proposition~\ref{proposition:strictly-convex-template-implies-strictly-convex-conditional-risk} --- Strict convexity of \(\TP\)
\item
Lemma~\ref{lemma:Az-matrix-is-nonsing-M-matrix}
---
\(\nabla_{\TP}(\cdot) \prec \vzero\)
\item
Lemma~\ref{lemma: useful lemma for proving convexity}
---
twice differentiability of \(\TP\)
      \end{compactenum}
}{}

      \subsection{Semi-coercive functions}\label{section:semi-coercivity}

  % \begin{proposition}
  %   \label{proposition: \calL is coercive - main text version}
  %   If $\TP$ is semi-coercive,
  %   then
  %   $C^{\calL}_{\bfp}$ is coercive
  %   for all $\bfp \in \interior(\Delta^k)$.
  % If $\TP$ is convex, then $C^{\calL}_{\bfp}$ is convex for all $\bfp \in \Delta^k$.
  % Furthermore, if $\TP$ is strictly convex, then $C^{\calL}_{\bfp}$ is strictly convex for all $\bfp \in \interior(\Delta^k)$.
% \end{proposition}

      In this section, we study the properties of a PERM loss \(\calL\) whose template \(\TP\) is semi-coercive (Definition~\ref{definition:coercive-functions}).

  \begin{lemma}\label{lemma:semi-coercivity-of-template-implies-bounded-sublevel-set}
    Let $\calL : \mathbb{R}^{k} \to \mathbb{R}^{k}_+$ be a PERM loss whose template $\TP$ is semi-coercive. Let $\LL$ be the reduced form of $\calL$.
    Then, for all $\bmzeta \in \mathbb{R}^k$, the set
    $\{\bfz \in \mathbb{R}^{k-1} : \LL(\bfz) \preceq \bmzeta\}$ is bounded.
  \end{lemma}
  \begin{proof}
    Below, let \(\bfz \) denote an arbitrary element of \(\mathbb{R}^{k-1}\).
    In set-builder notations, we write \(\{\bfz : \mathtt{condition}\}\) to mean
    \(\{\bfz \in \mathbb{R}^{k-1} : \mathtt{condition}\}\).
    Now, by \cref{equation:relative-margin-form-summarized}, we have that
   \(\LL(\bfz) \preceq \bmzeta\)
   if and only if
\(\TP(\ico_y \bfz) \le \zeta_y\)
for all \(y \in [k]\).
    Thus, we get
    \begin{align}
      \textstyle
\{\bfz : \LL(\bfz) \preceq \bmzeta\}
      % &
        % =
% \bigcap_{y \in [k]}\{\bfz \in \mathbb{R}^{k-1}: \LL_y(\bfz) \le \zeta_y\}
%   \nonumber
% \\
% &
=
\bigcap_{y \in [k]}\{\bfz : \TP(\ico_y \bfz) \le \zeta_y\}
     % \quad \because
     % \mbox{\cref{equation:relative-margin-form-summarized}}
  % \nonumber
      % \\ &
           =\bigcap_{y \in [k]}\ico_y\left( \{\bfz : \TP(\bfz) \le \zeta_y\}\right)
\end{align}
For the last equality above, we used the fact that $\ico_{y} = \ico_{y}^{-1}$
(Lemma~\ref{lemma: rho is involutional})
and that
\[\{\bfz : \TP(\ico_y \bfz) \le \zeta_y\}
=
\ico_y^{-1}\left( \{\bfz : \TP(\bfz) \le \zeta_y\}\right)
\]
for all \(y \in [k]\).
Moreover, by the semi-coercivity assumption on \(\TP\) we have that for all \(y \in [k]\) there exists $b_y \in \mathbb{R}$ such that
$
\{\bfz : \TP(\bfz) \le \zeta_y\}
\subseteq
\{\bfz : \min \bfz \ge b_y\}
$.
Putting it all together, we have
    \begin{equation}
      \textstyle
\{\bfz : \LL(\bfz) \preceq \bmzeta\}
 \subseteq
      \bigcap_{y\in[k]}
      \ico_y(
\{\bfz : \min \bfz \ge b_y\}
) =: B
\label{equation:defining-equation-for-B}
\end{equation}
    Thus, it suffices to show that $B$ is bounded.
    Below, we prove this.

Since the empty set is bounded, we assume below that $B$ is nonempty.
    Let \(\bfz' \in B \subseteq \mathbb{R}^{k-1}\) be a fixed arbitrary point.
First, recall the infinity-norm: $\|\bfz' \|_\infty := \max \{ |z'_i| : i \in [k-1]\}$.
To prove that \(B\) is bounded, it suffices to show that there exists some number \(M\) such that \(\|\bfz'\|_{\infty} \le M\).
Note that we can express \(\|\bfz' \|_\infty\) alternatively as
\begin{equation}
  \label{equation: infinity norm alternative characterization}
  \textstyle
  \|\bfz' \|_\infty = \max \left\{ |\max \bfz'| , | \min \bfz'|\right\},
\end{equation}
where \(\max \bfz' := \max_{j \in [k-1]} z'_{j}\) and
\(\min \bfz' := \min_{j \in [k-1]} z'_{j}\).
Define $M_1 = \max \{ |b_y|: y \in [k] \}$ and $M_2 =
\max\{ |b_y + b_j| : y \in[k], j \in [k]\}$.
Finally, define
$M = \max \{M_1,M_2\}$.
We claim that
$\|\bfz'\|_\infty \le M$.

First, we note that $\min \bfz' \ge b_k$. To see this, first recall that $\ico_{k}$ is the identity (Definition~\ref{definition:matrix-label-code}).
From
\cref{equation:defining-equation-for-B}
we have
$
B \subseteq
\ico_{y}(\{\bfz : \min \bfz \ge b_y\})$ for each \(y \in [k]\).
Now, recall from Definition~\ref{definition:matrix-label-code} that
\(\ico_{k}\) is the identity matrix.
Thus in particular
$
B \subseteq
\{\bfz : \min \bfz \ge b_k\}$ and so $\min \bfz' \ge b_k$ holds.

Next, let $y \in \argmin_{j \in [k-1]} z'_{j}$ and thus \(z'_{y} = \min \bfz'\).
From \cref{equation:defining-equation-for-B}, we have $\bfz' \in \ico_{y}(\{\bfz: \min \bfz \ge b_{y}\})$.
Thus, $\ico_{y}\bfz' \in \{\bfz : \min \bfz \ge b_{y}\}$ and in particular, $[\ico_{y}\bfz']_{y} \ge b_{y}$.
Moreover, by \cref{equation: action of rho i plus 1}, we have $[\ico_{y} \bfz']_y = -z'_y = -\min \bfz'$, and thus $\min \bfz' \le - b_{y}$.
Combining with the result from the previous paragraph, we now have that $\min \bfz' \in [b_k, -b_{y}]$ and, in particular, $|\min \bfz'| \le M_1$.
Note that here, we used the fact that if \(\alpha,\beta,\gamma \in \mathbb{R}\) are such that \(\gamma \in [\alpha,\beta]\), then \(|\gamma| \le \max \{|\alpha|, |\beta|\}\).

Next, let $l \in \argmax_{j \in [k-1]} z_{j}'$ (and $y$ be still the same as before).
First consider the case when $l = y$.
Then $\bfz'$ is a constant vector and $\|\bfz'\|_\infty = |\min \bfz'|$ in which case
$\|\bfz'\|_\infty \le M_1 \le M$ holds.

Next, consider the case when $l \ne y$.
Then we have
$[\ico_{l}\bfz']_y = z'_l - z'_y = (\min \bfz') - (\max \bfz')$ by \cref{equation: action of rho i plus 1}.
Similar as in the previous case,
\(\bfz' \in
B \subseteq
\ico_{l}(\{\bfz : \min \bfz \ge b_l\})\) which implies that $[\ico_{l}\bfz']_y \ge b_{l}$.
Thus, $\max \bfz' \le \min \bfz' - b_{l} \le -(b_{y} + b_{l})$.
Furthermore, $\max \bfz' \ge \min \bfz' \ge b_k$.
Thus, we've shown that $\max \bfz' \in [b_k, - (b_{y} + b_{l})]$.
This implies that $|\max \bfz'| \le \max\{M_{1},M_{2}\}= M$, where again we used the fact that
\(\gamma \in [\alpha,\beta]\) implies \(|\gamma| \le \max \{|\alpha|, |\beta|\}\), for arbitrary \(\alpha,\beta,\gamma \in \mathbb{R}\).
To conclude, since $|\min \bfz'| \le M$, by
  Equation~\eqref{equation: infinity norm alternative characterization}, we have $\|\bfz'\|_\infty \le M$.
  \end{proof}

  \begin{proposition}\label{proposition:semi-coercive-template-implies-coercive-conditional-risk}
    Let \(\calL\) be a nonnegative PERM loss with template \(\TP\).
    If $\TP$ is semi-coercive  (Definition~\ref{definition:coercive-functions}),
    then
    $C^{\calL}_{\bfp}$ is coercive
    for all $\bfp \in \interior(\Delta^k)$.
  \end{proposition}
  \begin{proof}
    Let \(c \in \mathbb{R}\) be arbitrary and
    $S := \{ \bfz \in \mathbb{R}^{k-1}: C^{\calL}_{\bfp}(\bfz) =\langle \bfp, \LL(\bfz) \rangle \le c\}$ be the \(c\)-sublevel set.
    To show that \(C^{\calL}_{\bfp}\) is coercive, we show that \(S \subseteq \mathbb{R}^{k-1}\) is a bounded set.
    Observe that for all $\bfz \in S$ we have
    \(      \textstyle
      c \ge
      \langle \bfp, \LL(\bfz)\rangle
      =
      \sum_{y \in [k]} p_{y} \TP(\ico_{y} \bfz)
      \ge
      p_{y} \TP(\ico_{y} \bfz)
\)
for all \(y\in [k]\).
We remark that the preceding inequality crucially uses  \(\TP(\cdot) \ge 0\).
Thus,
    \(
      S \subseteq \bigcap_{y\in[k]}\{\bfz \in \mathbb{R}^{k-1}: \TP(\ico_{y} \bfz) \le c/p_{y}\}
    \), where
 the right hand side is a bounded set    (Lemma \ref{lemma:semi-coercivity-of-template-implies-bounded-sublevel-set}). Hence, $S$ is also bounded.
  \end{proof}

  \subsection{The link function}\label{section:link-function}

  In this section, we study the set of minimizers of the conditional risk of a PERM loss $\calL$, i.e., the set
  $\argmin_{\bfz \in \mathbb{R}^{k-1}} C^{\calL}_{\bfp}(\bfz)$.
  When $\calL$ is the multinomial cross entropy
  (Example~\ref{example:cross entropy}), this argmin is a singleton set for all $\bfp \in \interior(\Delta^{k})$ and the mapping from $\interior(\Delta^{k}) \ni \bfp$ to this unique minimizer recovers the logit function.

  For a general loss $\calL$, this mapping is sometimes referred to as the \emph{link function}~\citep{nowak2019general,williamson2016composite}.
  See Definition~\ref{definition: inverse link function} below.
  This section will study the properties of the link function, culminating in a sufficient condition for when the link function is a bijection (Proposition~\ref{proposition:link-function-is-bijection}).
\begin{proposition}
  \label{proposition:strictly-convex-template-implies-strictly-convex-conditional-risk}
  Let $\calL$ be a nonnegative PERM loss with template $\TP$.
  If $\TP$ is convex, then $C^{\calL}_{\bfp}$ is convex for all $\bfp \in \Delta^k$.
  Furthermore, if $\TP$ is strictly convex, then $C^{\calL}_{\bfp}$ is strictly convex for all $\bfp \in \interior(\Delta^k)$.
\end{proposition}

\begin{proof}
  Recall that $C^{\calL}_{\bfp}(\bfz) = \sum_{y\in [k]} p_{y} \TP(\ico_{y} \bfz)$ where $\ico_{y}$ is an invertible matrix by Lemma~\ref{lemma: rho is involutional}.
  Thus, if $\TP$ is (strictly) convex, then $\bfz \mapsto \TP(\ico_{y} \bfz)$ is (strictly) convex for each $y \in [k]$.
  For each $\bfp \in \Delta^{k}$, $C^{\calL}_{\bfp}$ is a convex combination of convex function and is thus convex.
  Furthermore, if $\bfp \in \interior( \Delta^{k})$, then $C^{\calL}_{\bfp}$ is a convex combination of strictly convex function and is thus strictly convex.
  See \citet[Section 3.2.1]{boyd2004convex} for instance.
\end{proof}
% \begin{proof}
%   Let $\bfz,\bfu \in \mathbb{R}^{k-1}$ be distinct points and $\lambda \in (0,1)$.
%   Then
%   \begin{align*}
%     \lambda C^{\calL}_{\bfp}(\bfz)
%     +
%     (1-\lambda) C^{\calL}_{\bfp}(\bfu)
%     &=
%     \langle \bfp, \lambda \LL(\bfz) + (1-\lambda) \LL(\bfu)\rangle
%     \\
% &    =
%     \sum_{i \in [k]} p_i
%     \left(\lambda \TP(\ico_i \bfz) + (1-\lambda) \TP(\ico_i \bfu)\right)
%     \\
% &    \overset{\star}{\ge}
%     \sum_{i \in [k]} p_i
%     \TP(\lambda \ico_i \bfz+(1-\lambda)\ico_i \bfu)
%     \\
% &=
%     \sum_{i \in [k]} p_i
%     \TP(\ico_i(\lambda \bfz+(1-\lambda)\bfu))
%     \\
% &=
% \langle \bfp , \LL(\lambda \bfz+(1-\lambda)\bfu)\rangle\\
% &=
% C^{\calL}_{\bfp}(\lambda \bfz + (1-\lambda)\bfu),
%   \end{align*}
%   where for the inequality $\star$, we used the fact that $\TP$ is convex and that $\ico_i \bfz \ne \ico_i \bfu$ since $\ico_i$ is an invertible matrix.
%   When $\TP$ is strictly convex, the inequality at $\star$ is strict.
% \end{proof}

An easy consequence of the above result is the following:

\begin{corollary}\label{corollary: inverse link map}
  Let $\bfp \in \interior(\Delta^k)$ be arbitrary and $\calL$ be a nonnegative PERM loss whose $\TP$ is semi-coercive.
  If $\TP$ is convex, then the infimum $\inf_{\bfz \in \mathbb{R}^{k-1}} C^{\calL}_{\bfp}(\bfz)$ is attained.
  Furthermore, if $\TP$ is strictly convex, then the infimum is attained by
  a \emph{unique} minimizer, i.e., $\argmin_{\bfz \in \mathbb{R}^{k-1}} C^{\calL}_{\bfp}(\bfz)$ is a singleton set.
\end{corollary}
\begin{proof}
  By Proposition~\ref{proposition:semi-coercive-template-implies-coercive-conditional-risk}, $C^{\calL}_{\bfp}$ is coercive.
  By Proposition~\ref{proposition:strictly-convex-template-implies-strictly-convex-conditional-risk}, $C^{\calL}_{\bfp}$ is strictly convex.
  By the Extreme Value Theorem, a continuous and coercive function has at least one global minimum.
  Furthermore, a {strictly} convex functions have at most one global minimum.
  For a reference of these result standards, see \citet[Section 4.2]{boyd2004convex}.
\end{proof}

In view of Corollary~\ref{corollary: inverse link map}, we define:
\begin{definition}\label{definition: inverse link function}
  Let $\calL$ be a PERM loss whose template $\TP$ is nonnegative, strictly convex and semi-coercive.
  Define the \emph{link function}
  $\gsm^{\calL}: \interior(\Delta^k) \to \mathbb{R}^{k-1}$ by letting $\gsm^{\calL}(\bfp)$ be the unique element of $\argmin_{\bfz \in \mathbb{R}^{k-1}} C^{\calL}_{\bfp}(\bfz)$.
\end{definition}

In this section, we give a sufficient condition on $\calL$ for $\gsm^{\calL}$ of Definition~\ref{definition: inverse link function} to be a \emph{bijection}.
We will need the concept of an \emph{M-matrix}, which is reviewed in Section~\ref{section:M-matrix}.

% \begin{lemma}
%   \label{proposition:generalized-link-function-is-bijection - main text version}
%   % \assumption
%   Let $L$ be a regular PERM loss and $\gsm^{\calL} : \interior(\Delta^k) \to \mathbb{R}^{k-1}$ the mapping from Definition \ref{definition: inverse link function}.
%   Then $\gsm^{\calL}$  is a bijection.
% \end{lemma}

\begin{lemma}\label{lemma:Az-matrix-is-nonsing-M-matrix}
  Let $\calL: \mathbb{R}^{k} \to \mathbb{R}^k_+$ be a regular PERM loss (Definition~\ref{definition:regular-PERM-loss}) with reduced form \(\LL\) and template \(\TP\).
  For all $\bfz \in \mathbb{R}^{k-1}$, the $(k-1)\times(k-1)$ matrix
  \[
    \bfA(\bfz) :=
    \begin{bmatrix}
      \nabla_{\LL_1}(\bfz)
      & \cdots &
      \nabla_{\LL_{k-1}}(\bfz)
    \end{bmatrix}
  \]
  is a non-singular M-matrix with strictly positive diagonal elements.
  % Thus, by Theorem~\ref{theorem: non-singular M-matrix iff monotone}, $\bfA(\bfz)$ is a monotone matrix.
  Furthermore, both $\bfA(\bfz)$ and \(\bfA(\bfz)^{\top}\) are strictly monotone.
\end{lemma}
\iftoggle{arxiv}{
  See Section~\ref{section:vector-calculus}
  for our notations on the derivative operator \(\nabla\).
}
  {
\textcolor{black}{
  While our usage of the operator \(\nabla\) is standard, for completeness we refer the reader to the ``Vector Calculus'' section in the appendix of the arXiv version of this manuscript \citep{wang2023unified}
  for our notations on the derivative operator \(\nabla\).}}
\vspace{1em}

The definition of a \emph{(strictly) monotone} matrix is discussed in Section~\ref{section:M-matrix}.
\begin{proof}
First, we compute the gradient of $\LL = (\LL_1,\dots, \LL_k) : \mathbb{R}^{k-1} \to \mathbb{R}^k$.
For each $y \in [k]$, we have by definition that $\LL_y(\bfz) = \TP(\ico_y \bfz)$ (\cref{equation:relative-margin-form-summarized}).
Thus, by the chain rule, we have
\begin{equation}
  \label{equation:loss-components-derivative-formula}
  \nabla_{\LL_y} (\bfz)
  =
  \ico_y^{\top}
\nabla_{\TP}(\ico_y \bfz).
\end{equation}
  Next, fix $y \in [k-1]$ and $\bfz \in \mathbb{R}^{k-1}$.
  Let $\bfw := \nabla_{\TP}(\ico_y \bfz)$.
  Then by the assumption that \(\nabla_{\TP}(\cdot) \prec \vzero\), we have \(\bfw\) is a entrywise negative column vector, i.e., \(\bfw \prec \vzero\).
  Moreover, by
  \cref{equation:loss-components-derivative-formula} we have
  \(
  \nabla_{\LL_y} (\bfz)
    =
    \ico_y^{\top}
    \bfw
  \).
  Thus, for each $j \in [k-1]$, we have $[\ico_y^{\top} \bfw]_j
  =([\ico_y]_{:j})^{\top} \bfw
  $.
  Recall from
Definition~\ref{definition:matrix-label-code} that
  \[[\ico_y]_{:j}
    =
  \begin{cases}
    \bfe_j^{(k-1)} &: j \ne y\\
    -\vone^{(k-1)} &: j = y
  \end{cases},
  \quad \mbox{ which implies that } \quad
  % TODO: move this definition to earlier.
 [\ico_y^{\top} \bfw]_j
    =
    \begin{cases}
      w_j &: j \ne y\\
      -\sum_{l \in [k-1]} w_l &: j = y.
    \end{cases}\]
  In particular, $[\ico_y^\top \bfw]_j \le 0$ for all $j\ne y$ which proves that $\bfA(\bfz)$ is a Z-matrix.
  Furthermore, note that the fact $\bfw < 0$ and $[\ico_y^\top \bfw]_{y} = -
  \sum_{l \in [k-1]} w_l$ implies that the diagonals of $\bfA(\bfz)$ are positive.
  Observe that $\ico_y^\top \bfw$ has the property that
  \[
    \textstyle
    |[\ico_y^\top \bfw]_{y}|
    =
    -\sum_{l \in [k-1]}
    w_l
  >
    -\sum_{l \in [k-1]: l \ne y}
    w_l
   =
    \sum_{l \in [k-1]: l \ne y}
    |[\ico_y^\top \bfw]_l|.\]
  Note that the strict inequality above follows from the fact that \(\bfw \prec \vzero\) and so in particular \(-w_{y } >0\).
  This proves that $\bfA(\bfz)$ is strictly diagonally dominant.
  By Corollary~\ref{corollary: strictly diagonally dominant Z-matrix is nonsingular M-matrix}, we have that $\bfA(\bfz)$ is a non-singular M-matrix.
  For the ``Furthermore'' part, we can apply Lemma~\ref{lemma: strictly monotone lemma} since the diagonal elements of $\bfA(\bfz)$ are positive.
  By Corollary~\ref{corollary:strictly-monotone-transpose}, \(\bfA(\bfz)^{\top}\) is also strictly monotone.
\end{proof}

\begin{lemma}\label{lemma: KKT structural lemma}
  Let $\calL: \mathbb{R}^{k} \to \mathbb{R}^k_+$ be a regular PERM loss.
  Let $\bfz \in \mathbb{R}^{k-1}$ and $\bfp \in \Delta^k$ be arbitrary. Then $\bfz$ minimizes $C^{\calL}_{\bfp}$ if and only if
  \begin{equation}
    \label{equation: KKT condition 9}
    -p_k\nabla_{\TP}(\bfz)
    =
    \bfA(\bfz)
    \begin{bmatrix}
      p_1
      &\cdots& p_{k-1}
    \end{bmatrix}^{\top}
  \end{equation}
  where \(\bfA(\bfz)\) is as in Lemma~\ref{lemma:Az-matrix-is-nonsing-M-matrix}.
  Furthermore, \(C^{\calL}_{\bfp}\) has a minimizer, then \(\bfp \in \interior(\Delta^k)\).
\end{lemma}
\begin{proof}
  Proposition~\ref{proposition:strictly-convex-template-implies-strictly-convex-conditional-risk} asserts that \(C^{\calL}_{\bfp}\) is convex.
  % $\bfp \in \interior(\Delta^k)$ be generic elements to be determined later.
  For a differentiable convex function, recall that the gradient-vanishing condition is necessary and sufficient for  optimality of unconstrained optimization~\citep{boyd2004convex}.
  Thus, \(\bfz\) minimizes \(C^{\calL}_{\bfp}\)
  iff
  \begin{equation}
    \label{equation: KKT condition}
    \textstyle
    \vzero
    =
    \nabla_{C^{\calL}_{\bfp}}(\bfz)
    = \sum_{j \in [k]} p_j
    \nabla_{\LL_j} (\bfz)
    =
    p_k
    \nabla_{\TP} (\bfz)
    +
    \bfA(\bfz)
    \begin{bmatrix}
      p_1 & \cdots & p_{k-1}
    \end{bmatrix}^{\top}.
  \end{equation}
  which is equivalent to
    \cref{equation: KKT condition 9}.

  Next, for the ``Furthermore'' part, first note that
  Lemma~\ref{lemma:Az-matrix-is-nonsing-M-matrix} says $\bfA(\bfz)$ is a non-singular M-matrix.
  Now, we first show that \(p_{k} \ne 0\). If $p_k = 0$, then Equation~\eqref{equation: KKT condition 9}
  reduces to \(
    \vzero
    =
    \bfA(\bfz)
    \begin{bmatrix}
      p_1 & \cdots & p_{k-1}
    \end{bmatrix}^{\top}\).
  Since $\bfA(\bfz)$ is non-singular, we must have $p_1 = \cdots = p_{k-1} = 0$ as well which implies that \(\bfp = \vzero\). But  this contradicts that $\bfp \in \Delta^k$.
  Thus, $p_k > 0$ and so $-p_k\nabla_{\TP}(\bfz)  \succ 0$.
  From Lemma \ref{lemma:Az-matrix-is-nonsing-M-matrix}, we have that $\bfA(\bfz)$ is strictly monotone.
  Combined with
    \cref{equation: KKT condition 9}, we can conclude  that $p_y  > 0$ for each \(y \in [k-1]\) as well.
\end{proof}

% \kificationap

The ``Furthermore'' part of
  Lemma \ref{lemma: KKT structural lemma}
  immediately implies the following.

\begin{corollary}
  \label{corollary: boundary p has no argmin}
  If $\bfp \in \Delta^k \setminus \interior(\Delta^k)$, then $\argmin_{\bfz \in \mathbb{R}^{k-1}} C^{\calL}_{\bfp}(\bfz) = \emptyset$.
\end{corollary}

\begin{proposition}\label{proposition:link-function-is-bijection}
  % \assumption
  Let $\calL$ be a regular PERM loss.
  % Recall the mapping $\gsm^{\calL} : \interior(\Delta^k) \to \mathbb{R}^{k-1}$ from Definition~\ref{definition: inverse link function}.
  Then $\gsm^{\calL}$ (Definition~\ref{definition: inverse link function})  is a bijection.
\end{proposition}
\begin{proof}
  First, we prove that $\gsm$ is injective. Suppose that $\bfp,\bfq \in \interior(\Delta^k)$ are such that $\gsm^{\calL}(\bfp) = \gsm^{\calL}(\bfq) =: \bfz$.
  Then by
    \cref{equation: KKT condition 9}, we have that
  \begin{equation}
    \label{equation: KKT condition 4}
    -\nabla_{\TP}(\bfz)\bfA(\bfz)^{-1}
    =
    p_k^{-1}
    \begin{bmatrix}
      p_1 &\cdots& p_{k-1}
    \end{bmatrix}^{\top}
    =
    q_k^{-1}
    \begin{bmatrix}
      q_1 &\cdots& q_{k-1}
    \end{bmatrix}^{\top}.
  \end{equation}
  Thus, $(1-p_{k})/p_{k} =(p_{1}+\cdots + p_{k-1}) /p_{k} =
  (q_{1}+\cdots + q_{k-1}) /q_{k}
  = (1-q_{k})/q_{k}$ implies that $p_{k} = q_{k}$.
  Therefore, \cref{equation: KKT condition 4} implies that $p_{y} = q_{y}$ for each \(y \in [k-1]\) as well.
  Thus, \(\bfp = \bfq\) which proves that $\gsm$ is injective.

  Next, we prove that $\gsm$ is surjective. Pick $\bfz \in \mathbb{R}^{k-1}$.
  From Lemma~\ref{lemma:Az-matrix-is-nonsing-M-matrix}, we have that $\bfA(\bfz)$ is non-singular and strictly monotone.
  By non-singular-ness, there exists $\bfv \in \mathbb{R}^{k-1}$ such that $-\nabla_{\TP}(\bfz)
  =
  \bfA(\bfz)^{\top}
  \bfv
  $.
  Furthermore, since $-\nabla_{\TP}(\bfz) \succ 0$ and $\bfA(\bfz)$ is strictly monotone, we have $\bfv \succ 0$.
  Define $p_1,\dots, p_k$ by $p_k := (v_1+ \dots + v_{k-1}+1)^{-1}$ and
  $
    p_y := v_{y}p_k
  $
  for each \(y \in [k-1]\).
  Clearly, we have $\bfp \succ 0$. Furthermore,
  \[p_1 + p_2 + \dots + p_k =
    p_k(v_1 + \cdots + v_{k-1}+1) = 1.
\]
Thus, we have $\bfp \in \interior(\Delta^k)$.
By construction, $\bfz$ and $\bfp$ satisfy \cref{equation: KKT condition 9}. Thus $\gsm^{\calL}(\bfp) = \bfz$ follows from the definition of \(\gsm^{\calL}\).
\end{proof}

\begin{remark}
 Before proceeding, we remark that Proposition~\ref{proposition:link-function-is-bijection} gives theoretical support to the conjectural observation in \citet[Remark 3.1]{nowak2019general} regarding the injectivity of the \emph{link function}.
\end{remark}

% \begin{remark}[Connection to proper composite loss]
%   A ($k$-ary) \emph{probability-estimation loss} is
% a function $\Lambda : \Delta^{k} \to [0,+\infty]^{k}$. We say that $\Lambda$ is \emph{strictly proper}
% if
% \[
%   \argmin_{\bfq \in \Delta^{k}} \bfp^{\top} \Lambda(\bfq) = \{\bfp\}
% \]
% for all $\bfp \in \Delta^{k}$.
% A multiclass loss function $\calL : \mathbb{R}^{k} \to \mathbb{R}^{k}$ is a \emph{proper composite loss}~\citep[Page 6]{williamson2016composite} if $\calL = \Lambda \circ \Phi$ for some $\Phi : \mathbb{R}^{k} \to \Delta^{k}$.
% Now, let $\calL$ is as in
% Proposition~\ref{proposition:link-function-is-bijection}
% and
% $\Lambda^{\calL} := \calL \circ \gsm^{\calL}$.
% This is because $\gsm^{\calL}(\bfp)$ is the unique minimizer of $\calL$
% by the construction of $\gsm$ in Definition~\ref{definition: inverse link function}.
% Thus, $\bfp$ is the unique minimizer of $\bfq \mapsto \calL \circ \gsm^{\calL}(\bfq)$ since $\gsm^{\calL}$ is a bijection.
% Thus, $\calL = (\calL \circ \gsm^{\calL}) \circ (\gsm^{\calL})^{-1}$ expresses $\calL$ as a
% Proper composite losses have important implications in
% \end{remark}

\subsection{Geometry of the loss surface}\label{section:geometry-of-loss-surface}

Recall from Definition~\ref{definition:loss-surface} and
Theorem~\ref{theorem: multiclass classification calibration} that
the classification-calibration of the \emph{set} $\cran(\calL)$ implies the
classification-calibration of the \emph{loss} $\calL$.
In general, the set $\cran(\calL)$ may be difficult to compute.
In this section, we study the geometry of the set $\ran(\calL)$ when $\calL$ is a regular PERM loss which enables us to compute the convex hull $\cran(\calL)$ of $\ran(\calL)$.
One of the main tools is the mapping defined below:

\begin{corollary}
  \label{corollary: L is injective}
  % \assumption
  Let $\calL$ be a regular PERM loss with reduced form $\LL$.
  Then $\LL: \mathbb{R}^{k-1} \to \mathbb{R}^k$ is injective.
\end{corollary}
\begin{proof}
  Suppose that $\bfz, \bfw \in \mathbb{R}^{k-1}$ are such that $\LL(\bfz) = \LL(\bfw)$.
  By Proposition~\ref{proposition:link-function-is-bijection}, there exists $\bfp \in \interior(\Delta^k)$ such that $\gsm^{\calL}(\bfp) = \bfz$.
  Now, $\langle \bfp, \LL(\bfz)\rangle =
  \langle \bfp, \LL(\bfw)\rangle$ implies that both $\bfz,\bfw$ minimize $C^{\calL}_{\bfp}$.
  By Corollary \ref{corollary: inverse link map}, we have $\bfz = \bfw$ and so $\LL$ is injective.
\end{proof}

\begin{definition}\label{definition:loss-surface-foliation-F}
  Given a PERM loss $\calL$ with reduced form $\LL$, we define two functions $F$ and $G$ mapping from $\mathbb{R}^{k-1} \times \mathbb{R}$ to $\mathbb{R}^k$ by
  \(    F(\bfz,\lambda) = \LL(\bfz) + \lambda \vone
  \) and \(
    G(\bfz,t) = \LL(\bfz) + t\bfe^{(k)}_{k}
\), where \(\bfz \in \mathbb{R}^{k-1}\) and \(\lambda \in \mathbb{R}\).
For computing \(\nabla_{F}\),
 we view the tuple \(\bfz, \lambda\) as a (column) vector
\(
\begin{bmatrix}
  \bfz^{\top}
  &
  \lambda
\end{bmatrix}^{\top}
\).
Hence,
\(\nabla_{F} ( \bfz, \lambda ) =
\begin{bmatrix}
  \nabla_{\LL}(\bfz)^{\top} & \vone^{\top}
\end{bmatrix}^{\top}
\).
Likewise for the tuple \(\bfz, t\) and \(\nabla_{G}\).
\end{definition}

\begin{remark}
  In the context of Definition~\ref{definition:loss-surface-foliation-F}, it is perhaps more precise to write \(F(
\begin{bmatrix}
  \bfz^{\top}
  &
  \lambda
\end{bmatrix}^{\top}
  )\). However, this notation is cumbersome. Below, we will always use the tuple notation.
  The reason we discuss the vector notation is so that  \(\nabla F\) and \(\nabla G\) can be more simply treated as gradients of vector input-valued functions (rather than matrix input-valued).
\end{remark}

Below, we will study the properties of the two functions from Definition~\ref{definition:loss-surface-foliation-F}.

\subsubsection{Properties of the $F$ function}

% \begin{definition}\label{definition:loss-surface-foliation-G}
%   Given a PERM loss $\calL$ with reduced form $\LL$, we define $G: \mathbb{R}^{k-1} \times \mathbb{R} \to \mathbb{R}^k$ by
%   \[
%   \]
% \end{definition}

The main result of this section is to prove that \(F\) is a homeomorphism
from \(\mathbb{R}^{k-1} \times \mathbb{R}\) to \(\mathbb{R}^{k}\) (Corollary~\ref{corollary:F-is-a-homeomorphism}).

\begin{lemma}\label{lemma: foliation map is injective}
  Let $\calL$ be a regular PERM loss with reduced form $\LL$ and $F$ be as in Definition~\ref{definition:loss-surface-foliation-F}. Then $F$ is injective. In other words,
  if $\LL(\bfz) + \lambda \vone = \LL(\bfw) + \mu \vone$ for some $\bfz,\bfw\in \mathbb{R}^{k-1}$ and $\lambda,\mu \in \mathbb{R}$, then both $\bfz = \bfw$ and $\lambda = \mu$.
\end{lemma}
\begin{proof}
Let $\bfz,\bfw\in \mathbb{R}^{k-1}$ and $\lambda,\mu \in \mathbb{R}$ be as in the statement of the lemma.
Our goal is to show that both $\bfz = \bfw$ and $\lambda = \mu$.
  First, consider the case that $\bfz = \bfw$. Then $\LL(\bfz) = \LL(\bfw)$ and so $\lambda = \mu$.
  Thus, $\bfz = \bfw$ implies $\lambda = \mu$.

  Next, consider the case  that $\lambda = \mu$. Then we have $\LL(\bfz) = \LL(\bfw)$.
  By Corollary \ref{corollary: L is injective}, we have $\bfz = \bfw$.
  Therefore, $\lambda = \mu$ implies $\bfz = \bfw$.

  Thus, it only remains to show that if both $\bfz \ne \bfw$ and $\lambda \ne \mu$, then we have  a contradiction.
  Without loss of generality, suppose that $\lambda > \mu$.
  Then we have $\LL(\bfz) + (\lambda - \mu) \vone = \LL(\bfw)$.
  Thus, for all $\bfp \in \Delta^k$, we have
  \(    \langle \bfp , \LL(\bfw)\rangle
    =
    \langle \bfp , \LL(\bfz) + (\lambda - \mu) \vone \rangle
    >
    \langle \bfp , \LL(\bfz) \rangle.
\)
  Thus, there does not exist $\bfp \in \Delta^k$ such that $\bfw$ is  the minimizer of $C^{\calL}_{\bfp}$.
  But this contradicts since
  Proposition~\ref{proposition:link-function-is-bijection} implies that
  $\gsm$ is surjective.
\end{proof}

\begin{lemma}\label{lemma:gradient-of-F-is-nonsingular}
  Let $\calL$ be a regular PERM loss with reduced form $\LL$. Let $F$ be as in Definition~\ref{definition:loss-surface-foliation-F}.
  Then for all $(\bfz,\lambda) \in \mathbb{R}^{k-1} \times \mathbb{R}$, $\nabla_{F}(\bfz,\lambda)$ is non-singular. % and monotone.
\end{lemma}
  \begin{proof}
    % First, we check that that $\nabla_{F}(\bfz,\lambda)$ is non-singular.
    Recall that
\(    \nabla_{\LL}(\bfz) \in \mathbb{R}^{(k-1)\times k}
\)
and
\(
     \vone^{(k)}  \in \mathbb{R}^{k}
\).
Thus,
  \(\nabla_{F} ( \bfz, \lambda ) =
\begin{bmatrix}
  \nabla_{\LL}(\bfz)^{\top} & (\vone^{(k)})^{\top}
\end{bmatrix}^{\top}
\) is a \(k\times k\) square matrix.
To show that it is non-singular, first pick $\bfv \in \mathbb{R}^k$ arbitrary.
    It suffices to check that if $\nabla_{F}(\bfz,\lambda) \bfv =\vzero$ then $\bfv =\vzero$.

 Towards this, first note that $\nabla_{F}(\bfz,\lambda)\bfv =\vzero$ can be equivalently stated as
 both
  \(
  \nabla_{\LL}(\bfz)\bfv = 0
  \) and
  \((\vone^{(k)})^{\top} \bfv = v_1 + \cdots + v_k = 0\).
  Replacing \(\bfv\) by \(-\bfv\) if necessary, we assume that \(v_k \ge 0\).
  Recall \(\bfA(\bfz)\) from
Lemma~\ref{lemma:Az-matrix-is-nonsing-M-matrix}.
  Then the identity \( \nabla_{\LL}(\bfz)\bfv = \vzero\) can be rewritten as
    \(% \label{equation: KKT condition 5}
      -v_k \nabla_{\TP}(\bfz) =
      \bfA(\bfz)
\begin{bmatrix}
      v_1 &\cdots& v_{k-1}
    \end{bmatrix}^{\top}\).
  Since $\bfA(\bfz)$ is monotone and $-v_k \nabla_{\TP}(\bfz) \succeq 0$, we get that $v_{y} \ge 0$ for each \(y \in [k-1]\).
  Combined with the fact that
  \(\bfv^{\top} \vone = v_1 + \cdots + v_k = 0\), we have that $\bfv = \vzero$, as desired.
  \end{proof}

  Now, by applying the \iftoggle{arxiv}{inverse function theorem (Theorem~\ref{theorem:inverse-function-theorem}),}
  {\textcolor{black}{inverse function theorem\footnote{The inverse function theorem is a standard result in multivariate calculus. See the arXiv version of this manuscript \citep{wang2023unified} for the result's statement and a textbook reference.},}
} we immediately have the following.
  \begin{corollary}\label{lemma: F is a local diffeomorphism}
  Let $\calL$ be a regular PERM loss with reduced form $\LL$. Let $F$ be as in Definition~\ref{definition:loss-surface-foliation-F}.
    For all $(\bfz,\lambda) \in \mathbb{R}^{k-1} \times \mathbb{R}$, there exist open neighborhoods $U \ni (\bfz,\lambda)$ and $V \ni F(\bfz,\lambda)$ such that $F|_U: U \to V$ is a diffeomorphism.
  \end{corollary}

\begin{proposition}\label{proposition:F-is-a-bijection}
  Let $\calL$ be a regular PERM loss with reduced form $\LL$. Let $F$ be as in Definition~\ref{definition:loss-surface-foliation-F}.
  The map $F$ is a bijection.
\end{proposition}
\begin{proof}
  Lemma \ref{lemma: foliation map is injective} shows that $F$ is injective.
  To show that $F$ is surjective, we prove that $\ran(F)$ is both open and closed as a subset of $\mathbb{R}^k$.
  This would imply that $\ran(F) = \mathbb{R}^k$
  since the only subets of $\mathbb{R}^{k}$ that are both open and closed are $\emptyset$ and $\mathbb{R}^{k}$.
  Now, Lemma \ref{lemma: F is a local diffeomorphism} shows that $\ran(F)$ is an open subset of $\mathbb{R}^k$.
  It remains to prove that $\ran(F)$ is closed.

  To this end, consider a sequence $\{(\bfz^{(i)},\lambda^{(i)})\}_{i=1}^{\infty}$ such that $F(\bfz^{(i)},\lambda^{(i)}) = \LL(\bfz^{(i)}) + \lambda^{(i)} \vone$ converges to $\bmzeta \in \mathbb{R}^k$.
  Our goal is to show that $\bmzeta \in \ran(F)$.

 We begin by first picking $\epsilon > 0$.
 Let $\vone := \vone^{(k)}$ (without the superscript \((k)\)) denotes the $k$-dimensional vector of all ones. Since the sequence converges, there exists $M$ such that
 \[\bmzeta - \epsilon \vone \preceq \LL(\bfz^{(i)}) + \lambda^{(i)} \vone \preceq \bmzeta + \epsilon \vone\]
 for all $i \ge M$. Before proceeding, we prove a helper lemma.

  \begin{lemma}[Helper lemma]\label{lemma:helper-lemma-for-zero-vector}
    Let $\calL$ be a PERM loss with reduced form $\LL$ and template $\TP$ such that $\nabla_{\TP}(\cdot) \preceq 0$. Then for all $\bfz \in \mathbb{R}^{k-1}$, we have that
    $\min_{y \in [k-1]}\LL_{y}(\bfz) \le \TP(\vzero^{(k-1)})=: C$, where $\vzero^{(k-1)}$ is the $(k-1)$-dimensional all-zeros vector.
  \end{lemma}
  \begin{proof}[Proof of helper lemma]
    Let $\bfv \in \mathbb{R}^{k}$ be such that $\bmpi\bfv = \bfz$, where we recall from Definition~\ref{definition:relative-marginalization-mapping} that \(\bmpi =
    \begin{bmatrix}
                    -\mathbf{I}_{k-1}
      & \vone^{(k-1)}
    \end{bmatrix}
    \). As in Remark~\ref{remark:relative-margin-form-summarized}, we can take $\bfv =
    \begin{bmatrix}
      -\bfz^{\top} & 0 \end{bmatrix}^{\top}$.
    Next, let $y \in \argmax_{j\in[k-1]} v_{j}$.
Recall from Section~\ref{section:appendix:full-notations} that \(\min(\cdot)\) over vector-valued inputs denotes entrywise minimum.
    Then by Remark~\ref{remark:well-incentivized-loss-and-min-max-identity} and Proposition~\ref{proposition:sufficient-condition-for-well-incentivized-ness},
 we have that \(\min(\calL(\bfz)) = \calL_{y}(\bfv)\).
    Next by \cref{equation:relative-margin-form-summarized}  from Remark~\ref{remark:relative-margin-form-summarized}, we have
\(    [\calL(\bfv)]_y = \LL_{y}(\bfz).
\)

    Let $\bfw := \sigma_{y}(\bfv)$.
    (Recall from Section~\ref{section:appendix:full-notations} that $\sigma_{y} \in \mathtt{Sym}(k)$ is the transposition that swaps $k$ and $y$.)
    Note that by construction we have $k \in \argmax \bfw$.
    By permutation-equivariance, we have
\(    [\calL(\bfv)]_{y}
    =
    [\calL(\bfv)]_{\sigma_{y}(k)} = [\calL(\sigma_{y}(\bfv))]_{k}
    =
    [\calL(\bfw)]_{k}.
\)
    Again by \cref{equation:relative-margin-form-summarized}  from Remark~\ref{remark:relative-margin-form-summarized}, we have
\(    [\calL(\bfw)]_{k}
    =
    \TP(\bmpi\bfw)
    \).

    Since $k \in \argmax \bfw$, we have that $\bmpi\bfw \in \mathbb{R}^{k-1}_{\ge 0}$ belongs to the non-negative orthtant.
    In other words, $\bmpi\bfw \ge \vzero^{(k-1)}$.
    Finally, since $\nabla_{\TP}(\cdot) \preceq 0$, we have that $\TP(\bmpi\bfw) \le \TP(\vzero^{(k-1)})$.
  \end{proof}

  We now return to the proof of the proposition.
  Let $C$ be as in the helper lemma.
  Then
  \[
    \min(\LL(\bfz^{(i)}) + \lambda^{(i)} \vone)
  = \min(\LL(\bfz^{(i)})) + \lambda^{(i)}
  \le
  C + \lambda^{(i)}.
\]
Thus, we have
\(  \min(\bmzeta) - \epsilon
  \le
  C + \lambda^{(i)}
\)
and which implies that $-\lambda^{(i)} \le C + \epsilon - \min(\bmzeta) =: D$.
From this, we get that
\[
  \LL(\bfz^{(i)}) \preceq \bmzeta + \epsilon\vone -\lambda^{(i)}\vone
  \preceq
  \bmzeta + (\epsilon + D)\vone
\]
Thus, for all \(i \ge M\) we have $\bfz^{(i)} \in \{\bfz \in \mathbb{R}^{k-1} : \LL(\bfz) \preceq
\bmzeta + (\epsilon + D)\vone\}$
which is a bounded set by
Lemma~\ref{lemma:semi-coercivity-of-template-implies-bounded-sublevel-set}.
By passing to a subsequence, we may assume that $\bfz^{(i)}$ converges to some $\bfz^* \in \mathbb{R}^{k-1}$.
Thus, we have
$\lambda^{(i)} \vone$ converges to $\LL(\bfz^*) + \bmzeta$, which implies in particular that $\lambda^{(i)}$ converges to some $\lambda^*$.
Putting it all together, we have shown that $F(\bfz^{(i)},\lambda^{(i)})$ converges to $\bmzeta = F(\bfz^*,\lambda^*)$ and so $\ran(F)$ is closed.
\end{proof}

\begin{corollary}\label{corollary:F-is-a-homeomorphism}
  Let $\calL$ be a regular PERM loss with reduced form $\LL$. Let $F$ be as in Definition~\ref{definition:loss-surface-foliation-F}.
  The map $F$ is a diffeomorphism, i.e., $F$ is a differentiable bijection with a differentiable inverse.
  In particular, $F$ is a homeomorphism.
\end{corollary}
\begin{proof}
  \citet[Proposition 4.6 (f)]{lee2013smooth} states that every bijective local diffeomorphism is a (global) diffeomorphism.
  Thus, the result follows in view of the facts that
  $F$ is a bijection (Proposition~\ref{proposition:F-is-a-bijection})
  and that $F$ is a local diffeomorphism (Corollary~\ref{lemma: F is a local diffeomorphism}).
\end{proof}

\begin{proposition}\label{proposition: concavity of restriction to lines}
  Let $\calL$ be a regular PERM loss with reduced form $\LL$. Let $F$ be as in Definition~\ref{definition:loss-surface-foliation-F}.
  Consider arbitrary $\bfv,\bfx \in \mathbb{R}^k$ and $t \in \mathbb{R}$. Define\footnote{
Note that by Corollary~\ref{corollary:F-is-a-homeomorphism}, such \(\alpha(t)\) and \(\beta(t)\) exist and are unique with respect to this property.
  } $\alpha(t) \in \mathbb{R}^{k-1}$ and $\beta(t) \in \mathbb{R}$ to be such that
$t\bfv + \bfx = F(\alpha(t),\beta(t)) = \LL(\alpha(t)) + \beta(t) \vone^{(k)}$.
Then for all $t \in \mathbb{R}$, we have
\begin{enumerate}
  \item
$\alpha : \mathbb{R} \to \mathbb{R}^{k-1}$ and $\beta : \mathbb{R} \to \mathbb{R}$ are differentiable,
\item If $\bfv \succ 0$, then $\beta'(t) =\frac{d\beta}{dt}(t) > 0$,
\item $\beta$ is concave, i.e., \(\beta''(t)=\frac{d^{2} \beta}{d t^{2}}(t) \le 0\).
\end{enumerate}
\end{proposition}
\begin{remark}
  We note that \(\alpha\) and \(\beta\) in
  Proposition~\ref{proposition: concavity of restriction to lines} implicitly depend on
  \(\bfx\) and \(\bfv\).
\end{remark}
\begin{proof}
  Below, let \(\vone := \vone^{(k)}\) be the \(k\)-dimensional vector of all ones.
  To prove the first part, first note that $(\alpha(t),\beta(t)) = F^{-1}(t\bfv + \bfx)$.  Hence, $\alpha$ and $\beta$ are differentiable.

  Next, we prove the second part.
Let \(y \in [k]\) be arbitrary, to be specified later.
Now, the \(y\)-th coordinate of \(t\bfv + \bfx = \LL(\alpha(t)) + \beta(t) \vone\) is
\begin{equation}
  \label{equation: implicit diff}
  v_y t + x_y = \LL_y(\alpha(t)) + \beta(t).
\end{equation}
Differentiating \eqref{equation: implicit diff} on both sides with respect to \(t\) and applying the
  \iftoggle{arxiv}{chain rule (Theorem~\ref{theorem:chain-rule}),}
  {\textcolor{black}{chain rule\footnote{The chain rule is a standard result in multivariate calculus. See the arXiv version of this manuscript \citep{wang2023unified} for the result's statement and a textbook reference.},}} we get
\begin{equation}
  \label{equation: implicit diff 2}
  v_{y} =  \alpha'(t)^{\top}\nabla_{\LL_{y}}(\alpha(t))+ \beta' (t).\end{equation}
For convenience, we adopt the notation \(f'(\cdot) := (\nabla_{f}(\cdot))^{\top}\) for functions  with a scalar input.
\iftoggle{arxiv}{
See Remark~\ref{remark:time-derivatives}.
}{}
Now, we claim that $\alpha'(t)^{\top}\nabla_{\LL_y}(\alpha(t)) \le 0$ for some \(y \in [k]\). In fact, we prove this claim by proving a slightly more general statement that will be used again later.

  \begin{lemma}\label{lemma: useful lemma for proving convexity}
  Let $\calL$ be a regular PERM loss with reduced form $\LL$.
    Let $\bfz,\bfw \in \mathbb{R}^{k-1}$ be arbitrary.
    Then there exists $y \in [k]$ such that $\bfw^{\top}\nabla_{\LL_y}(\bfz) \le 0$.
  \end{lemma}
  \begin{proof}
    Suppose for the sake of contradiction that \(\bfw^{\top}\nabla_{\LL_y}(\bfz)> 0\) for all \(y \in [k]\).
    Then
  \[
    0 \prec
    \begin{bmatrix}
  \nabla_{\LL_{1}}(\bfz)&
                          \cdots
                          &
  \nabla_{\LL_{k-1}}(\bfz)&
  \nabla_{\LL_k}(\bfz)
    \end{bmatrix}^{\top}
    \bfw
    =
    \begin{bmatrix}
  \bfA(\bfz) &
  \nabla_{\TP}(\bfz)
    \end{bmatrix}^{\top}\bfw.
  \]
  In other words, we have $\bfA(\bfz)^{\top} \bfw \succ 0$ and $\bfw^{\top}\nabla_{\TP}(\bfz)> 0$.
By Lemma~\ref{lemma:Az-matrix-is-nonsing-M-matrix}, \(\bfA(\bfz)^{\top}\)
is strictly monotone .
Thus, \(\bfA(\bfz)^{\top}\bfw \succ \vzero \) implies\footnote{This is the  definition of ``strictly monotone matrices''. See Section~\ref{section:M-matrix}}
\(\bfw \succ \vzero\).
  But $\nabla_{\TP}(\bfz) \prec \vzero$ by assumption that $\calL$ is a regular PERM loss.
  Hence, $\bfw^{\top}\nabla_{\TP}(\bfz) < 0$, a contradiction.
\end{proof}
Applying Lemma~\ref{lemma: useful lemma for proving convexity} with $\bfw = \alpha'(t)$, we get the desired claim.
Now, pick $y \in [k]$ such that $\alpha'(t)^{\top}\nabla_{\LL_y}(\alpha(t)) \le 0$.
Thus,
from Eqn.~\eqref{equation: implicit diff 2} we have
\(\beta'(t) = v_{y} - \alpha'(t)^{\top}\nabla_{\LL_{y}}(\alpha(t))  > 0\).
This proves the second part of  Proposition~\ref{proposition: concavity of restriction to lines}.

% \begin{lemma}[Product rule]
%   Let \(f\) and \(g: \mathbb{R} \to \mathbb{R}^{n}\)  be differentiable functions of a single variable \(t\).
%   Then
%   \(\frac{d}{dt} (f(t)^{\top}g(t))
%   =
%   f'(t)^{\top} g(t)
%   +
%   f(t)^{\top} g'(t)
%   \).
% \end{lemma}

Finally, we prove the last part of Proposition~\ref{proposition: concavity of restriction to lines}.
Pick \(y \in [k]\) (unrelated to the earlier choice) such that
\(
\alpha''(t)^{\top}
    \nabla_{\LL_y}(\alpha(t)) \ge 0
    \). Such a \(y \in [k]\) exists by Lemma~\ref{lemma: useful lemma for proving convexity} by setting \(\bfw = -\alpha''(t)\).
    Differentiating \cref{equation: implicit diff 2} with respect to \(t\), we get by the
    \iftoggle{arxiv}{product rule for curves (see \cref{equation:product-rule-for-curves} from Remark~\ref{remark:time-derivatives})}
  {\textcolor{black}{
product rule for curves\footnote{
  See the ``Vector Calculus'' section in the appendix of the arXiv version of this manuscript \citep{wang2023unified}.
  }}} that
\[0
  =
\alpha''(t)^{\top}
    \nabla_{\LL_y}(\alpha(t))
    +
\alpha'(t)^{\top}
    \tfrac{d}{dt}(\nabla_{\LL_{y}}(\alpha(t)))
  + \beta'' (t)\]
By
the
\iftoggle{arxiv}{chain rule for curves (see
\cref{equation:chain-rule-for-curves} from Remark~\ref{remark:time-derivatives})}
  {\textcolor{black}{
chain rule for curves\footnote{
  See the ``Vector Calculus'' section in the appendix of the arXiv version of this manuscript \citep{wang2023unified}.
  }}}, we have
\(  \tfrac{d}{dt}(\nabla_{\LL_{y}}(\alpha(t))) =
  \nabla_{\LL_{y}}^{2}(\alpha(t))
  \alpha'(t)
\) which combined with the above implies
\[
  - \beta'' (t)
  =
\alpha''(t)^{\top}
    \nabla_{\LL_y}(\alpha(t))
    +
\alpha'(t)^{\top}
  \nabla_{\LL_{y}}^{2}(\alpha(t))
  \alpha'(t).
\]
Since \(\LL_{y}\) is convex, \(
  \nabla_{\LL_{y}}^{2}(\alpha(t))
\)
is positive semidefinite. Therefore,
\(
\alpha'(t)^{\top}
  \nabla_{\LL_{y}}^{2}(\alpha(t))
  \alpha'(t)
  \ge 0\).
  Putting it all together, \(-\beta''(t) \ge 0\), i.e., \(\beta\) is concave.
\end{proof}

\subsubsection{Properties of the $G$ function}

\begin{lemma}\label{lemma:gradient-of-G-is-nonsingular}
  Let $\calL$ be a regular PERM loss with reduced form $\LL$.
  Let $G$ be as in Definition~\ref{definition:loss-surface-foliation-F}.
  Then for all $(\bfz,t) \in \mathbb{R}^{k-1} \times \mathbb{R}$, the gradient (matrix) $\nabla_{G}(\bfz,t)$ is non-singular. % and monotone.
\end{lemma}
\begin{proof}
  The proof proceeds similarly as in Lemma~\ref{lemma:gradient-of-F-is-nonsingular}.
    Recall that
\(    \nabla_{\LL}(\bfz) \in \mathbb{R}^{(k-1)\times k}
\)
and
\(
     \bfe_{k}^{(k)}  \in \mathbb{R}^{k}
\).
Below, we suppress this superscript and simply write \(\bfe_{k}:= \bfe_{k}^{(k)}\)
Thus,
  \(\nabla_{G} ( \bfz, \lambda ) =
\begin{bmatrix}
  \nabla_{\LL}(\bfz)^{\top} & \bfe_{k}^{\top}
\end{bmatrix}^{\top}
\) is a \(k\times k\) square matrix.
To show that it is non-singular, first pick $\bfv \in \mathbb{R}^k$ arbitrary.
    It suffices to check that if $\nabla_{G}(\bfz,\lambda) \bfv =\vzero$ then $\bfv =\vzero$.

 Towards this, first note that $\nabla_{G}(\bfz,\lambda)\bfv =\vzero$ can be equivalently stated as
 both
  \(
  \nabla_{\LL}(\bfz)\bfv = 0
  \) and
  \(\bfe_{k}^{\top} \bfv = v_k = 0\).
  Recall \(\bfA(\bfz)\) from
Lemma~\ref{lemma:Az-matrix-is-nonsing-M-matrix}.
  Then the identity \( \nabla_{\LL}(\bfz)\bfv = \vzero\) can be rewritten as
    \(% \label{equation: KKT condition 5}
      \vzero=-v_k \nabla_{\TP}(\bfz) =
      \bfA(\bfz)
\begin{bmatrix}
      v_1 &\cdots& v_{k-1}
    \end{bmatrix}^{\top}\).
  Since $\bfA(\bfz)$ is non-singular, we get that $v_{y} = 0$ for each \(y \in [k-1]\).
Thus, we have that $\bfv = \vzero$, as desired.
\end{proof}

\begin{lemma}\label{lemma:bottom-right-entry-nonnegativity}
  Suppose that \(\bfA \in \mathbb{R}^{n\times n}\)  and \(\bfv \in \mathbb{R}^{n}\) are such that
  \emph{1)} \(\bfA\) is strictly monotone (Definition~\ref{definition:M-matrix}),
  \emph{2)} \(\bfv \prec \vzero\) has strictly negative entries,
  \emph{3)} the matrix
\(  \left[\begin{smallmatrix}
      \bfA & \bfv \\
      \vone^{\top} & 1
  \end{smallmatrix}\right]
\)is invertible, and \emph{4)} \(\mathbf{M} :=
\left[\begin{smallmatrix}
      \bfA & \bfv \\
      \vzero^{\top} & 1
  \end{smallmatrix}\right]
\left(\left[\begin{smallmatrix}
      \bfA & \bfv \\
      \vone^{\top} & 1
  \end{smallmatrix}\right]\right)^{-1} \in \mathbb{R}^{(n+1)\times(n+1)}
\) is also invertible. Then \(M_{n+1,n+1} > 0\), i.e., the bottom right entry of \(\mathbf{M}\) is positive. \end{lemma}
\begin{proof}
Define \(\mathbf{B} \in \mathbb{R}^{n\times n}\), \(\bfw, \bfu \in \mathbb{R}^{n}\) and \(c \in \mathbb{R}\) such that
\(
\left(
  \left[\begin{smallmatrix}
      \bfA & \bfv \\
      \vone^{\top} & 1
        \end{smallmatrix}\right]
    \right)^{-1}
=
\left[\begin{smallmatrix}
      \mathbf{B} & \bfw \\
      \bfu^{\top} & c
  \end{smallmatrix}\right]
\).
Then by definition of the matrix inverse, we have
\begin{equation}
  \begin{bmatrix}
    \mathbf{I}_{n} & \vzero \\
    \vzero^{\top} & 1
  \end{bmatrix}
  =
  \begin{bmatrix}
      \bfA & \bfv \\
      \vone^{\top} & 1
        \end{bmatrix}
\begin{bmatrix}
      \mathbf{B} & \bfw \\
      \bfu^{\top} & c
  \end{bmatrix}
  =
  \begin{bmatrix}
  \mathbf{A} \mathbf{B} + \bfv  \bfu^{\top}  & \bfA \bfw + c \bfv
    \\
    \vone^{\top} \mathbf{B} + \bfu^{\top} & \vone^{\top} \bfw + c
  \end{bmatrix}.
\label{equation:inverse-of-gradient-of-F}
\end{equation}
Now, we observe that
\[
  \mathbf{M}=
  \begin{bmatrix}
      \bfA & \bfv \\
      \vzero^{\top} & 1
        \end{bmatrix}
\begin{bmatrix}
      \mathbf{B} & \bfw \\
      \bfu^{\top} & c
  \end{bmatrix}
  =
  \begin{bmatrix}
  \mathbf{A} \mathbf{B} + \bfv  \bfu^{\top}  & \bfA \bfw + c \bfv
    \\
     \bfu^{\top} & c
  \end{bmatrix}
  =
  \begin{bmatrix}
    \mathbf{I}_{n} & \vzero \\
     \bfu^{\top} & c
  \end{bmatrix}
\]
Since \(\mathbf{M}\) is invertible by assumption, we have that \(c \ne 0\).
To finish the proof, it suffices to show that \(c >0\).
Below, we assume that \(c <0\) and derive
a contradiction.

The top right block of \cref{equation:inverse-of-gradient-of-F} implies that
\(\bfA \bfw + c \bfv = \vzero\), i.e., \(\bfA (-\bfw) = c\bfv\).
Since \(c <0\) and \(\bfv \prec \vzero\), we have \(\bfA (-\bfw) \succ \vzero\).
By strict monotonicity of \(\bfA\), we have that \(-\bfw \succ 0\), or equivalently, \(\bfw \prec \vzero\).
Finally, the bottom right  entry of \cref{equation:inverse-of-gradient-of-F} implies that \(1 = \vone^{\top} \bfw + c <0\), which is a contradiction.
\end{proof}

  \begin{lemma}\label{lemma:interior-projection-surjectivity}
    Let $\calL$ be a regular PERM loss with reduced form $\LL$ and
    \(\bfz \in \mathbb{R}^{k-1}\) be arbitrary.
Define \(\drop\) to be the projection \(\mathbb{R}^{k} \to \mathbb{R}^{k-1}\) that drops the last coordinate, i.e., \(\drop([v_{1},\dots, v_{k}]^{\top}) = [v_{1},\dots,v_{k-1}]^{\top}\).
    Then there exists \(\bfz^{*} \in \mathbb{R}^{k-1}\) and \(t^{*} \in \mathbb{R}_{>0}\)
    such that \(\drop\LL(\bfz^{*}) + t^{*}\vone =
\drop(\LL(\bfz))
\).
  \end{lemma}
  \begin{proof}
    Define functions \(\zeta : \mathbb{R}^{k-1} \times \mathbb{R} \to \mathbb{R}^{k-1} \) and \(\tau : \mathbb{R}^{k-1} \times \mathbb{R} \to \mathbb{R}\) by
    \[    \begin{bmatrix}
  \zeta(\bfz,t)^{\top} &
  \tau(\bfz,t)
\end{bmatrix}^{\top}
:=
F^{-1} \circ G(\bfz,t).
\]
    We first show that \(\tau(\bfz, 0) = 0\).
Observe that
    \begin{equation}
F^{-1} (G(\bfz,0))
=
F^{-1} (\LL(\bfz)+0\cdot\bfe_{k})
=
F^{-1} (\LL(\bfz)+0 \cdot\vone)
=
\begin{bmatrix}
  \bfz^{\top} &
  0
\end{bmatrix}^{\top}.
\label{equation:Finv-of-G-of-z-0}
\end{equation}
    Therefore, \(
  \zeta(\bfz,t) = \bfz
    \) and \(\tau(\bfz, 0) = 0\).
    Next, we claim that \(\frac{\partial \tau}{\partial t} (\bfz,0) \ne 0\).
    By definition,
    \begin{equation}
      \nabla_{F^{-1} \circ G} (\bfz,t) =
      \begin{bmatrix}
        \frac{\partial \zeta}{\partial \bfz} ( \bfz, t ) &
        \frac{\partial \tau}{\partial \bfz} ( \bfz, t ) \\
        \frac{\partial \zeta}{\partial t} ( \bfz, t ) &
        \frac{\partial \tau}{\partial t} ( \bfz, t ) \\
      \end{bmatrix}
\label{equation:jacobian-of-F-inv-of-G}
\end{equation}
    Let ``\(\cdot\)'' be a notational shorthand for the input \((\bfz,t)\) to \(F\) and \(G\).
    The gradient of \(F^{-1} \circ G\)
    \[
      \nabla_{F^{-1} \circ G} (\cdot)
      =
      \nabla_{G}(\cdot) \nabla_{F^{-1}}(G(\cdot))
      =
      \nabla_{G}(\cdot) (\nabla_{F}(F^{-1}\circ G(\cdot)))^{-1}
    \]
    Below, let \(\bfe_{k} := \bfe_{k}^{(k)}\), i.e., we drop the superscript.
    Now, from the proof of Lemma~\ref{lemma:gradient-of-G-is-nonsingular}, we have that
    \(\nabla_{G}(\bfz, t) =
    \begin{bmatrix}
  \nabla_{\LL}(\bfz)^{\top} & \bfe_{k}
    \end{bmatrix}^{\top}
    \)
  and from the proof of Lemma~\ref{lemma:gradient-of-F-is-nonsingular}  that
    \(\nabla_{F}(\bfz, t) =
    \begin{bmatrix}
  \nabla_{\LL}(\bfz)^{\top} & \vone^{(k)}
    \end{bmatrix}^{\top}
    \).
    By
\cref{equation:Finv-of-G-of-z-0}, we have
\[
      \nabla_{G}(\bfz,0) (\nabla_{F}(F^{-1}\circ G(\bfz,0)))^{-1}
      =
      \nabla_{G}(\bfz,0) (\nabla_{F}(\bfz,0))^{-1}
\]
The above is invertible because
      \(\nabla_{G}(\bfz,0)\)  is invertible by Lemma~\ref{lemma:gradient-of-G-is-nonsingular}.
Moreover,
    \begin{align}
\nabla_{G}(\bfz,0) (\nabla_{F}(\bfz,0))^{-1}
&=
    \begin{bmatrix}
  \nabla_{\LL}(\bfz)\\ \bfe_{k}^{\top}
    \end{bmatrix}
       \left(
    \begin{bmatrix}
  \nabla_{\LL}(\bfz) \\ (\vone^{(k)})^{\top}
    \end{bmatrix}
\right)^{-1}
  \\
&=
    \begin{bmatrix}
      \bfA(\bfz) & \nabla_{\TP}(\bfz)\\
      (\vzero^{(k-1)})^{\top} & 1
    \end{bmatrix}
       \left(
    \begin{bmatrix}
      \bfA(\bfz) & \nabla_{\TP}(\bfz)\\
      (\vone^{(k-1)})^{\top} & 1
    \end{bmatrix}
\right)^{-1}.
\label{equation:jacobian-of-Finv-of-G-at-z-0}
\end{align}
By
\cref{equation:jacobian-of-F-inv-of-G}, \(\frac{\partial \tau}{\partial t} (\bfz,0)\) is the bottom right element of
\(
\nabla_{G}(\bfz,0) (\nabla_{F}(\bfz,0))^{-1}\).
By applying Lemma~\ref{lemma:bottom-right-entry-nonnegativity} to the RHS of
\cref{equation:jacobian-of-Finv-of-G-at-z-0} with \(\bfA = \bfA(\bfz)\) and \(\bfv = \nabla_{\TP}(\bfz)\), we see that the bottom right entry of
\(
\nabla_{G}(\bfz,0) (\nabla_{F}(\bfz,0))^{-1}\) is nonzero. This proves that
\(\frac{\partial \tau}{\partial t} (\bfz,0) \ne 0\).
Note that
Lemmas~\ref{lemma:Az-matrix-is-nonsing-M-matrix}
and~\ref{lemma:gradient-of-G-is-nonsingular}
together guarantee that
the requirements of Lemma~\ref{lemma:bottom-right-entry-nonnegativity} are all met.

Next, we claim that there exists some \(t^{\circ} \in \mathbb{R}\) such that
\(\tau(\bfz,t^{\circ}) >0\). To see this, assume the contrary. Then \(\tau(\bfz,t) \le 0\) for all \(t \in \mathbb{R}\). In particular, \(t=0\) is a (global) maximizer of \(t \mapsto \tau(\bfz,t)\) which implies that
\(\frac{\partial \tau}{\partial t} (\bfz,0) = 0\), a contradiction. Thus, the claim follows.

Now, fix \(t^{\circ} \in \mathbb{R}\) such that \(\tau(\bfz,t^{\circ}) > 0\).
Define \(\bfz^{*}:=
      \zeta(\bfz,t^{\circ})
\)
and
\(t^{*}
=
  \tau(\bfz,t^{\circ})
\).
Then by the definition of \(\zeta\) and \(\tau\), we have
\(F(\bfz^{*}, t^{*})
=G(\bfz,t^{\circ})
\) by applying \(F\) to both side of
\[
    \begin{bmatrix}
      (\bfz^{*})^{\top}
      &
        t^{*}
\end{bmatrix}^{\top}
=
    \begin{bmatrix}
      \zeta(\bfz,t^{\circ})^{\top}
      &
  \tau(\bfz,t^{\circ})
\end{bmatrix}^{\top}
=
F^{-1}(G(\bfz,t^{\circ})).
\]
Unwinding the definition of \(F\) and \(G\) (Definition~\ref{definition:loss-surface-foliation-F}), the identity
\(F(\bfz^{*}, t^{*})
=G(\bfz,t^{\circ})
\)
implies
\[
  \LL(\bfz^{*})
  +
  t^{*}
  \vone =
    F(
    \bfz^{*},
    t^{*})
=
G(\bfz,t^{\circ})
=
\LL(\bfz) + t^{\circ} \bfe_{k}.
\]
Now, applying \(\drop(\cdot)\) to both side, we have
\(  \drop(\LL(\bfz^{*})
  +
  t^{*}
  \vone
  )=
\drop(\LL(\bfz) + t^{\circ} \bfe_{k})
=
\drop(\LL(\bfz) )
\)
where we used the fact that \(\drop(\cdot)\) is linear and that \(\drop(\bfe_{k}) = \vzero\).
  \end{proof}

  \begin{corollary}\label{corollary:interior-projection-surjectivity}
    Let $\calL$ be a regular PERM loss with reduced form $\LL$ and
    \(\bfz \in \mathbb{R}^{k-1}\) be arbitrary.
Define \(\proj\) to be the projection \(\mathbb{R}^{k} \to \mathbb{R}^{k-1}\) that projs the first coordinate, i.e., \(\proj([v_{1},\dots, v_{k}]^{\top}) = [v_{2},\dots,v_{k}]^{\top}\).
    Then there exists \(\bfz^{*} \in \mathbb{R}^{k-1}\) and \(t^{*} \in \mathbb{R}_{>0}\)
    such that \(\proj(\LL(\bfz^{*}) + t^{*}\vone) =
\proj\LL(\bfz)
\).
  \end{corollary}
  \begin{proof}
    Recall \(\mathbf{P}\) from
  {Lemma}~\ref{lemma:interior-projection-surjectivity}.
Let \(\sigma \in \mathtt{Sym}(k)\) be such that \(\proj =\mathbf{P}\mathbf{S}_{\sigma}\).
Recall \(\bmPi_{\sigma}\) from Definition~\ref{definition:irreducibility-representation}.
Then
\(   \mathbf{S}_{\sigma} \LL(\bfz)
    =
    \LL(\bmPi_{\sigma}\bfz)
\) by
      \cref{equation:permutation-equivariance-for-reduced-loss}.
Applying Lemma~\ref{lemma:interior-projection-surjectivity} to \(\bmPi_{\sigma} \bfz\), there exists \(\tilde{\bfz} \in \mathbb{R}^{k-1}\) and \({t}^{*} \in \mathbb{R}_{>0}\) such that
\(
\mathbf{P}
    \LL(\bmPi_{\sigma}\bfz)
    =
\mathbf{P}(\LL(\tilde{\bfz}^{*}) + {t}^{*}\vone)
\). Putting it all together, we have
\begin{align*}
  \proj
    \LL(\bfz)
    =
\mathbf{P} \mathbf{S}_{\sigma}
    \LL(\bfz)
    =
\mathbf{P}
    \LL(\bmPi_{\sigma}\bfz)
    =
\mathbf{P}(\LL(\tilde{\bfz}^{*}) + {t}^{*}\vone).
\end{align*}
Next, \(\proj =\mathbf{P}\mathbf{S}_{\sigma}\)
implies
\(\proj \mathbf{S}_{\sigma^{-1}} = \mathbf{P}\). Hence,
again applying \cref{equation:permutation-equivariance-for-reduced-loss}, we have
\[
\mathbf{P}(\LL(\tilde{\bfz}^{*}) + {t}^{*}\vone)
=
\proj \mathbf{S}_{\sigma^{-1}}(\LL(\tilde{\bfz}^{*}) + {t}^{*}\vone)
=
\proj (\LL(\bmPi_{\sigma^{-1}}\tilde{\bfz}^{*}) + {t}^{*}\vone).
\]
Letting \(\bfz^{*} :=\bmPi_{\sigma^{-1}}\tilde{\bfz}^{*}\), we have
\(
  \proj
    \LL(\bfz)
    =
\proj (\LL(\bfz^{*}) + {t}^{*}\vone)
\), as desired.
  \end{proof}

\begin{definition}\label{definition: loss surface - alternative characterization}
Let $f: \mathbb{R}^m \to \mathbb{R}^n$ be a function.
Define the following sets:
\begin{enumerate}
\item
$\cran^{\bullet}(f)
:=
\{
\bmzeta + \lambda \vone : \bmzeta \in \ran(f), \lambda \in [0,\infty)
\}
$
\item
  $\cran^{\circ}(f)
  :=
\{
\bmzeta + \lambda \vone : \bmzeta \in \ran(f), \lambda \in (0,\infty)
\}
$
\end{enumerate}

\end{definition}
When $f= \LL$ is the reduced form of a PERM loss, the above two sets are closely related to $\cran(\LL)$ (Definition~\ref{definition:loss-surface}), as the following lemma and
Proposition~\ref{proposition: boundary of S is R} show. These sets are convenient alternative characterizations.

\begin{lemma}\label{lemma: R is closed}
  Let $\calL$ be a regular PERM loss with reduced form \(\LL\).
  Then we have the following:
  \begin{compactenum}
    \item
  $\ran(\LL)$ is closed.
\item
  $\cran^{\bullet}(\LL)$ is closed and $\mathtt{bdry}(\cran^{\bullet}(\LL)) = \ran(\LL)$.
\item $\cran^{\circ}(\LL) = \interior(\cran^{\bullet}(\LL))$
  and $\mathtt{bdry}(\cran^{\circ}(\LL)) = \ran(\LL)$.
  \end{compactenum}
\end{lemma}
\begin{proof}
  For \emph{item 1}, define $C := \mathbb{R}^{k-1} \times \{0\}$ which is a closed subset of $\mathbb{R}^{k-1} \times \mathbb{R}$.
  Now, note that $\ran(\LL) = F(C)$ where \(F\) is as in Definition~\ref{definition:loss-surface-foliation-F}.
  Since $F$ is a homeomorphism (Corollary~\ref{corollary:F-is-a-homeomorphism}), $F(C)$ is closed as well.

  For \emph{item 2}, define \(D := \mathbb{R}^{k-1} \times [0,\infty)\) which  is a closed subset of $\mathbb{R}^{k-1} \times \mathbb{R}$.
  Then $\cran^{\bullet}(\LL) = F(D)$.
  Thus, as in the previous case $\cran^{\bullet}(\LL)$ is closed.
  Next, we have
  \[
    \mathtt{bdry}(\cran^{\bullet}(\LL))
    =
    \mathtt{bdry}(F(D)) {=} F(\mathtt{bdry}(D))
    =F(C) = \ran(\LL),
  \]
  where the second equality from the left follows from $F$ being a homeomorphism.

  For \emph{item 3}, let $E = \mathbb{R}^{k-1} \times (0,\infty)$. Then similar to the above, we have
  \[
    \interior( \cran^{\bullet}(\LL))
    =
    \interior( F(D)) = F(\interior( D))
    =F(E) = \cran^{\circ}(\LL).
  \]
  To conclude, note that $\mathtt{bdry}(\cran^{\circ}(\LL))
  =
  \mathtt{bdry}(F(E)) =
  F(\mathtt{bdry}(E)) = F(C) = \ran(\LL)$.
\end{proof}

\begin{proposition}
  \label{lemma: T is convex}
  Let $\calL$ be a regular PERM loss.
  Then $\cran^{\bullet}(\LL)$ is convex.
\end{proposition}
\begin{proof}
  Let $\bmzeta,\bmxi \in \cran^{\bullet}(f)$.
  Write $\bmzeta = \LL(\bfz) + \lambda \vone$
  and
  $\bmxi = \LL(\bfw) + \mu \vone$, where $\bfz,\bfw \in \mathbb{R}^{k-1}$ and $\lambda,\mu \in [0,\infty)$.
  Let $\bfv = \bmxi - \bmzeta \in \mathbb{R}^k$ and $\bfx = \bmzeta$.
  Take $\alpha$ and $\beta$ as defined in Proposition~\ref{proposition: concavity of restriction to lines}, i.e.,
  we have for all $t \in \mathbb{R}$ that
  \begin{equation}
    \label{equation: T is convex equation 1}
    t\bfv + \bfx = F(\alpha(t), \beta(t)) = \LL(\alpha(t)) + \beta(t)\vone.
  \end{equation}
  Plugging in $t=0$ into \eqref{equation: T is convex equation 1}, we get $\bmzeta = \LL(\alpha(0)) + \beta(0)$.
  Thus, $\alpha(0) = \bfz$ and $\beta(0) = \lambda$.
  Likewise, plugging in $t= 1$, we get $\alpha(1) = \bfw$ and $\beta(1) = \mu$.
  In particular, we have $\beta(0) \ge 0$ and $\beta(1) \ge 0$.
  By
  Proposition \ref{proposition: concavity of restriction to lines}, $\beta$ is concave.
  Thus, $\beta(t) \ge 0$ for all $t \in [0,1]$, i.e.,
  \[
    t\bfv + \bfw = t \bmxi + (1-t) \bmzeta = \LL(\alpha(t)) + \beta(t) \vone \in \cran^{\bullet}(\LL)
  \]
  for all $t \in [0,1]$. This proves that $\cran^{\bullet}(\LL)$ is convex.
\end{proof}
  % Consider the point $\bmxi = c \zeta_1 + (1-c) \zeta_2$.
  % We wish to show that $\bmxi \in \cran^{\bullet}(f)$.
  % Define
  % \[
  %   I := \left\{\lambda \in [0,1]: \substack{\exists \bfz \in \mathbb{R}^{k-1},\, \mu \in \mathbb{R} \mbox{ so that }
  %   \\ \lambda\zeta_1 + (1-\lambda) \zeta_2 = \LL(\bfz) + \mu\vone}\right\}.
% \]
% By assumption, we have $\{0,1\} \subseteq I$. We show that $I = [0,1]$ by showing that $I$ is both open and closed (WIP).

% Next, for each $\lambda \in [0,1]$, define $f(\lambda) \in \mathbb{R}^{k-1}$ and $g(\lambda) \in \mathbb{R}$ to be such that
% $\lambda\zeta_1 + (1-\lambda) \zeta_2 = \LL(f(\lambda)) + g(\lambda)\vone$.
% We view $g$ as a function $[0,1] \to \mathbb{R}$.
% By the Implicit Function Theorem, we have that $g$ is smooth.
% Observe that $g(0) = \mu_1$ and $g(1) = \mu_2$.

The following result is a restatement of \citet[Theorem 9]{beltagy2013boundary}:
\begin{theorem}[\citet{beltagy2013boundary}]
  \label{theorem: conv hull of boundary}
  Let $C$ be a nonempty closed convex subset of $\mathbb{R}^n$. If $C$ contains no hyperplane, then $C = \mathrm{conv}(\mathtt{bdry}(C))$.
\end{theorem}

\begin{proposition}\label{proposition: boundary of S is R}
  Let $\calL$ be a regular PERM loss with reduced form $\LL$.
  Then
  $\cran(\LL) = \cran^{\bullet}(\LL)$ and
  $\mathtt{bdry}(\cran(\LL)) = \ran(\LL)$.
\end{proposition}
\begin{proof}
  Clearly, $\cran^{\bullet}(\LL)$ is nonempty. Furthermore, by
  Lemma \ref{lemma: R is closed} and Lemma \ref{lemma: T is convex}, $\cran^{\bullet}(\LL)$ is closed and convex.
  Next, note that $\cran^{\bullet}(\LL)$ lies in the nonnegative quadrant $[0,\infty)^k$.
  Since no hyperplane lies entirely inside the nonnegative quadrant, $\cran^{\bullet}(\LL)$ cannot contain any hyperplane.
  Hence, we have verified that $\cran^{\bullet}(\LL)$ satisfies the condition of Theorem \ref{theorem: conv hull of boundary}.
  To finish the proof, we have
  \begin{align}
    \cran(\LL)
    &=
    \mathrm{conv}(\ran(\LL))
    \qquad \because \mbox{Definition of $\cran(\LL)$}
    \\
    &=
    \mathrm{conv}(\mathtt{bdry}(\cran^{\bullet}(\LL)))
    \qquad \because
    \mbox{Lemma \ref{lemma: T is convex}}
    \\
    &=
    \cran^{\bullet}(\LL)
    \qquad \because
    \mbox{Theorem \ref{theorem: conv hull of boundary}}
  \end{align}
  This proves the first part. For the second part, note
  that
\(\mathtt{bdry}(\cran(\LL))
    = \mathtt{bdry}(\cran^{\bullet}(\LL)) =
    \ran(\LL)\)
  by
  \mbox{Lemma \ref{lemma: T is convex}}.
\end{proof}
Before moving on, we summarize the  results on \(\cran(\LL)\) we have thus obtained  below:
\begin{corollary}\label{corollary: helpful summary of S(L)}
  Let $\calL$ be a regular PERM loss with reduced form $\LL$.
 Recall from Definition~\ref{definition: loss surface - alternative characterization} \[\cran^{\circ}(\LL):=
\{
\bmzeta + \lambda \vone : \bmzeta \in \ran(\LL), \lambda \in (0,\infty)
\}.\]
  Then
  $\cran(\LL)$ is a closed and convex set with the following properties:
  \begin{enumerate}
    \item
  \label{corollary: helpful summary of S(L) - characterization}
  $\cran(\LL)
  =
\{
\bmzeta + \lambda \vone : \bmzeta \in \ran(\LL), \lambda \in [0,\infty)
\}$
\item
  $\interior(\cran(\LL))
  =
  \cran^{\circ}(\LL)$ (see Definitions~\ref{definition:loss-surface} and \ref{definition: loss surface - alternative characterization})
\item $\mathtt{bdry}(\cran(\LL)) = \mathtt{bdry}(\cran^{\circ}(\LL)) = \ran(\LL)$.
  \end{enumerate}

\end{corollary}
We state one more result about the set $\cran^{\circ}(\LL)$ which will be useful later.

\begin{lemma}\label{lemma: superprediction set representation}
  Let $\calL$ be a regular PERM loss with reduced form $\LL$.
  Then
  $\cran^{\circ}(\LL)
  =
  \{
    \bmzeta \in \mathbb{R}^{k}: \exists \bfz \in \mathbb{R}^{k-1} \mbox{ such that }\bmzeta \succ \LL(\bfz)
  \}$.
\end{lemma}
\begin{proof}
  Recall that by definition we have
$\cran^{\circ}(\LL)
=
\{
  \LL(\bfz)+ \lambda \vone : \bfz \in \mathbb{R}^{k-1}, \lambda \in (0,\infty)
\}
$.
Thus, the ``$\subseteq$'' direction is immediate. For the other inclusion, take $\bmzeta = \LL(\bfz) + \bfv$ where $\bfz \in \mathbb{R}^{k-1}$ and $\bfv \succ 0$.
Let $\bfx = \LL(\bfz)$.
  Take $\alpha$ and $\beta$ as defined in Proposition \ref{proposition: concavity of restriction to lines}, i.e.,
  we have for all $t \in \mathbb{R}$ that
  \begin{equation}
    \label{equation: superprediction set 1}
    t\bfv + \bfx = F(\alpha(t), \beta(t)) = \LL(\alpha(t)) + \beta(t)\vone.
  \end{equation}
  Plugging in $t=0$ into \eqref{equation: superprediction set 1}, we get $\bfx= \LL(\bfz)= \LL(\alpha(0)) + \beta(0)$.
  Thus, $\alpha(0) = \bfz$ and $\beta(0) = 0$.
  Recall that $\bfv = \bmzeta - \LL(\bfz) \succ 0$  by assumption.
  Hence, by
  Proposition~\ref{proposition: concavity of restriction to lines}, $\beta$ is strictly increasing.
  In particular, $\beta(1) > \beta(0) = 0$.
  Now, plugging in $t= 1$ into \eqref{equation: superprediction set 1}, we get $\bfv + \bfx = \bfv + \LL(\bfz) = \bmzeta = \LL(\alpha(1)) + \beta(1)\vone$.
  This shows that $\bmzeta \in \cran^{\circ}(\LL)$, as desired.
\end{proof}

\begin{remark}\label{remark:relation-to-superprediction-set}
  From basic point-set topology\footnote{The boundary of a set \(A\) is defined as the closure of \(A\) minus the interior of \(A\).}, we know that a closed set is the union of its interior and its boundary.  Thus, $\cran(\LL) = \interior(\cran(\LL)) \, \cup \,  \mathtt{bdry}(\cran(\LL))$.
  Hence, a consequence of Lemma~\ref{lemma: superprediction set representation} and Lemma~\ref{lemma: R is closed} is that
  $\cran(\LL)$ is precisely the \emph{superprediction set} of $\LL$ (see \citet[Definition 15]{williamson2016composite} and \citet{kalnishkan2008weak}):
  \[
    \cran(\LL) = \{
    \bmzeta \in \mathbb{R}^{k}: \exists \bfz \in \mathbb{R}^{k-1} \mbox{ such that }\bmzeta \succeq \LL(\bfz)
    \}.
  \]
\end{remark}

  Recall \citep[Definition 5]{tewari2007consistency}:
  \begin{definition}[Admissible sets; \citet{tewari2007consistency}]\label{definition:admissible-sets}
    Let $S \subseteq \mathbb{R}^{k}_+$ be a set
    and $\bmzeta \in \mathbb{R}^k_+$.
    Define the set
  \(    \mathcal{N}(\bmzeta; S) :=
  \{\bfp \in \Delta^k :
  \langle \bmxi - \bmzeta, \bfp \rangle \ge 0,\, \forall \bmxi \in S\}.
\)
We say that $S$ is \emph{admissible} if
for all $\bmzeta \in \mathtt{bdry}(S)$ and $\bfp \in \mathcal{N}(\bmzeta; S)$ we have
$\argmin(\bmzeta) \subseteq \argmax(\bfp)$.
  \end{definition}

  \begin{proposition}[\citet{tewari2007consistency}]
    \label{proposition: tewari sufficient condition for admissibility}
    Let $S \subseteq \mathbb{R}^{k}_+$ be a symmetric set. If $|\mathcal{N}(\bmzeta; S)| = 1$ for all $\bmzeta \in \mathtt{bdry}(S)$, then $S$ is admissible.
  \end{proposition}

  \begin{lemma}\label{lemma: normal set are all the same}
  Let $\calL$ be a regular PERM loss with reduced form $\LL$.
    Let $\bmzeta \in \ran(\LL)$. Then we have
    $
    \mathcal{N}(\bmzeta; \ran(\LL))
    =
\mathcal{N}(\bmzeta; \cran(\LL))
=\mathcal{N}(\bmzeta; \cran^{\circ}(\LL))$.
  \end{lemma}
  \begin{proof}
    We first prove that $
    \mathcal{N}(\bmzeta; \ran(\LL))
    =
\mathcal{N}(\bmzeta; \cran(\LL))$.
    Since $\cran(\LL) \supseteq  \ran(\LL)$, we immediately have
    $
    \mathcal{N}(\bmzeta; \ran(\LL))
    \supseteq
\mathcal{N}(\bmzeta; \cran(\LL))$.
For the other inclusion,
we first note that $\cran(\LL) = \cran^{\bullet}(\LL)$ by Proposition~\ref{proposition: boundary of S is R}.
Thus, every $\bmxi \in \cran(\LL)$ can be written as $\bmxi = \bmalpha + \beta \vone$ for some $\bmalpha \in \ran(\LL)$  and $\beta \ge 0$.
Now, let $\bfp \in
    \mathcal{N}(\bmzeta; \ran(\LL))$ and let $\bmxi \in \cran(\LL)$ be decomposed as in the preceding sentence.
    Then
\[
  \langle \bmxi - \bmzeta, \bfp \rangle
  =
  \langle \bmalpha + \beta \vone - \bmzeta, \bfp \rangle
  =
  \langle \bmalpha - \bmzeta,\bfp \rangle + \beta \langle \vone , \bfp \rangle \ge 0
\]
where the last inequality holds since (1) $\langle \bmalpha - \bmzeta,\bfp \rangle \ge 0$ because $\bfp \in
    \mathcal{N}(\bmzeta; \ran(\LL))$, and (2) $\beta \ge 0$.
Hence, such a $\bfp$ satisfies
$\langle \bmxi - \bmzeta, \bfp \rangle \ge 0,\, \forall \bmxi \in \cran(\LL)$ as well which implies that $\bfp \in \mathcal{N}(\bmzeta; \cran(\LL))$, as desired.

Next, we prove $
\mathcal{N}(\bmzeta; \cran(\LL))
=\mathcal{N}(\bmzeta; \cran^{\circ}(\LL))$.
Again, since $\cran(\LL) \supseteq  \cran^{\circ}(\LL)$, we immediately have
    $
    \mathcal{N}(\bmzeta; \cran^{\circ}(\LL))
    \supseteq
\mathcal{N}(\bmzeta; \cran(\LL))$.
For the other inclusion, we first note that \(\closure(\cran^{\circ}(\LL)) = \cran(\LL)\).
Suppose $\bfp \in \Delta^k$ is such that $\langle \bmxi - \bmzeta, \bfp \rangle \ge 0$ for al $\bmxi \in \cran^{\circ}(\LL)$.
Then by continuity, we must have that $\langle \bmxi - \bmzeta, \bfp \rangle \ge 0$ for all $\bmxi \in \closure(\cran^{\circ}(\LL)) = \cran(\LL)$.
  \end{proof}

\begin{proposition}
  \label{proposition: main result on admissibility}
  Let $\calL$ be a regular PERM loss with reduced form $\LL$.
  Then
  $\cran(\LL)$ and $\cran^{\circ}(\LL)$ are both admissible.
\end{proposition}
\begin{proof}
  By Proposition~\ref{proposition: tewari sufficient condition for admissibility}, it suffices to check the following two claims hold:
  \begin{compactenum}
    \item
  for all $\bmzeta \in \mathtt{bdry}(\cran(\LL))$ we have $|\mathcal{N}(\bmzeta; \cran(\LL))| = 1$, and
\item
  for all $\bmzeta \in \mathtt{bdry}(\cran^{\circ}(\LL))$ we have $|\mathcal{N}(\bmzeta; \cran^{\circ}(\LL))| = 1$.
  \end{compactenum}

  By Corollary \ref{corollary: helpful summary of S(L)}, we have
  $\mathtt{bdry}(\cran(\LL)) = \mathtt{bdry}(\cran^{\circ}(\LL)) = \ran(\LL)$.
  Hence, by Lemma \ref{lemma: normal set are all the same}, to show both above claims it suffices to show that
$|\mathcal{N}(\bmzeta; \ran(\LL))| = 1$ for all $\bmzeta \in \ran(\LL)$.
Note that here we can replace $\mathcal{N}(\bmzeta; \cran(\LL))$ and $\mathcal{N}(\bmzeta; \cran^{\circ}(\LL))$ by $\mathcal{N}(\bmzeta; \ran(\LL))$ because of
Lemma~\ref{lemma: normal set are all the same}.
Below, fix $\bmzeta = \LL(\bfz) \in \ran(\LL)$ where $\bfz \in \mathbb{R}^{k-1}$ is arbitrary. Then
  \begin{align*}
    \mathcal{N}(\bmzeta; \ran(\LL))
    &=
      \{\bfp \in \Delta^k: \langle \bmxi - \bmzeta, \bfp \rangle \ge 0, \, \forall \bmxi \in \ran(\LL)\} \quad \because
\mbox{Definition~\ref{definition:admissible-sets}}
    \\
    &=
    \{\bfp \in \Delta^k: \langle \LL(\bfw) - \LL(\bfz), \bfp \rangle \ge 0, \, \forall \bfw \in \mathbb{R}^{k-1}\}
    \quad \because \mbox{Corollary \ref{corollary: L is injective}}\\
    &=
      \left\{\bfp \in \Delta^k: \bfz \in \textstyle\argmin_{\bfw \in \mathbb{R}^{k-1}} C^{\calL}_{\bfp}(\bfw)\right\}
\quad \because \mbox{definition of \(\argmin\)}
      \\
    &=
    \left\{\bfp \in \interior(\Delta^k): \bfz \in\textstyle \argmin_{\bfw \in \mathbb{R}^{k-1}} C^{\calL}_{\bfp}(\bfw)\right\}
    \quad \because \mbox{Corollary \ref{corollary: boundary p has no argmin}}\\
    &=
    \{\bfp \in \interior(\Delta^k): \bfz = \gsm^{\calL}(\bfp)\}
    \quad \because \mbox{Definition \ref{definition: inverse link function} and Corollary \ref{corollary: inverse link map}}
  \end{align*}
  By Proposition~\ref{proposition:link-function-is-bijection}, $\gsm^{\calL}$ is an injection.
  Thus, $|\left\{\bfp \in \interior(\Delta^k): \bfz = \gsm^{\calL}(\bfp)\right\}| = 1$.
\end{proof}

% \subsection{Calibration of nested family}
\section{Proof of Theorem~\ref{theorem: nested family of regular PERM losses are CC - exposition version}}\label{section:proof theorem totally regular PERM loss is CC}

\begin{lemma}\label{lemma: trivial projection identity}
 Suppose that \(k \ge 3\) and \(y \in \{2,\dots, k\}\).
 Let \(\proj\) be as in Corollary~\ref{corollary:interior-projection-surjectivity}.
  Then
  \(
  \proj \ico_{y}^{(k)} = \ico_{y-1}^{(k-1)} \proj
  \).
\end{lemma}

\iftoggle{arxiv}
{\begin{proof}
  First, consider the case when  \(y = k\). Then \(\ico^{(k)}_{y} = \mathbf{I}_{k-1}\) and \(\ico^{(k-1)}_{y-1} = \mathbf{I}_{k-2}\) are  both identity matrices.
  Thus,
\(    \proj \ico_{y}^{(k)}
    =
    \ico_{y-1}^{(k-1)}    \proj
\)
holds trivially.
Below, we assume that \(y \in \{2,\dots, k-1\}\).
Let \(\bfz = (z_{1},\dots, z_{k-1}) \in \mathbb{R}^{k-1}\) be arbitrary.
  For each \(j \in [k-2]\), we have \(j+1 \in \{2,\dots, k-1\}\) and so
  \[
    [\proj \ico_y^{(k)}\bfz]_{j}
    =
    [\ico_y^{(k)}\bfz]_{j+1}
    =
    \begin{cases}
      z_{j+1} - z_{y} &: j+1 \ne y\\
      -z_{y} &: j+1 = y.
    \end{cases}
  \]
  On the other hand,
  \[
    [\ico_{y-1}^{(k-1)}\proj \bfz]_j
    =
    \begin{cases}
      [\proj \bfz]_j - [\proj \bfz]_{y-1} &: j \ne y-1\\
      -[\proj \bfz]_{y-1} &: j = y-1
    \end{cases}
    =
    \begin{cases}
      z_{j+1} - z_{y} &: j+1 \ne y\\
      -z_{y} &: j+1 = y.
    \end{cases}
  \]
  Since \(\bfz\) is arbitrary, we have \(
  \proj \ico_{y}^{(k)} = \ico_{y-1}^{(k-1)} \proj
  \)
  for \(y \in \{2,\dots, k-1\}\).
  % Note that \(w_j = z_j\) and \(w_{y} = z_{y}\) since \(j , y \in [k-2]\).
\end{proof}}
{
  \textcolor{black}{
The proof of
{Lemma}~\ref{lemma: trivial projection identity}
is  similar to that of
{Lemma}~\ref{lemma: rho pi sigma relation}
and is thus omitted here.
For the proof, see the arXiv version of this work \citep{wang2023unified}.
}
}

\begin{lemma}\label{lemma:truncation-alternative-characterization}
Assume \(k \ge 3\).  Let \(\calL : \mathbb{R}^{k} \to \mathbb{R}^{k}\) be a regular PERM loss with template \(\TP: \mathbb{R}^{k-1} \to \mathbb{R}\).
 Let \(\proj\) be as in Corollary~\ref{corollary:interior-projection-surjectivity}.
  Recall from Proposition~\ref{proposition:truncation-of-a-PERM-loss}
  the {truncation} of \(\TP\)  denoted by \(\con[\TP] : \mathbb{R}^{k-2} \to \mathbb{R}\).
Let \(\bfz \in \mathbb{R}^{k-1}\) and  \(\bfw \in \mathbb{R}^{k-2}\) be arbitrary and such that \(\proj \bfz = \bfw\).
Then
\(\lim_{\lambda \to \infty} \TP(\bfz + \lambda \bfe^{(k-1)}_{1}) = \con[\TP](\bfw)\).
\end{lemma}
\begin{proof}
  Write \(\bfz = [z_{1},\dots, z_{k-1}]\) and \(\bfw = [w_{1},\dots, w_{k-2}]\).
Since \(\proj \bfz = \bfw\), we have \(z_{2} =w _{1}\), \(z_{3} = w_{2}\) and so on.
  Now, define \(g: \mathbb{R} \to \mathbb{R}\) by
  \(g(\lambda) :=\TP(\bfz + \lambda \bfe^{(k-1)}_{1})
=
\TP(z_{1}+ \lambda,\, z_{2},\, \dots, \, z_{k-1})
\)
  and \(h: \mathbb{R} \to \mathbb{R}\) by
  \[
    h(\lambda):=
    \TP(\lambda, \bfw) =
    \TP(\lambda, w_{1},\,w_{2},\,\dots,\, w_{k-2})
    =
    \TP(\lambda, z_{2},\,z_{3},\,\dots,\, z_{k-1})
    = g(\lambda - z_{1}).
  \]
  As  argued in the proof of Proposition~\ref{proposition:truncation-of-a-PERM-loss}, \(h\) is decreasing and nonnegative.
  Thus, \(g\) is also decreasing and nonnegative. Moreover, \(\lim_{\lambda \to \infty}g(\lambda) =\lim_{\lambda \to \infty}g(\lambda - z_{1}) = \lim_{\lambda \to \infty} h(\lambda)\).
  The right hand side is equal to
  \(\con[\TP](\bfw)\),
  as in the proof of Proposition~\ref{proposition:truncation-of-a-PERM-loss}
\end{proof}

Earlier in Proposition~\ref{proposition:truncation-of-a-PERM-loss}
and Corollary~\ref{corollary:iterated-truncation-of-a-PERM-loss}, we defined \(\con[\TP]\),
\(\TP^{(n)}\) and \(\calL^{(n)}\). We now define the analogous notation for the reduced form \(\LL\):
\begin{definition}[Truncation of \(\LL\)]\label{definition:truncation-of-the-reduced-form}
  Assume \(k \ge 3\).
  Let \(\calL : \mathbb{R}^{k} \to \mathbb{R}^{k}_{\ge 0}\) be a regular PERM loss with template \(\TP: \mathbb{R}^{k-1} \to \mathbb{R}\).
   As in {Proposition}~\ref{proposition:truncation-of-a-PERM-loss},
  define
  \(\con[\LL]\) to be the reduced form of \(\con[\calL]\) (whose template is \(\con[\TP]\)).
\end{definition}
% We recall that
% {Definition}~\ref{definition:totally-regular-PERM-loss} that \(\calL^{(n)}\) is the unique PERM loss associated to \(\TP^{(n)} := \con^{\iter{k-n}}[\TP]\).

\begin{lemma}\label{lemma: nesting structural lemma}
Assume \(k \ge 3\).  Let \(\calL : \mathbb{R}^{k} \to \mathbb{R}^{k}\) be a regular PERM loss with template \(\TP: \mathbb{R}^{k-1} \to \mathbb{R}\).
 Let \(\proj\) be as in Corollary~\ref{corollary:interior-projection-surjectivity}.
  Recall from Proposition~\ref{proposition:truncation-of-a-PERM-loss}
  the {truncation} of \(\TP\)  denoted by \(\con[\TP] : \mathbb{R}^{k-2} \to \mathbb{R}\).
  Let
\(\con[\LL]\) be as in Definition~\ref{definition:truncation-of-the-reduced-form}.
 Let \(\proj\) be as in Corollary~\ref{corollary:interior-projection-surjectivity}.
  Let \(\bfz \in \mathbb{R}^{k-1}\) and \(\bfx \in \mathbb{R}_{\ge 0}^{k}\) be arbitrary.
    For brevity, let \(\bfe_{1} := \bfe^{(k-1)}_{1}\).
  % Define \(\tilde{\bfz} := \proj\bfz \in \mathbb{R}^{k-2}\) and \(\tilde{\bfx} := \proj\bfx \in [0,\infty)^{k-2}\).
  Then we have
\begin{equation}
  \label{equation: nesting structural lemma - limit point}
  \lim_{\lambda \to + \infty}
  \proj\left(
    \LL\left(
      \bfz + \lambda
      \bfe_{1}
    \right)
    +\bfx
\right)
  =
  \con[\LL](\proj \bfz)
  +
  \proj \bfx
\end{equation}
and
\begin{equation}
  \label{equation: nesting structural lemma - inclusion}
  \proj\left(
    \LL\left(
      \bfz
    \right)
    + \bfx\right)
  \succ
\con[\LL](\proj \bfz).\end{equation}
  \end{lemma}

  % \begin{lemma}\label{lemma: nesting structural lemma-old}
%     \yw{OLD}
%   Assume that we are in the situation stated at the beginning of Section~\ref{section:proof theorem totally regular PERM loss is CC}.
%   Let $\bfz \in \mathbb{R}^{k-1}$ and $\bfx \in [0,\infty)^{k}$.
%   Define $\tilde{\bfz} := \proj\bfz \in \mathbb{R}^{k-2}$ and $\tilde{\bfx} := \proj\bfx \in [0,\infty)^{k-2}$.
%   Then we have
% \begin{equation}
%   \label{equation: nesting structural lemma - limit point-OLD}
%   \lim_{\lambda \to + \infty}
%   \proj\left(
%     \LL^{(k)}\left(
%       \bfz + \lambda
%       \bfe^{(k-1)}_{k-1}
%     \right)
%     +\bfx
% \right)
%   =
%   \LL^{(k-1)}(\tilde{\bfz})
%   +
%   \tilde{\bfx}
% \end{equation}
% and
% \begin{equation}
%   \label{equation: nesting structural lemma - inclusion-OLD}
%   \proj\left(
%     \LL^{(k)}\left(
%       \bfz
%     \right)
%     + \bfx\right)
% >
% \LL^{(k-1)}(\tilde{\bfz}).\end{equation}
%   \end{lemma}
  \begin{proof}
    % first change n to k
    % second change m to n
    Throughout this proof, let \(y \in \{2,\dots, k\}\) be arbitrary.
    First, we claim that
  \(\ico_y^{(k)} \bfe_{1} = \bfe_{1}\)
  for each \(y \in \{2,\dots,k\}\).
  If \(y = k\), then this is clearly true since \(\ico_{k}^{(k)}\) is, by definition, the identity matrix.
  Now, for \(y \in \{2,\dots, k-1\}\), recall that \(\ico_{y}^{(k)}\) is defined as the matrix obtained by replacing the \(y\)-th column of the identity matrix with the all \(-1\)'s vector.
  Thus, since \(y>1\), the first column of \(\ico_{y}^{(k)}\) is equal to that of the identity matrix. Therefore, \(\ico_{y}^{(k)} \bfe_{1} = \bfe_{1}\) as well.
  Next, still assuming \(y \in \{2,\dots, k-1\}\), we have \[\LL_y(\bfz+ \lambda \bfe_{1})
  =
  \TP(\ico_y^{(k)} (\bfz + \lambda \bfe_{1}))
  =
  \TP(\ico_y^{(k)} \bfz+ \lambda \bfe_{1}).
\]
Hence, we have
\begin{align*}
  \lim_{\lambda \to +\infty}
  \LL_{y} (\bfz+ \lambda \bfe_{1})
&=
  \lim_{\lambda \to +\infty}
  \TP(\ico_{y}^{(k)} \bfz+ \lambda \bfe_{1})
\quad \because \mbox{
  Theorem~\ref{theorem:appendix:relative-margin-form}.
}
\\&
=
\con[\TP](\proj\ico_{y}^{(k)} \bfz)
\quad \because \mbox{
  Proposition~\ref{proposition:truncation-of-a-PERM-loss}
  and
  Lemma~\ref{lemma:truncation-alternative-characterization}
}
\\&
=
\con[\TP]\left(\ico_{y-1}^{(k-1)} \proj \bfz\right)
\quad \because \mbox{
  Lemma \ref{lemma: trivial projection identity}
}
\\&
=
[\con[\LL]]_{y-1}(\proj \bfz)
\quad \because \mbox{
  Theorem~\ref{theorem:appendix:relative-margin-form} and explanation below.
}
\end{align*}
Application of Theorem~\ref{theorem:appendix:relative-margin-form} to the last equality requires a bit more explanation.  For said equality, we used the fact that
\(\con[\LL]\) is the reduced form of \(\calL^{(k-1)}\)
({Definition}~\ref{definition:truncation-of-the-reduced-form}), whose template is
\(\con[\TP]\).
Thus,
applying
\Cref{equation:relative-margin-form-summarized}
(which is a collary of
    \Cref{equation:appendix:relative-margin-form}
    from
Theorem~\ref{theorem:appendix:relative-margin-form})
to \(\con[\calL]\), we have
\(  [\con[\LL]]_{y-1}(\bfz)
  =\con[\TP](\ico_{y-1}^{(k-1)} \bfz)
\).
Thus,
\[
  \lim_{\lambda \to + \infty}
  \proj\left( \LL(\bfz + \lambda \bfe_{1})  + \bfx\right)
  =
  \con[\LL](\proj \bfz)
  +
  \proj\bfx.
\]
% This proves the ``Furthermore'' part.
Next, for every \(y \in \{2,\dots, k\}\), we note that the function
\[
  g_{y}(\lambda):=
  \LL_{y}(\bfz+\lambda \bfe_{1}) = \TP(\ico_{y} \bfz + \lambda \bfe_{1})\]
is strictly decreasing. To see this, by the
\iftoggle{arxiv}{chain rule for curves (see
\cref{equation:chain-rule-for-curves} from Remark~\ref{remark:time-derivatives})}
  {\textcolor{black}{
chain rule for curves\footnote{
  See the ``Vector Calculus'' section in the appendix of the arXiv version of this manuscript \citep{wang2023unified}.
  }}},
we have
\[{g_{y}'}(\lambda)
=
\nabla_{\TP}(\ico_{y} \bfz + \lambda \bfe_{1})^{\top} \bfe_{1} < 0.\]
Thus, \(\LL_{y}(\bfz) = g_{y}(0) > \lim_{\lambda \to + \infty} g_{y}(\lambda) =
\con[\LL]_{y-1}(\proj \bfz)\),
which proves that
\[
\proj\left( \LL(\bfz) + \bfx\right)
\succ
\con[\LL](\proj \bfz) + \proj \bfx
\succeq
\con[\LL](\proj \bfz)
\]
as desired.
\end{proof}

\begin{lemma}\label{lemma:closure-range-identity}
Assume \(k \ge 3\).  Let \(\calL : \mathbb{R}^{k} \to \mathbb{R}^{k}\) be a regular PERM loss with template \(\TP: \mathbb{R}^{k-1} \to \mathbb{R}\).
 Let \(\proj\) be as in Corollary~\ref{corollary:interior-projection-surjectivity}.
  Recall from Proposition~\ref{proposition:truncation-of-a-PERM-loss}
  the {truncation} of \(\TP\)  denoted by \(\con[\TP] : \mathbb{R}^{k-2} \to \mathbb{R}\).
  Let
\(\con[\LL]\) be as in Definition~\ref{definition:truncation-of-the-reduced-form}.
 Let \(\proj\) be as in Corollary~\ref{corollary:interior-projection-surjectivity}.
  Then
  $\proj(\cran(\LL)) \subseteq \cran^{\circ}(\con[\LL])$
  and
  $\closure[\proj(\cran(\LL))] = \cran(\con[\LL])$.
\end{lemma}
\begin{proof}
  Let $C := \proj(\cran(\LL))$ and take $\bmzeta \in C$.
  We first prove
  $C \subseteq \cran^{\circ}(\con[\LL])$.
  By the characterization of $\cran(\LL)$ from Corollary~\ref{corollary: helpful summary of S(L)} item~\ref{corollary: helpful summary of S(L) - characterization}, there exists $\bfz \in \mathbb{R}^{k-1}$ and $\bfx \in [0,\infty)^k$
  such that $\bmzeta = \proj(\LL(\bfz) + \bfx)$.
Now let $\tilde{\bfz} := \proj(\bfz)$.
  Applying Eqn.~\eqref{equation: nesting structural lemma - inclusion} from Lemma~\ref{lemma: nesting structural lemma},
we get
\[\bmzeta = \proj\left( \LL(\bfz) + \bfx\right) \succ \con[\LL](\tilde{\bfz}).\]
In particular, by
the characterization of $\cran^{\circ}(\con[\LL])$ from
  Lemma~\ref{lemma: superprediction set representation}, we have that
  $\bmzeta \in \cran^{\circ}(\con[\LL])$.
  This proves that $C \subseteq \cran^{\circ}(\con[\LL])$.

  Next, we prove
  $\closure[C] = \cran(\con[\LL])$.
  We first show that $\closure[C] \supseteq \cran(\con[\LL])$ by proving that
  every point $\cran(\con[\LL])$ is a limit point of $C$.

  Let $\bmzeta \in \cran(\con[\LL])$.
  By the characterization of $\cran(\con[\LL])$ as in Corollary~\ref{corollary: helpful summary of S(L)}, there exists $\bar{\bfz} \in \mathbb{R}^{k-2}$ and $\bar{\bfx} \in [0,\infty)^{k-1}$ such that
  $\bmzeta = \con[\LL](\bar{\bfz}) + \bar{\bfx}$.
  Now, pick $\bfz \in \mathbb{R}^{k-1}$ and \(\bfx \in [0,\infty)^{k}\) such that
  $\bar{\bfz} = \proj\bfz$ and $\bar{\bfx} = \proj\bfx$.
  Applying Lemma~\ref{lemma: nesting structural lemma}
  Eqn.~\eqref{equation: nesting structural lemma - limit point}, we get that $\bmzeta$ is a limit point of $S$, which proves the desired claim.
  This proves that $\closure(C) \supseteq \cran(\con[\LL])$.
  By the first part, we know that $C \subseteq \cran^{\circ}(\con[\LL])$.
  By Corollary~\ref{corollary: helpful summary of S(L)}, $\cran^{\circ}(\con[\LL]) =
  \interior(\cran(\con[\LL]) )
  \subseteq
\cran(\con[\LL])$.
Putting it all together, we have
\[
  C \subseteq \cran^{\circ}(\con[\LL]) \subseteq
\cran(\con[\LL])
\subseteq \closure(C).
\]
From
  Corollary \ref{corollary: helpful summary of S(L)}, we have that $
\cran(\con[\LL])$ is closed.
Since by definition $\closure(C)$ is the smallest closed set containing $C$, we get
that
$\cran(\con[\LL])
= \closure(C)$, as desired.
\end{proof}

\begin{theorem}[\citet{blackwell1979theory}]\label{theorem: convex set has the same interior as its closure}
  Let $C \subseteq \mathbb{R}^n$ be a convex set. Then $\interior(C) = \interior(\closure(C))$.
\end{theorem}

\begin{proposition}\label{proposition: nested-ness of the superprediction set}
Assume \(k \ge 3\).  Let \(\calL : \mathbb{R}^{k} \to \mathbb{R}^{k}\) be a regular PERM loss with template \(\TP: \mathbb{R}^{k-1} \to \mathbb{R}\).
 Let \(\proj\) be as in Corollary~\ref{corollary:interior-projection-surjectivity}.
  Recall from Proposition~\ref{proposition:truncation-of-a-PERM-loss}
  the {truncation} of \(\TP\)  denoted by \(\con[\TP] : \mathbb{R}^{k-2} \to \mathbb{R}\).
  Let
\(\con[\LL]\) be as in Definition~\ref{definition:truncation-of-the-reduced-form}.
 Let \(\proj\) be as in Corollary~\ref{corollary:interior-projection-surjectivity}.
  Then we have
  $\proj(\cran(\LL)) = \cran^{\circ}(\con[\LL])$
\end{proposition}
\begin{proof}
  For brevity, let $C := \proj(\cran(\LL))$.
  By Corollary~\ref{corollary: helpful summary of S(L)}, $\cran(\con[\LL])$ is convex.
  Since convexity is preserved under projection, we have that $C$ is convex as well.
  Now,
  \begin{align}
    \interior(C) &= \interior(\closure(C))  \quad \because \mbox{ Theorem \ref{theorem: convex set has the same interior as its closure} }
    \\
                    &=
 \interior(\cran(\con[\LL]))
\quad \because \mbox{ Lemma \ref{lemma:closure-range-identity} }
 \\
                    &=\cran^{\circ}(\con[\LL])
\quad \because \mbox{ Lemma \ref{lemma: R is closed} }
                  \\&\supseteq
                  C
\quad \because \mbox{ Lemma \ref{lemma:closure-range-identity} }
  \end{align}
  Since $C \supseteq \interior(C)$ by definition, we conclude that $C = \cran^{\circ}(\con[\LL])$.
  \
  \end{proof}

  \begin{proposition}\label{proposition:projection-of-interior-of-S}
Assume \(k \ge 3\).  Let \(\calL : \mathbb{R}^{k} \to \mathbb{R}^{k}\) be a regular PERM loss with template \(\TP: \mathbb{R}^{k-1} \to \mathbb{R}\).
 Let \(\proj\) be as in Corollary~\ref{corollary:interior-projection-surjectivity}.
  Recall from Proposition~\ref{proposition:truncation-of-a-PERM-loss}
  the {truncation} of \(\TP\)  denoted by \(\con[\TP] : \mathbb{R}^{k-2} \to \mathbb{R}\).
  Let
\(\con[\LL]\) be as in Definition~\ref{definition:truncation-of-the-reduced-form}.
 Let \(\proj\) be as in Corollary~\ref{corollary:interior-projection-surjectivity}.
  Then we have
  \(\proj(\cran^{\circ}(\LL)) = \cran^{\circ}(\con[\LL]).\)
  \end{proposition}
  \begin{proof}
    By the preceding Proposition~\ref{proposition: nested-ness of the superprediction set}, we have
  $\proj(\cran(\LL)) = \cran^{\circ}(\con[\LL])$.
  Since
\(\cran^{\circ}(\LL)
\subseteq
  \cran(\LL) \)
  we have
\(\proj(\cran^{\circ}(\LL))
\subseteq
  \proj(\cran(\LL))
  \). Thus, to prove the result we only have to show
  \(\proj(\cran^{\circ}(\LL)) \supseteq \cran^{\circ}(\con[\LL])\).

  To this end, let \( \con[\LL](\bfw) \in \cran^{\circ}(\con[\LL])\) and $\bfz \in \cran(\LL)$ be such that $\proj\LL(\bfz) = \con[\LL](\bfw)$.
  By
  % Lemma~\ref{lemma:interior-projection-surjectivity}
  Corollary~\ref{corollary:interior-projection-surjectivity}, there exist
    \(\bfz^{*} \in \mathbb{R}^{k-1}\) and \(t^{*} \in \mathbb{R}\) such that
    \(t^{*} > 0\) and \(\proj(\LL(\bfz^{*}) + t^{*}\vone)  = \proj\LL(\bfz)= \con[\LL](\bfw)\).
    Since \(\LL(\bfz^{*}) + t^{*}\vone \in \cran^{\circ}(\LL) \), we get that
    \(\con[\LL](\bfw) \in \proj(\cran^{\circ}(\LL))
    \)
    as desired.
  \end{proof}

\begin{definition}\label{definition:iterated-truncation-of-the-reduced-form}
    For each \(m \in \{0,1,\dots,k-2\}\), define \(\con^{\iter{m}}[\LL]\) to be \(m\)-fold repeated applications of \(\con\) to \(\LL\), i.e.,
    \(
\con^{\iter{m}}[\LL] :=\con[\cdots \con[\con[\LL]] \cdots]
    \)
    where \(\con\) appears \(m\)-times.
  By convention, let \(\con^{\times 0}[\LL]  = \LL\).
  Moreover, for each \(n \in \{2,\dots, k\}\), define the notational shorthand \(\LL^{(n)} := \con^{\times (k-n)}[\LL]\).
\end{definition}
\begin{remark}
  It follows tautologically from
the definition of \(\TP^{(n)}\) in Corollary~\ref{corollary:iterated-truncation-of-a-PERM-loss}, that
define \(\LL^{(n)}\) is the reduced form of
  \(\calL^{(n)}\) (whose template is \(\TP^{(n)}\)).
\end{remark}

Next, let \(m \ge 1\) be an integer.
  Below, let \(\proj^{\iter{m}}\) denote the \(m\)-fold iterated composition of \(\proj\). In other words,
  \(\proj^{\iter{m}}:={\proj \circ \cdots \circ \proj}\) repeated \(m\) times.
  \begin{proposition}\label{proposition:projection-of-interior-of-S-iterated}
  Suppose that  \(\calL\) is totally regular. Then
  Let
  \(m \in \{1,\dots, k-2\}\).
  Then
  \(\proj^{\iter{m}}(\cran(\LL)) = \cran^{\circ}(\con^{\iter{m}}[\LL])\).
  \end{proposition}

  \begin{proof}
    We prove by induction. The case when \(m=1\) is simply
Proposition~\ref{proposition:projection-of-interior-of-S}.
Now, suppose that the result holds for \(m\) where \(1<m < k-2\).
Then
\begin{align*}
  \proj^{\iter{m+1}}(\cran(\LL))
  &=
    \proj(
  \proj^{\iter{m}}(\cran(\LL))
    )\\
  &=
    \proj( \cran^{\circ}(\con^{\iter{m}}[\LL]) )
    \qquad \because \mbox{Induction hypothesis}\\
  &=
    \cran^{\circ}(\con[\con^{\iter{m}}[\LL]])
    \qquad \because
    \mbox{
  Proposition~\ref{proposition:projection-of-interior-of-S}
    }
  \\
  &=
    \cran^{\circ}(\con^{\iter{m+1}}[\LL]).
\end{align*}
This completes the induction step and the desired result follows.
  \end{proof}
% \subsection{Sufficient condition for classification calibration}

  Before presenting the
proof of Theorem~\ref{theorem: nested family of regular PERM losses are CC - exposition version}, we recall \citep[Theorem 7]{tewari2007consistency}:
\begin{theorem}[\citet{tewari2007consistency}]\label{theorem: equivalent condition for classification calibration}
  Let \(S \subseteq \mathbb{R}^k_+\) be a symmetric convex set. Then \(S\) is classification calibrated if and only if
  1.\ \(S\) is admissible and
  2.\ \(\proj^{\iter{m}}(S)\) is admissible for all \(m \in \{1,\dots, k-2\}\).
\end{theorem}

Finally, we conclude with the

\noindent\textbf{Proof of Theorem~\ref{theorem: nested family of regular PERM losses are CC - exposition version}}\,\,
  Assume that we are in the situation stated at the beginning of Section~\ref{section:proof theorem totally regular PERM loss is CC}.
  Let \(S = \cran(\LL^{(k)})\).
  By Theorem~\ref{theorem: equivalent condition for classification calibration},
  it suffices to prove that
  1.\ \(S\) is admissible and
  2.\ \({\proj^{\iter{m}}}(S)\) is admissible for all \(m \in \{1,\dots, k-2\}\).
  First, from Proposition~\ref{proposition: main result on admissibility}, \(S = \cran(\LL^{(k)})\) is admissible.
  Next, for each \(m \in \{1,\dots, k-2\}\), we have by
  Proposition~\ref{proposition:projection-of-interior-of-S-iterated}
  that
  \(\proj^{\iter{m}}(\cran(\LL^{(k)})) = \cran^{\circ}(\LL^{(k-m)})\).
  Since \(\calL^{(k-m)}\) is a regular PERM loss with reduced form \(\LL^{(k-m)}\), we have again by Proposition \ref{proposition: main result on admissibility} that \(\cran^{\circ}(\LL^{(k-m)})\) is admissible.
\hfill \(\blacksquare\)

\section{Classification-Calibration of Fenchel-Young losses}\label{section:FY-loss-is-CC}
The goal of this section is two fold. The first subsection
(Section~\ref{section:proof-of-sufficient-condition-for-FY})
presents the proof of Theorem~\ref{theorem: sufficient condition for FY loss to be CC - exposition version} on a sufficient condition for the classification-calibration of Fenchel-Young losses.
The second subsection
(Section~\ref{section:non-strongly-convex-totally-regular-entropy})
shows the existence of a totally regular negentropy that is strictly convex, but not strongly convex.

To prepare for the later proofs, we introduce a more convenient parametrization of the the negenetropy and its Fenchel-Young loss.
Define the \emph{reduced $k$-probability simplex} as
\begin{equation}
\label{equation-definition:reduced-simplex}
  \textstyle\tilde{\Delta}^{k} := \{ \tilde{\bfp} := (p_{1},\dots, p_{k-1}) \in [0,1]^{k-1}: \sum_{i=1}^{k-1} p_{i}\le 1 \}.
\end{equation}
  In other words, $\tilde{\Delta}^{k}$ is simply $\Delta^{k}$ without the first coordinate.
  To every function $\Omega : \Delta^k \to \mathbb{R}$ with domain on the $k$-simplex,
  we define a corresponding function \(\tilde{\Omega} : \tilde{\Delta}^k \to \mathbb{R}\) called the \emph{reduced form} of \(\Omega\), defined by
\begin{equation}
  \label{equation:simplex-reduced-form}
  \textstyle
\tilde{\Omega}(\tilde{\bfp})
  :=
  \Omega\left( p_{1},\,\dots,\, p_{k-1},\,1-\sum_{i=1}^{k-1} p_i\right),
  \quad
  \mbox{for all }
\tilde{\bfp} = (p_{1},\,\dots,\,p_{k-1})^{\top} \in \tilde \Delta^k.
\end{equation}
 \Cref{equation:simplex-reduced-form} induces a one-to-one correspondence between functions $\Omega : \Delta^{k} \to \mathbb{R}$ on the simplex $\Delta^{k}$ and functions $\tilde{\Omega} : \tilde{\Delta}^{k} \to \mathbb{R}$ on the reduced simplex $\tilde{\Delta}^{k}$.

\subsection{Proof of Theorem~\ref{theorem: sufficient condition for FY loss to be CC - exposition version}}\label{section:proof-of-sufficient-condition-for-FY}
Before proceeding with the proof, we establish two key results.
Recall that \(\tilde{\Delta}^{k}\) is defined in Eqn.~\eqref{equation-definition:reduced-simplex}
and
\(\tilde{\Omega}\) in Eqn.~\ref{equation:simplex-reduced-form}.
\begin{proposition}\label{proposition: FY loss is PERM}
Let \(\Omega\) be a negentropy (Definition~\ref{definition:negentropy}) and \(\mu \in \mathbb{R}\).
Then the Fenchel-Young loss \(\calL^{\Omega,\mu}\) associated to \(\Omega\) and \(\mu\) (Definition~\ref{definition:Fenchel-Young}) is a PERM loss that is closed, convex, and non-negative.
The template \(\TP^{\Omega,\mu}\) of \(\calL^{\Omega,\mu}\) is semi-coercive and is given by
\begin{equation}
\TP^{\Omega,\mu}(\bfz)
=
\max_{\tilde \bfp  = [p_{1},\dots, p_{k-1}]\in \tilde{\Delta}^k}
-\tilde\Omega(\tilde{\bfp})
+
\mu \vone^{\top} \tilde{\bfp}
- \langle \tilde{\bfp}, \bfz\rangle.
\label{equation:FY-template}
\end{equation}
Furthermore, if \(\Omega\) is a regular negentropy, then \(\calL\) is a regular PERM loss.
\end{proposition}
% See Proposition \ref{proposition: FY loss is PERM} in the Appendix.

\begin{remark}
We note that results from \cite{blondel2020learning} may be used to prove portions of Proposition~\ref{proposition: FY loss is PERM}.
Combining their Proposition 1-``Order preservation'' part with the expression immediately following their Definition 2 can be used to prove that \(\calL^{\Omega,\mu}\) is permutation equivariant.
Moreover, the expression in our
Eqn.~\eqref{equation:FY-template} is related to their Eqn.~(15).
However, there are key technical differences since  we work with the (\(k-1\)-dimensional) relative margin, while \cite{blondel2020learning} use the class-score formulation.
For this reason and for the reader's convenience, we prove Proposition~\ref{proposition: FY loss is PERM} without leveraging their results.
\end{remark}

\begin{proof}
  In this proof, all elementary basis vectors are implicitly assumed to be \(k\)-dimensional, i.e., we write \(\bfe_{y}\) instead of \(\bfe_{y}^{(k)}\).
First, recall that the Fenchel conjugate of a closed convex function is again closed convex \citep{rockafellar1970convex}.
Next, we show that \(\calL\) is permutation equivariant.
First, we show that \(\calL\) is symmetric.
By Lemma~\ref{lemma:transposition-suffices-for-permutation}, it suffices to prove
the claim that \(\calL(\Tsp_j \bfv) = \Tsp_j \calL(\bfv)\) for all \(j \in [k]\) and \(\bfv \in \mathbb{R}^{k}\).
To this end, let \(y \in [k]\) be arbitrary.
Recall that the Fenchel-Young loss is defined as
\begin{equation}
  \textstyle
  [\calL(\bfv)]_{y}
  =
  \max_{\bfp \in \Delta^k}
- \Omega(\bfp)
  +
\langle \bfc_y, \bfp -\bfe_y \rangle
+
\langle \bfv, \bfp -\bfe_y \rangle
\label{equation:FY-loss-restated}
\end{equation}
Thus,
\begin{align*}
  [\Tsp_{j} \calL(\bfv)]_{y}
  &=
  [\calL(\bfv)]_{\tsp_{j}(y)}
  =
  \max_{\bfp \in \Delta^k}
  - \Omega(\bfp)
  +
  \langle \bfv + \bfc_{\tsp_{j}(y)}, \bfp -\bfe_{\tsp_{j}(y)} \rangle
  \quad \because{\mbox{definition of \(\calL\)}}
    \\
  &=
  \max_{\bfp \in \Delta^k}
  - \Omega(\bfp)
  +
    \langle \Tsp_{j}(\bfv + \bfc_{\tsp_{j}(y)}), \Tsp_{j}(\bfp -\bfe_{\tsp_{j}(y)}) \rangle
\quad \because{\mbox{\(\Tsp_{j}\) is an isometry}}
    \\
  &=
  \max_{\bfp \in \Delta^k}
  - \Omega(\bfp)
  +
    \langle \Tsp_{j}\bfv + \bfc_{\tsp_{j}(\tsp_{j}(y))}, \Tsp_{j}\bfp -\bfe_{\tsp_{j}(\tsp_{j}(y))} \rangle
    \\
  &=
  \max_{\bfp \in \Delta^k}
  - \Omega(\bfp)
  +
    \langle \Tsp_{j}\bfv + \bfc_{y}, \Tsp_{j}\bfp -\bfe_{y} \rangle
\quad \because{\mbox{\(\tsp_{j}\circ \tsp_{j}\) is the identity}}
  \\
  &=
  \max_{\bfp \in \Delta^k}
  - \Omega(\Tsp_{j}\bfp)
  +
    \langle \Tsp_{j}\bfv + \bfc_{y}, \Tsp_{j}\bfp -\bfe_{y} \rangle
\quad \because{\mbox{\(\Omega\) is symmetric}}
    \\
  &=
  \max_{\bfp \in \Delta^k}
  - \Omega(\bfp)
  +
  \langle \Tsp_{j}\bfv + \bfc_{y}, \bfp -\bfe_{y}) \rangle
\quad \because{\mbox{\(\Tsp_{j}|_{\Delta^{k}}\) is a bijection}}
  \\
  &=
  [\calL(\Tsp_{j}\bfv)]_{y} \quad \because \mbox{definition of \(\calL\).}
\end{align*}
Since \(y\) is arbitrary, we have proven the claim.

Next, we show that $\calL$ is relative margin-based, i.e.,
\cref{equation:FY-loss-restated} depends on \(\bfv\) only through \(\margin\bfv\). Recall that \(\margin \bfv
=
\begin{bmatrix}
  v_{k} - v_{1} & v_{k} - v_{2} & \cdots & v_{k} - v_{k-1}
\end{bmatrix}^{\top}
\). See Definition~\ref{definition:relative-marginalization-mapping}.
Now, only the ``\(\langle \bfv, \bfp -\bfe_y \rangle\)'' term  of
\cref{equation:FY-loss-restated}
 depends on $\bfv$. Thus, it suffices to prove that this term only depends on \(\margin \bfv\). To this end, observe that
\begin{align*}
  \langle \bfv, \bfp \rangle
  =
  p_1v_1 + \cdots + p_k v_k
  &=
  p_{1}v_1 +  \cdots + p_{k-1} v_{k-1} + (1 - (p_{1} + \cdots + p_{k-1})) v_k
  \\
  &=
  v_k -
  \left(
  p_1 (v_k - v_{1})
  +
  \cdots +
  p_{k-1}(v_{k} - v_{k-1})\right)
  = v_k - \langle \tilde{\bfp}, \margin \bfv\rangle
\end{align*}
where we write $\tilde{\bfp}$ to denote the vector $[p_{1},\dots, p_{k-1}]^{\top}$.
Thus, for each \(y \in [k-1]\), we have
\[
\langle \bfv , \bfp -\bfe_y \rangle
=
\langle \bfv, \bfp \rangle
-v_y
=
v_{k} - v_y
 - \langle \tilde{\bfp}, \margin \bfv\rangle
 =
 [\margin \bfv]_{y}
 - \langle \tilde{\bfp}, \margin \bfv\rangle
\]
This shows that $\calL$ is relative margin-based.
Now the template of \(\calL\) is
\[
  \textstyle
  \TP(\bfz)
  =
  \LL_{k}(\bfz)
  =
  \max_{\bfp \in \Delta^k}
  -\Omega(\bfp)
  +
  \langle \bfc_k, \bfp - \bfe_k \rangle
  - \langle \tilde{\bfp}, \bfz\rangle.
\]
Since \(\bfe_k \in \Delta^k\) and \([\bfe_k]_{y} = 0\) for \(y \in [k-1]\), we have by construction that
\[
  \TP(\bfz)
  \ge
  -\Omega(\bfe_1)
  +
  \langle \bfc_1, \bfe_1 - \bfe_1 \rangle
  - \langle \vzero, \bfz\rangle
  =0.
\]

When \(\bfc_k = \mu(\vone - \bfe_k)\), we have
\begin{equation}
  \label{equation:FY-template-case-1}
  \textstyle
  \TP(\bfz)
  =
  \max_{\tilde \bfp \in \tilde{\Delta}^k}
  -\tilde\Omega(\tilde{\bfp})
  +
  \mu \vone^{\top}  \tilde{\bfp}
  - \langle \tilde{\bfp}, \bfz\rangle, \quad \mbox{for all \(\bfz \in \mathbb{R}^{k-1}\).}
\end{equation}
Finally, we prove that $\TP$ is semi-coercive (Definition~\ref{definition:coercive-functions}), i.e., for all \(c \in \mathbb{R}\) there exists \(b \in \mathbb{R}\) such that
 \(\{\bfz \in \mathbb{R}^{k-1} : \TP(\bfz) \le c\} \subseteq \{ \bfz \in \mathbb{R}^{k-1} : b \le \min \bfz\}\).
To this end, let $c \in \mathbb{R}$ and $\bfz \in \mathbb{R}^{k-1}$ be such that $c \ge \TP(\bfz)$.
Let $j \in \argmin \bfz$.
Then since $\bfe_j = \bfe_j^{k-1} \in \tilde{\Delta}^k$, we have
\[
  \textstyle
  c \ge
  \TP(\bfz)
  =
  \sup_{\bfp \in \tilde{\Delta}^k} -\tilde \Omega(\bfp) - \langle \bfp,\bfz\rangle
  \ge
  -\tilde \Omega(\bfe_j) + \mu - \langle \bfe_j, \bfz\rangle
  \ge
  -z_j
  =
  -\min \bfz.
\]
Thus, we have shown that
 \(\{\bfz \in \mathbb{R}^{k-1} : \TP(\bfz) \le c\} \subseteq \{ \bfz \in \mathbb{R}^{k-1} : \min \bfz \ge -c\}\).
% \end{proof}
% \begin{proof}

Next, we prove the ``Furthermore'' part.
  By the first part, it remains to show that $\TP$ is strictly convex, twice differentiable and $\nabla_{\TP}(\bfz) \prec \vzero$ for all $\bfz \in \mathbb{R}^{k-1}$.
  Define $\varphi(\tilde{\bfp}) := \tilde \Omega(\tilde{\bfp}) - \mu \vone^{\top}  \tilde{\bfp}$.
  Then $\varphi: \tilde{\Delta}^k \to \mathbb{R}$ is also of Legendre type.
  Moreover, note that
  \begin{align*}
    \TP(\bfz)
  &=
    \textstyle
  \max_{\tilde \bfp \in \tilde{\Delta}^k} \langle \tilde{\bfp}, -\bfz \rangle
  -
    \tilde \Omega(\tilde{\bfp}) + \mu \vone^{\top} \tilde{\bfp} \quad \because \mbox{Eqn.~\eqref{equation:FY-template-case-1}}
  \\
  &=
    \textstyle
  \max_{\tilde \bfp \in \tilde{\Delta}^k} \langle \tilde{\bfp}, -\bfz \rangle
  -
  \varphi(\tilde{\bfp})
  = \varphi^*(-\bfz) \quad \because \mbox{definition of Fenchel conjugate}
  \end{align*}

  Recall the following fundamental theorem regarding convex conjugates \citep{rockafellar1970convex}.
\begin{theorem}[\citet{rockafellar1970convex}]\label{theorem: fundamental theorem of legendre transformation}
  If \((C,f)\) is a convex function of Legendre type, then \((C^*, f^*)\) is a convex function of Legendre type.
  The map \(\nabla_f : C \to C^*\) is a homeomorphism and \(\nabla_{f^*} = (\nabla_f)^{-1}\).
\end{theorem}
  By Theorem~\ref{theorem: fundamental theorem of legendre transformation}, we have
  \begin{enumerate}
    \item
      The function
      \(\varphi^*\), and hence \(\TP\), is of Legendre type.
      In particular, \(\TP\) is strictly convex.
    \item
  The derivative
  \(\nabla_{\varphi} : \interior(\tilde{\Delta}^k) \to \mathbb{R}^{k-1}\) is a bijection and
  the derivative of \(\varphi^*\) satisfies
  \(\nabla_{\varphi^*} = {(\nabla_{\varphi})}^{-1} : \mathbb{R}^{k-1} \to \interior(\tilde{\Delta}^k)\).
  \end{enumerate}
  It follows that if \(\varphi\) is twice differentiable, then so is \(\varphi^*\).
Finally, by the chain rule, we have
\(
  \nabla_{\TP}(\bfz)
  =
  -\nabla_{\varphi^*}(\bfz).
\)
Since \(\nabla_{\varphi^*}(\bfz) \in \interior(\tilde{\Delta}^k)\) for all \(\bfz\), we have
in particular that \(\nabla_{\varphi^*}(\bfz) \succ \vzero\). Thus, \(\nabla_{\TP}(\bfz) \prec \vzero\) for all \(\bfz\).
\end{proof}

\begin{proposition}\label{proposition:trunc-of-FY-loss}
  Let \(k \ge 3\) be an integer, \(\Omega: \Delta^{k} \to \mathbb{R}\) a regular negentropy and \(\mu \ge 0\).
  Let \(\calL^{\Omega,\mu}\) be the Fenchel-Young loss corresponding to \(\Omega\) and \(\mu\).
Then
\(\con \left[
\calL^{\Omega,\mu}
\right]
=
\calL^{
\con \left[
  \Omega\right],\mu}
\).
\end{proposition}
\begin{proof}
  % n is k
  % m is n
  Throughout this proof, let \(\Theta := \con[\Omega]\).
  Let \(\tilde{\Omega}\) and \(\tilde{\Theta}\) be the reduced
  versions
(see \cref{equation:simplex-reduced-form})
  of \(\Omega\) and \(\Theta\), respectively.
  We remark that the notation used in this proof slightly departs from that of \cref{equation:simplex-reduced-form} in the following way: Here, elements of \(\tilde{\Delta}^k\) will be denoted as \(\bfp\) instead of \(\tilde{\bfp}\). Likewise, elements of \(\tilde{\Delta}^{k-1}\) will be denoted as \(\bfq\) rather than \(\tilde{\bfq}\).
Moreover, throughout this proof, we abuse notation and denote by \(\vone\) either the \((k-1)\) or the \((k-2)\)-dimensional vector. The dimension of \(\vone\) will be clear from the context.

  Now, from
\cref{equation:FY-template} in
  Proposition~\ref{proposition: FY loss is PERM}, recall that the template of \(\calL^{\Omega, \mu}\) is
\begin{equation}
  \TP^{\Omega,\mu} : \mathbb{R}^{k-1} \to \mathbb{R}, \mbox{ where} \quad
\TP^{\Omega,\mu}(\bfz)
=
\max_{\bfp  = [p_{1},\dots, p_{k-1}]\in \tilde{\Delta}^k}
-\tilde\Omega({\bfp})
+
\mu \vone^{\top} {\bfp}
- \langle \bfz, {\bfp} \rangle.
\label{equation:truncation-of-FY-loss-template-of-L-Omega}
\end{equation}
Similarly, the template of \(\calL^{\Theta,\mu}\) is
\begin{equation}
  \TP^{\Theta,\mu} : \mathbb{R}^{k-2} \to \mathbb{R}, \mbox{ where} \quad
\TP^{\Theta,\mu}(\bfw)
=
\max_{\bfq  = [q_{1},\dots, q_{k-2}]\in \tilde{\Delta}^{k-1}}
-\tilde\Theta({\bfq})
+
\mu \vone^{\top} {\bfq}
- \langle \bfw, {\bfq} \rangle.
\label{equation:truncation-of-FY-loss-template-of-L-Theta}
\end{equation}
By the definition of the truncation of a PERM loss (given in
Proposition~\ref{proposition:truncation-of-a-PERM-loss}), the template of \(\con[\calL^{\Omega,\mu}]\) is
\(
\con[\TP^{\Omega,\mu}]\).
Since PERM losses are uniquely defined by their templates (Theorem~\ref{theorem:relative-margin-form}), the result of
{Proposition}~\ref{proposition:trunc-of-FY-loss} can be equivalently stated as
\(
 \con[ \TP^{\Omega,\mu}]
 =
 \TP^{\Theta,\mu}
 \).
 Below, our proof will focus on proving this identity.

 Both
 \(
 \con[ \TP^{\Omega,\mu}]
 \) and \(
 \TP^{\Theta,\mu}
 \) have \(\mathbb{R}^{k-2}\) as domain.
 Fix \(\bfw \in \mathbb{R}^{k-2}\) arbitrarily. Our goal is to show that
\(
 \con[ \TP^{\Omega,\mu}](\bfw)
 =
 \TP^{\Theta,\mu}(\bfw)
 \).
 Pick \(\bfz \in \mathbb{R}^{k-1}\) such that \(\proj \bfz = \bfw\). For instance, we can pad \(\bfw\) with a zero, i.e., take \(\bfz := [0, \bfw] = \inj(\bfw)\).
By Lemma~\ref{lemma:truncation-alternative-characterization}, we have
\[
 \con[ \TP^{\Omega,\mu}](\bfw)
 =
 \lim_{\lambda \to +\infty} \TP^{\Omega,\mu}(\bfz + \lambda \bfe_{1}^{(k-1)})
\]
  To simplify notations, let \(\bfe_{1} := \bfe^{(k-1)}_{1}\).
  In view of
  \cref{equation:truncation-of-FY-loss-template-of-L-Omega}
  and
\cref{equation:truncation-of-FY-loss-template-of-L-Theta},
establishing our goal, i.e., the identity
\(
 \con[ \TP^{\Omega,\mu}](\bfw)
 =
 \TP^{\Theta,\mu}(\bfw)
 \),
 is equivalent to proving
  \begin{equation}
    \lim_{\lambda \to + \infty}
\max_{\bfp \in \tilde{\Delta}^k}
% \tilde H^{(k)}
-\tilde{\Omega}
(\bfp)
+
\mu \vone^{\top}  \bfp
-\langle \bfz + \lambda \bfe_{1}, \bfp \rangle
=
\max_{\bfq\in \tilde{\Delta}^{k-1}}
% \tilde H^{(k-1)}
-\tilde{\Theta}
(\bfq)
+
\mu \vone^{\top}  \bfq
-\langle \proj \bfz, \bfq\rangle.
\label{equation:fenchel-young-cc-nested-main-equation}
\end{equation}

  For brevity, we define \begin{equation}
g(\lambda) :=
\max_{\bfp \in \tilde{\Delta}^k}
-\tilde{\Omega}
(\bfp)
+
\mu \vone^{\top}  \bfp
-\langle \bfz + \lambda \bfe_{1}, \bfp \rangle.
\label{equation:FY-CC-proof-definition-of-g}
\end{equation}
For each \(\lambda \in \mathbb{R}\), fix arbitrarily an element \(\bfp^{\lambda} \in \argmax_{\bfp \in \tilde{\Delta}^k}-\tilde{\Omega}
(\bfp) + \mu \vone^{\top} \bfp -\langle \bfz + \lambda \bfe_{1}, \bfp \rangle\). We note that the maximization is over a compact domain with a continuous objective.
We will prove the result via a series of claims.

\noindent\underline{Claim 1}: \(g: \mathbb{R} \to \mathbb{R}_+\), defined at
\cref{equation:FY-CC-proof-definition-of-g}, is monotone non-increasing.

\iftoggle{arxiv}{
\noindent\underline{Proof of Claim 1}:
For \(\lambda, \nu \in \mathbb{R}\) such that \(\lambda \le \nu\), we observe that
\begin{align*}
  g(\nu)
  &=
-\tilde{\Omega} (\bfp^{\nu})+
\mu \vone^{\top}  \bfp^{\nu}
-\langle \bfz + \nu \bfe_{1}, \bfp^{\nu}\rangle
\qquad \because \mbox{Definitions of \(g(\nu)\) and \(\bfp^{\nu}\)}
\\
&\le
 -\tilde{\Omega}(\bfp^{\nu})
+
\mu \vone^{\top}  \bfp^{\nu}
-\langle \bfz + \lambda \bfe_{1}, \bfp^{\nu}\rangle
\qquad \because \nu \ge \lambda
\\
&=
g(\lambda)
\qquad \because \mbox{Definition of \(g(\lambda)\).}
\end{align*}
This proves the Claim 1.
For the rest of the proof,
fix once and for all an element \begin{equation}
  \textstyle
  \widehat{\bfq} \in \argmax_{\bfq\in \tilde{\Delta}^{k-1}}
  -\tilde{\Theta}(\bfq) + \mu\vone^{\top}  \bfq-\langle \bfw, \bfq\rangle.
\label{equation:FY-CC-proof-definition-of-q-hat}
\end{equation}
}{}
\noindent\underline{Claim 2}: For all \(\lambda \in \mathbb{R}\),
\(% \label{equation: g of lambda is lower bounded}
g(\lambda)
\ge
-\tilde{\Theta} (\widehat{\bfq})
  + \mu \vone^{\top} \widehat{\bfq}
  -\langle \bfw , \widehat{\bfq}\rangle.\)

\iftoggle{arxiv}{
  \noindent\underline{Proof of Claim 2}:
To proceed, first  define $\widehat{\bfr}:=
\inj(\widehat{\bfq}) \in \tilde{\Delta}^k$.
Then we observe that
\begin{align*}
  g(\lambda)\ge
 -\tilde{\Omega}(\widehat{\bfr})
  + \mu \vone^{\top} \widehat{\bfr}
  -\langle \bfz + \lambda \bfe_{1}, \widehat{\bfr}\rangle
  &=
 -\tilde{\Omega}(\widehat{\bfr})
  + \mu \vone^{\top} \widehat{\bfr}
  -\langle \bfz , \widehat{\bfr}\rangle
  \quad \because \langle \bfe_{1}, \widehat{\bfr}\rangle = 0
  \\
  &=
-\tilde{\Theta}(\widehat{\bfq})
  + \mu \vone^{\top} \widehat{\bfq}
  -\langle \bfz , \widehat{\bfr}\rangle
  \quad \because \vone^{\top}  \widehat{\bfq} = \vone^{\top}  \widehat{\bfr}
  \\
  &=
 -\tilde{\Theta}(\widehat{\bfq})
  + \mu \vone^{\top} \widehat{\bfq}
  -\langle \bfw , \widehat{\bfq}\rangle.
\end{align*}
This proves
  Claim 2.
Below, fix a sequence $\{\lambda_{t}\}_{t=1,2,\dots}$ such that $\lim_{t\to \infty}\lambda_t = +\infty$ and \(\bfp^{\lambda_t} \to \overline{\bfp} = [\overline{p}_{1},\dots, \overline{p}_{k}]^{\top} \in \tilde{\Delta}^k\) as $t \to \infty$.
Such a sequence exists because $\tilde{\Delta}^{k}$ is compact.
}{}

\noindent\underline{Claim 3}: $\overline{p}_{1} = 0$.
\iftoggle{arxiv}{
\noindent\underline{Proof of Claim 3}:
Suppose that this is false.
Then for all $t$ sufficiently large, there exists an $\epsilon > 0$ such that
\([\bfp^{\lambda_{t}}]_{k} =
  p^{\lambda_{t}}_{k}
  \ge \epsilon\).
Now, we have
\begin{align*}
  g(\lambda_{t})
  &=
 -\tilde{\Omega}(\bfp^{\lambda_{t}})
  +\mu\vone^{\top}  \bfp^{\lambda_{t}}
  -\langle \bfz + \lambda_{t} \bfe_{1}, \bfp^{\lambda_{t}}\rangle
  \quad \because \mbox{definition of \(g\) (\cref{equation:FY-CC-proof-definition-of-g})}
  \\
  &=
 -\tilde{\Omega}(\bfp^{\lambda_{t}})
  +\mu\vone^{\top}  \bfp^{\lambda_{t}}
  -\langle \bfz, \bfp^{\lambda_{t}}\rangle -\lambda_{t} \langle \bfe_{1}, \bfp^{\lambda_{t}}\rangle
  \\
  &\le
  g(0)
  -\lambda_{t} \langle \bfe_{1}, \bfp^{\lambda_{t}}\rangle
  \quad \because \mbox{definition of \(g(0)\) (\cref{equation:FY-CC-proof-definition-of-g})}
  \\
  &\le
  g(0)
  -\lambda_{t}
  p^{\lambda_{t}}_1
  \quad \because \mbox{definition of \(\bfe_{1} := \bfe^{(k-1)}_{1}\)}
  \\
  &\le
  g(0)
  -\lambda_{t}
  \epsilon
    \quad \because
  [\bfp^{\lambda_{t}}]_{k} =
  p^{\lambda_{t}}_{k}
  \ge \epsilon.
\end{align*}
Thus, we have $\lim_{t\to \infty} g(\lambda_{t}) = -\infty$, which contradicts Claim 2. This proves Claim 3.
}{}

\noindent \underline{Claim 4}:
\(\lim_{t\to \infty} g(\lambda_t)
  =
 -\tilde{\Omega}(\widehat{\bfq})
  + \mu \vone^{\top}  \widehat{\bfq}
  -\langle \bfw , \widehat{\bfq}\rangle\)

  \iftoggle{arxiv}{
\noindent \underline{Proof of Claim 4}:
Define $\overline{\bfq} := \proj\overline{\bfp}$.
Note that Claim 3 implies that $\inj(\overline{\bfq}) = \overline{\bfp}$.
Then
\begin{align*}
  \lim_{t\to \infty}
  g(\lambda_{t})
  &=\lim_{t\to \infty}
-\tilde{\Omega}(\bfp^{\lambda_{t}})
  + \mu\vone^{\top}  \bfp^{\lambda_{t}}
  - \langle \bfz, \bfp^{\lambda_{t}}\rangle
  \\
  &=
 -\tilde{\Omega}(\overline{\bfp})
  + \mu\vone^{\top}  \overline{\bfp}
  - \langle \bfz, \overline{\bfp}\rangle
  \quad \because \mbox{continuity of \(\tilde{H}^{(k)}\)}
  \\
  &=
 -\tilde{\Omega}(\inj(\overline{\bfq}))
  + \mu\vone^{\top} \inj(\overline{\bfq})
  - \langle \bfz, \inj(\overline{\bfq})\rangle
  \quad \because \mbox{Claim 3}
  \\
  &=
 -\tilde{\Omega}(\inj(\overline{\bfq}))
  + \mu\vone^{\top} \overline{\bfq}
  - \langle \bfw, \overline{\bfq}\rangle
  \\
  &=
 -\tilde{\Theta}(\overline{\bfq})
  + \mu\vone^{\top} \overline{\bfq}
  - \langle \bfw, \overline{\bfq}\rangle
  \quad \because \mbox{\(\Theta = \con[\Omega]\)}
  \\
  &\le
 -\tilde{\Theta}(\widehat{\bfq}) +\mu \vone^{\top} \widehat{\bfq}- \langle \bfw, \widehat{\bfq}\rangle
  \quad \because \mbox{Definition of \(\widehat{\bfq}\) (\cref{equation:FY-CC-proof-definition-of-q-hat}).}
\end{align*}

From Claim 2, we get the other inequality:
\(  \lim_{t\to \infty}g(\lambda_{t})
   \ge
 -\tilde{\Theta}(\widehat{\bfq})
  + \mu \vone^{\top}  \widehat{\bfq}
  -\langle \bfw , \widehat{\bfq}\rangle
\). This proves Claim 4
as desired.
}{}

\iftoggle{arxiv}{}{\textcolor{black}{The proofs of these claims proceed via straightforward but tedious computations and applications of inequalities. They are omitted here but appear in the arXiv version of this work \citep{wang2023unified}.}}
Now we finish the proof of
\cref{equation:fenchel-young-cc-nested-main-equation}, which we argued earlier suffices to establish Proposition~\ref{proposition:trunc-of-FY-loss}.
By Claims 1 and 4, we have that
\(  \lim_{\lambda \to +\infty}g(\lambda)
=
 -\tilde{\Theta}(\widehat{\bfq})
  + \mu \vone^{\top}  \widehat{\bfq}
  -\langle \bfw , \widehat{\bfq}\rangle
\).
But \(\lim_{\lambda \to +\infty}g(\lambda)\) is exactly the left hand side of
\cref{equation:fenchel-young-cc-nested-main-equation},
by the construction of \(g\) in \cref{equation:FY-CC-proof-definition-of-g}.
By the definition of \(\widehat{\bfq}\), moreover we have that
\(
 -\tilde{\Theta}(\widehat{\bfq})
  + \mu \vone^{\top}  \widehat{\bfq}
  -\langle \bfw , \widehat{\bfq}\rangle
  \)
is precisely the right hand side of
\cref{equation:fenchel-young-cc-nested-main-equation}.
\end{proof}

\begin{theorem}\label{proposition: nested Delta family is nested PERM}
  Let \(\Omega : \Delta^{k} \to \mathbb{R}\) be a negentropy (Definition~\ref{definition:negentropy}).
  For each \(n \in \{2,\dots,k\}\), let
  \(\Omega^{(n)}\) be the \(n\)-ary truncated negentropy of \(\Omega\) (Definition~\ref{definition:totally-regular-negentropy}).
  Let \( \calL^{\Omega,\mu}\) be the Fenchel-Young loss corresponding to \(\Omega\) and \(\mu\) (Definition~\ref{definition:Fenchel-Young}).
  Likewise,
  for each \(n \in \{2,\dots,k\}\), let \(\calL^{(n)}\) be the \(n\)-ary truncation of \(\calL^{\Omega,\mu}\) (Definition~\ref{definition:totally-regular-PERM-loss}).
  Then
  \(\calL^{(n)} = \calL^{\Omega^{(n)},\mu} \). In other words,
  the \(n\)-ary truncated loss \(\calL^{(n)}\)  of \(\calL^{\Omega,\mu}\) is equal to the Fenchel-Young loss of the \(n\)-ary truncated negentropy \(\Omega^{(n)}\) and \(\mu\).
\end{theorem}
\iftoggle{arxiv}{
\begin{proof}
  Unwinding the definitions of \(\calL^{(n)}\) and \(\Omega^{(n)}\), we have
  \[
    \calL^{(n)} := \con^{\times k-n}[\calL^{\Omega,\mu}]
    \quad
  \mbox{and} \quad
    \Omega^{(n)} := \con^{\times k-n} [\Omega]
  \]
  We proceed via a ``downward'' induction on \(n\).
Assume the induction hypothesis: Theorem~\ref{proposition: nested Delta family is nested PERM} holds for \(n \in \{3,\dots, k\}\), i.e.,
\(\calL^{(n)} = \calL^{\Omega^{(n)},\mu} \).
Our goal is to show that
\(\calL^{(n-1)} = \calL^{\Omega^{(n-1)},\mu} \) as well.
Note that
\begin{align*}
\calL^{(n-1)}:= \con^{\times k-n+1}[\calL^{\Omega,\mu}]
=
\con [
\con^{\times k-n}[\calL^{\Omega,\mu}]
]
=
\con [
  \calL^{(n)}
]
=
\con [
\calL^{\Omega^{(n)},\mu}
].
\end{align*}
% \begin{align}
% \calL^{(n-1)}&:= \con^{\times k-n+1}[\calL^{\Omega,\mu}]
%                \\
% &=
% \con \left[
% \con^{\times k-n}[\calL^{\Omega,\mu}]
% \right]
%   \\
% &=
% \con \left[
%   \calL^{(n)}
% \right]
%   \\
% &=
% \con \left[
% \calL^{\Omega^{(n)},\mu}
% \right].
% \end{align}
where the last equality is the induction hypothesis.
Furthermore, note that
\[
\Omega^{(n-1)} := \con^{\times k-n+1}[\Omega]
=
\con[ \con^{\times k -n } [ \Omega] ]
=
\con[ \Omega^{(n)}]
\]
Thus, by the above two identities, we have
\begin{align*}
\calL^{(n-1)}=
\con [
\calL^{\Omega^{(n)},\mu}
  ]
=
\calL^{
\con [
  \Omega^{(n)}],\mu}
  =
\calL^{
  \Omega^{(n-1)},\mu}
\end{align*}
where the middle equality follows from
Proposition~\ref{proposition:trunc-of-FY-loss}.
\end{proof}
}
{
  The proof of
{Theorem}~\ref{proposition: nested Delta family is nested PERM}
is omitted here since it
largely mirrors that of
{Proposition}~\ref{proposition:projection-of-interior-of-S-iterated}.
See the arXiv version of manuscript \citep{wang2023unified} for the proof.
\vspace{1em}
}

% For the reader's convenience, we restate
% \begin{theorem*}[Theorem~\ref{theorem: sufficient condition for FY loss to be CC - exposition version}, restated]
%   \input{chunks/FYCC.tex}
% \end{theorem*}
\begin{proof}[Proof of Theorem~\ref{theorem: sufficient condition for FY loss to be CC - exposition version}]
  By Theorem~\ref{theorem: nested family of regular PERM losses are CC - exposition version}, it suffices to show that \(\calL^{\Omega,\mu}\) is totally regular (Definition~\ref{definition:totally-regular-PERM-loss}).
  Unwinding the definition, our goal is to show for each \(n \in \{2,\dots, k\}\) that  the \(n\)-ary truncated loss of \(\calL^{\Omega,\mu}\),
denoted \(\calL^{(n)}\),
is regular.

Next, let \(\Omega^{(n)}\) be the \(n\)-ary truncated negentropy of \(\Omega\) (Definition~\ref{definition:totally-regular-negentropy})
and let \(\calL^{\Omega^{(n)},\mu}\) be the Fenchel-Young loss corresponding to \(\Omega^{(n)}\) and \(\mu\).
  By Theorem~\ref{proposition: nested Delta family is nested PERM}, we have
  \(\calL^{(n)} = \calL^{\Omega^{(n)},\mu}\).
  By the assumption of Theorem~\ref{theorem: sufficient condition for FY loss to be CC - exposition version}, \(\Omega^{(n)}\) is a regular negentropy.
  Thus, by Proposition~\ref{proposition: FY loss is PERM}, \(\calL^{\Omega^{(n)},\mu}\) is a regular PERM loss. Since
  \(\calL^{(n)} = \calL^{\Omega^{(n)},\mu}\), we are done.
\end{proof}

% The proof of Theorem~\ref{theorem: sufficient condition for FY loss to be CC} will be proven at the end of this section.

\subsection{Totally regular negentropy that is not strongly convex}\label{section:non-strongly-convex-totally-regular-entropy}

In this section, we prove
Proposition~\ref{proposition:non-strongly-convex-negentropy} which is used to
show that there exists totally regular entropies that are not strongly convex.
See Example~\ref{example:not-strongly-convex-totally-regular-negentropy}.
Thus, the associated Fenchel-Young loss is calibrated by Theorem~\ref{theorem: sufficient condition for FY loss to be CC - exposition version}.
Moreover, this calibration result is outside of the purview of previously established results~\citep{blondel2019structured,nowak2019general} which requires strong convexity

\iftoggle{arxiv}{}
{
  \textcolor{black}{
 The next four results serve as the technical tools for proving
Proposition~\ref{proposition:non-strongly-convex-negentropy}.
Their proofs are omitted here and are in the arXiv version of this work \citep{wang2023unified}.}}
\begin{proposition}
  \label{proposition:Legendre-composed-with-TF-is-again-Legendre}
 Let $f: D \to \mathbb{R}_{\ge 0}$ be of Legendre type and $\sqfunc : \mathbb{R}_{\ge 0} \to \mathbb{R}$ be convex, differentiable and strictly increasing.
 Let $C = \mathrm{int}(D)$.
 Suppose that $D$ is compact and there exists $\bfx^{*} \in C$ such that $\inf_{\bfx \in D} f(\bfx) = f(\bfx^{*})$.
 Then $\sqfunc \circ f : D \to \mathbb{R}$ is of Legendre type.
\end{proposition}
\iftoggle{arxiv}
{
\begin{proof}
  We check that the items of Definition~\ref{definition: Legendre type} hold.
  Item~\ref{definition: Legendre type - domain} clearly holds since $f$ and $\sqfunc \circ f$ have the same domain.

% part 2
  Now for Item~\ref{definition: Legendre type - function}, note that $\sqfunc \circ f$ is differentiable by the Chain Rule. Thus it remains to show that $\sqfunc \circ f$ is strictly convex.
  For all $\bfx,y \in D$ such that $\bfx \ne y$ and $\lambda \in (0,1)$, we have
  \[
f(\lambda \bfx +  (1-\lambda)y) < \lambda f(\bfx) + (1-\lambda) f(y).
    \]
    This is due to $f$ being strictly convex.
    Next, since $\sqfunc$ is strictly increasing, we have
  \[
    \sqfunc(f(\lambda \bfx +  (1-\lambda)y)) < \sqfunc(\lambda f(\bfx) + (1-\lambda) f(y))
    \]
    By the convexity of $\sqfunc$, we have
    $
\sqfunc(
\lambda f(\bfx) + (1-\lambda) f(y)
)
\le
\lambda
\sqfunc(f(\bfx)) + (1-\lambda)
\sqfunc(
f(y))
)
$
which shows that $\sqfunc \circ f$ is strictly convex.

% part 3
For Item~\ref{definition: Legendre type - boundary}, we check that
 $\lim_{i\to \infty} \|\nabla_{\sqfunc \circ f}(\bfx^i)\| = +\infty$ for all sequences $\{\bfx^i\} \subseteq C$ such that $\lim_{i \to \infty} \bfx^i \in \partial D$.
To this end,  we first prove the claim that there exists $\epsilon > 0$ such that for all sequences $\{ \bfx^{i}\} \subseteq C$ with $\lim_{i \to \infty} \bfx^{i} \in \partial D$ we have $\lim_{i \to \infty} f(\bfx^{i}) \ge \epsilon$.
Since $f$ is convex on $D$, we know that $f$ is continuous on $D$. This is \citet[Corollary 10.1.1]{rockafellar1970convex}.
In particular, $f$ is continuous on $\partial D = D \setminus C$ as well. Since $\partial D$ is compact, we have $\inf_{\bfx \in \partial D}f(\bfx) = f(\bfx^{\dagger})$ for some $\bfx^{\dagger} \in \partial D$.
Since $\bfx^{\dagger} \ne \bfx^{*}$, we must have $f(\bfx^{\dagger}) \ne f(\bfx^{*})$ by the strict convexity of $f$.
In particular, $f(\bfx^{\dagger}) > f(\bfx^{*})$. Now, letting $\epsilon = f(\bfx^{\dagger})$, the claim follows.

 Next we prove that $\sqfunc'(\epsilon) > 0$.
 Since $\sqfunc$ is increasing, we have $\sqfunc \ge 0$.
 We proceed by considering the two cases $\sqfunc'(0) > 0$ and $\sqfunc'(0) = 0$ separately.
 In the first case, the convexity of $\sqfunc$ implies that $\sqfunc'$ is non-decreasing and so $\sqfunc'(\epsilon) >0$ holds.
 In the second case, if $\sqfunc'(\epsilon) = 0$, then we must have $\sqfunc'(t) = 0$ for all $t \in [0,\epsilon]$.
 But this implies that $\sqfunc$ is constant on $[0,\epsilon]$ which contradicts that $\sqfunc$ is strictly increasing.
 Thus, $\sqfunc'(\epsilon) > 0$.

Finally, by the Chain Rule, we have $\nabla_{\sqfunc \circ f}(\bfx) = \nabla_{\sqfunc}(f(\bfx)) \nabla_{f}(\bfx) = \sqfunc'(f(\bfx)) \nabla_{f}(\bfx)$.
Thus,
\[
  \lim_{i \to \infty} \nabla_{\sqfunc \circ f}(\bfx^{i}) =
  \lim_{i \to \infty} \sqfunc'(f(\bfx^{i}))
  \lim_{i \to \infty} \nabla_{f}(\bfx^{i})
  \ge
  \sqfunc'(\epsilon)
  \lim_{i \to \infty} \nabla_{f}(\bfx^{i}).
\]
Since $\sqfunc'(\epsilon) > 0$ and does not depend on $i$, we have $\lim_{i \to \infty} \| \nabla_{\sqfunc \circ f } (\bfx^{i}) \| = + \infty$, as desired.
\end{proof}
}{}

\begin{lemma}
  \label{lemma:Legendre-not-strongly-convex}
  Let $f$ and $\sqfunc$ be as in Proposition~\ref{proposition:Legendre-composed-with-TF-is-again-Legendre}.
  If $\sqfunc'(0) = 0$ and $\bfx^{*} \in \mathrm{int}(D)$  is such that $\inf_{\bfx \in D} f(\bfx) = f(\bfx^{*}) = 0$, then the Hessian of $\sqfunc \circ f$ vanishes at $\bfx^{*}$, i.e., $\nabla^{2}_{\sqfunc \circ f}(\bfx^{*}) = 0$.
\end{lemma}
\iftoggle{arxiv}
{\begin{proof}
First, we have by the Chain Rule that
  $\nabla_{\sqfunc \circ f}(\bfx) = \sqfunc'(f(\bfx)) \nabla_{f}(\bfx)$ and
  \[\nabla_{\sqfunc \circ f}^{2}(\bfx) = \sqfunc''(f(\bfx)) \nabla_{f}(\bfx)^{\top} \nabla_{f}(\bfx) + \sqfunc'(f(\bfx)) \nabla^{2}_{f}(\bfx).\]
Note that $\nabla_{f}(\bfx)$ is a row vector by our convention. By assumption, we have $\nabla_{f}(\bfx^{*})  = 0$ and $\sqfunc'(f(\bfx)) = \sqfunc'(0) = 0$. Thus, in light of the formula for $\nabla^{2}_{\sqfunc \circ f}(\bfx)$ derived above, we are done.
\end{proof}}{}

  \begin{corollary}\label{corollary:never-strongly-convex}
    In the situation of
  Lemma~\ref{lemma:Legendre-not-strongly-convex},
    $\sqfunc \circ f$ is not $\alpha$-strongly convex for any $\alpha > 0$.
  \end{corollary}

\begin{proposition}\label{proposition:transform-of-regular-negentropy}
Let $\Omega : \Delta^{k} \to \mathbb{R}$ be
a regular negentropy and let
$\sqfunc : \mathbb{R}_{\ge 0} \to \mathbb{R}_{\ge 0}$ be as in
  Proposition~\ref{proposition:Legendre-composed-with-TF-is-again-Legendre}.
  Furthermore, suppose that $\sqfunc$ is twice differentiable.
  Let $a \in \mathbb{R}$ be a negative number such that $a \le \Omega(\bfp)$ for all $\bfp \in \Delta^{k}$.
Define $\Theta : \Delta^{k} \to \mathbb{R}$ by
\[
  \Theta(\bfp) := \sqfunc({\Omega}(\bfp) - a) - \sqfunc(-a),
  \quad \forall \bfp \in \Delta^{k}.
\]
Then $\Omega$ is a regular negentropy.
\end{proposition}
\iftoggle{arxiv}
{\begin{proof}
  We first check that $\Theta$ is a negentropy.
  Clearly, $\Theta$ is symmetric (item 2 of Definition~\ref{definition:negentropy}).
  Below, let $\bfp \in \Delta^{k}$ be arbitrary and let $\tilde{\bfp} = (p_{2},\dots, p_{k})^{\top} \in \tilde{\Delta}^{k}$.

  By assumption on $a$, we have
  $0 \le {\Omega}(\bfp) - a \le -a$.
  Therefore, by $\sqfunc$ being monotone, we have $\sqfunc({\Omega}(\bfp) - {\Omega}(\bfu)) \le \sqfunc(-a)$.
  This proves that $\Theta(\bfp) \le 0$.
  Since $\Omega(\bfe^{(k)}_{i}) = 0$, we have $\Theta(\bfe^{(k)}_{i}) = 0$ as well.
  This proves item 3 of Definition~\ref{definition:negentropy}.

  Next, since $\sqfunc : \mathbb{R}_{\ge 0} \to \mathbb{R}_{\ge 0}$ is continuous and strictly increasing, $\sqfunc$ is a homeomorphism. In particular, $\sqfunc$ is closed.
  Since $\Omega$ and $\sqfunc$ are both closed, it follows that $\Theta$ is also closed.

  It is easy to see that
  $
\tilde{\Theta}(\tilde{\bfp}) := \sqfunc(\tilde{\Omega}(\tilde{\bfp}) - a) - \sqfunc(-a).
$
Thus, $\tilde{\Theta}$ is twice differentiable in the interior of $\tilde{\Delta}^{k}$.
Furthermore, by Proposition~\ref{proposition:Legendre-composed-with-TF-is-again-Legendre}, we get that
$\tilde{\Theta}$ is of Legendre type.
In particular, $\tilde{\Theta}$ is strictly convex, and so $\Theta$ is convex. This proves item 1 of Definition~\ref{definition:negentropy}.
Thus, we have prove that $\Theta$ is a regular negentropy.
\end{proof}}{}

\begin{proof}[Proof of Proposition~\ref{proposition:non-strongly-convex-negentropy}]
First we note that the assumptions on
\(\sqfunc\) is in the setting of
  Proposition~\ref{proposition:Legendre-composed-with-TF-is-again-Legendre}.
  Next, by Definition~\ref{definition:totally-regular-negentropy}, we must show  that $\Theta^{(n)}$ is a regular negentropy for each $n \in \{2,\dots,k\}$.
  Let $a = \Omega(\bfu)$.
  Note that since $\Omega$ is symmetric and convex, we must have that $a \le \Omega(\bfp)$ for all $\bfp \in \Delta^{k}$.
  Furthermore, it is easy to see that
  $
\tilde{\Theta}^{(n)}(\tilde{\bfq}) := \sqfunc(\tilde{\Omega}^{(n)}(\tilde{\bfq}) - a) - \sqfunc(-a)$ for all $\bfq \in \tilde{\Delta}^{n}
$.
Now, apply Proposition~\ref{proposition:transform-of-regular-negentropy} to $\Theta^{(n)}$ and $a = \Omega(\bfu)$, we get the desired result.
The part of the proposition assuming \(g'(0) = 0\) follows immediately from Corollary~\ref{corollary:never-strongly-convex}.
\end{proof}

% \color{red}
% \input{perm_GP}
% \color{black}

  % \subfile{perm_FY}
\section{Uniqueness of the matrix label code}\label{section:uniqueness-of-MLC}

In this section, we prove Theorem~\ref{theorem:uniqueness-of-MLC-main}. The proof appears at the end of this section.

\begin{theorem}\label{theorem:uniqueness-of-MLC-main}
  Let \( \calL \) be the cross entropy loss and \(\TP\)  be its template as in Example~\ref{example:cross entropy}.
  Suppose that \(\{\mathbf{A}_{y}\}_{y=1}^{k}\) is a set of \((k-1)\times(k-1)\) matrices satisfying
  \(\calL_{y}(\bfv) =  \TP(\mathbf{A}_{y} \bmpi \bfv)\) for all \(y \in [k]\) and all \(\bfv \in \mathbb{R}^{k}\). Then for each \(y \in [k]\) there exists a permutation \( \sigma_{y} \in \mathtt{Sym}(k)\) so that \(\mathbf{A}_{y} = \mathbf{S}_{\sigma_{y}}\ico_{y}\).
  \end{theorem}

\begin{lemma}\label{lemma:uniqueness-of-MLC-1}
  Let \( \TP \) be the template of the multinomial logistic loss (Example~\ref{example:cross entropy}).
  Let \(\mathbf{a} \in \mathbb{R}^{k-1}\) be a vector. Suppose that
\( \TP(t\mathbf{a}) = \TP(t \bfe_{1} )\) for all  \(t \in \mathbb{R} \).
Then \(\mathbf{a} = \bfe_{i}\)  for some \(i \in [k-1]\).
\end{lemma}
\iftoggle{arxiv}{\begin{proof}
  First, we show that \( \mathbf{a} \ge 0 \) is entrywise nonnegative. To this end, suppose that there exists \( j \in [k-1]\) such that \(a_{j} <0\). Then by the monotonicity of \(\log\) we have
  \[
    \TP(t \mathbf{a}) = \log(1 + \textstyle{\sum_{i=1}^{k-1} \exp(-a_{i} t)})
    \ge \log(1 + \exp(-a_{j} t))
  \]
  Thus, \(a_{j} <0 \) implies that \(\lim_{t \to +\infty} \TP(t \mathbf{a}) \ge  \lim_{t \to +\infty} \log(1 + \exp(-a_{j} t)) = + \infty\).
  On the other hand, \(\lim_{t \to \infty} \TP(t \bfe_{1}) < +\infty\) which is a contradiction.

  Next, we show
\(\mathbf{a} \in \Delta^{k}\)
  i.e., that the entries of \(\mathbf{a}\) sum up to \(1\).
  We note that
  \[
    \frac{d}{d t} \TP(t \mathbf{a}) = \mathbf{a}^{\top} \nabla_{\TP}(t \mathbf{a}), \quad \mbox{and} \quad
    \frac{d}{d t} \TP(t \bfe_1) = \bfe_1^{\top} \nabla_{\TP}(t \bfe_1)
  \]
  In particular, evaluated at \(t = 0\), we have
  \(\mathbf{a}^{\top} \nabla_{\TP}(\vzero)
=
\bfe_{1}^{\top} \nabla_{\TP}(\vzero).
\)
Since the gradient of \(\TP\) at \(\vzero\) is a scalar multiple of the all-ones vector, we have
  \(    1 = \vone^{\top} \bfe_{1} = \vone^{\top} \mathbf{a}.
\)
  Finally, we derive the result via  the equality-attaining part of the Jensen's inequality:
  \[
    \TP(\mathbf{a}) = \TP(\sum_{j=1}^{k-1} a_{j}\bfe_{j})
    \le \sum_{j=1}^{k-1} a_{j}\TP(\bfe_{j})
    =
    \TP(\bfe_{1})
  \]
  attains equality if and only if \(\mathbf{a} = \bfe_{i}\) for some \(i \in [k-1]\).
  Note that the right-most equality holds since \(\TP\) is a symmetric function
  (Theorem~\ref{theorem:relative-margin-form}).
\end{proof}}{
\textcolor{black}
{
  {Lemma}~\ref{lemma:uniqueness-of-MLC-1} serves as the main tool for
  {Lemma}~\ref{lemma:uniqueness-of-MLC-2} below.
  The proofs of both Lemmas~\ref{lemma:uniqueness-of-MLC-1} and
\ref{lemma:uniqueness-of-MLC-2}
  appear in the arXiv version of our manuscript \citep{wang2023unified}.
}
}

\begin{lemma}\label{lemma:uniqueness-of-MLC-2}
  Let \( \TP \) be the template of the multinomial logistic loss (Example~\ref{example:cross entropy}).
Suppose that \(\mathbf{A}\) is such that \(\TP(\bfz) = \TP(\mathbf{A} \bfz)\) for all \(\bfz \in \mathbb{R}^{k-1}\), then \(\mathbf{A}\) is the identity matrix up to row permutation.
In other words, there exists \(\sigma \in \mathtt{Sym}(k)\) such that \(\mathbf{A} = \mathbf{S}_{\sigma}\).
\end{lemma}
\iftoggle{arxiv}{\begin{proof}
  Consider a fixed \(i \in [k-1]\) and let \(\bfz = \bfe_{i}\). The assumption in Lemma~\ref{lemma:uniqueness-of-MLC-2} then states that
\(\TP(\bfe_{1}) = \TP(\bfe_{i}) = \TP(\mathbf{a}_{i} )\).
Thus, by the previous Lemma~\ref{lemma:uniqueness-of-MLC-1}, there exists \(j_{i} \in [k-1]\) so that \(\mathbf{a}_{j_{i}} = \bfe_{i}\).
We claim that the mapping \( [k-1 ] \ni i \mapsto j_{i} \in [k-1]\) is a bijection.
The claim holds iff \(i \mapsto j_{i}\) is a surjection iff the vector \(\mathbf{A} \vone\) does not have any zero entry.
Note that this would imply that \(\mathbf{A}\) is a permutation matrix.

  To prove the claim, let \(\mathbf{b} := \mathbf{A} \vone\) and suppose that \(\mathbf{b}\) has a zero entry at \(j \in [k-1]\), i.e., \(b_{j} = 0\).
  Let \(t \in \mathbb{R}\) be an arbitrary number and
  let \(\bfz  := t\vone\) in the assumption of Lemma~\ref{lemma:uniqueness-of-MLC-2}.
  Then we have by this assumption that
  \(\TP(t \vone) = \TP(t\mathbf{A}\vone) = \TP(t \mathbf{b})\).
  Note that
  \[
    \lim_{t\to +\infty}
    \TP(t\vone)
    =
    \lim_{t\to +\infty}
    \log(1 + \textstyle{\sum_{i=1}^{k-1} \exp(-t)})
    =
    \log(1 )
    =0.
  \]
  On the other hand
  \[
    \lim_{t\to +\infty}
    \TP(t\mathbf{b})
    \ge
    \lim_{t\to +\infty}
    \log(1 + \exp(-tb_{j}))
    =
    \lim_{t\to +\infty}
    \log(1 + 1)
    >0.
  \]
  This is a contradiction and proves that \(\mathbf{A} \vone = \vone\).
\end{proof}}{}

\begin{lemma}\label{lemma:uniqueness-of-MLC-3}
  Let \( \TP \) be the template of the multinomial logistic loss (Example~\ref{example:cross entropy}).
  Let \(j \in [k]\) and \(\mathbf{A} \in \mathbb{R}^{(k-1)\times(k-1)}\) be arbitrary. Suppose that
  \( \TP(\ico_{j} \bfz) = \TP(\mathbf{A} \bfz)\) for all \(\bfz \in \mathbb{R}^{k-1}\).
  Then \(\mathbf{A}\) is equal to the \(\ico_{j}\) up to row permutation.
  In other words, there exists \(\sigma \in \mathtt{Sym}(k)\) such that \( \mathbf{A} = \mathbf{S}_{\sigma} \ico_{j}\).
\end{lemma}
\begin{proof}[Proof of Lemma~\ref{lemma:uniqueness-of-MLC-3}]
  We first prove that the assumption of Lemma~\ref{lemma:uniqueness-of-MLC-3}
  implies that of the previous Lemma~\ref{lemma:uniqueness-of-MLC-2}:
\( \TP(\bfz) = \TP(\mathbf{A} \bfz)\) for all \(\bfz \in \mathbb{R}^{k-1}\).
  Let \(\mathbf{u} \in \mathbb{R}^{k-1}\) be arbitrary and let \(\mathbf{z} := \ico_{j}\mathbf{u}\).
  Since \(\ico_{j}^{2} =\mathbf{I}\), we have
  \(\TP(\bfu) = \TP(\ico_{j} \ico_{j}\bfu) =\TP(\ico_{j} \bfz) = \TP(\mathbf{A} \bfz) = \TP(\mathbf{A} \ico_{j} \bfu)\)
  where  for the third equality from the left we used the assumption of Lemma~\ref{lemma:uniqueness-of-MLC-3}.

  Now, by the previous Lemma~\ref{lemma:uniqueness-of-MLC-2}, we have that \(\mathbf{A} \ico_{j} = \mathbf{S}_{\sigma}\) for some \(\sigma \in \mathtt{Sym}(k)\).
Using \(\ico_{j}^{2} = \mathbf{I}\), we get
\( \mathbf{S}_{\sigma} \ico_{j} = \mathbf{A}\ico_{j}\ico_{j} = \mathbf{A}\). That is to say \(\mathbf{A}\) is equal to \(\ico_{j}\) up to permutations of the rows.
\end{proof}

\begin{proof}[Proof of {Theorem}~\ref{theorem:uniqueness-of-MLC-main}]
  Let \(y \in [k]\) and \(\bfv \in \mathbb{R}^{k}\) be fixed. Let \(\bfz := \bmpi \bfv\).
  By Theorem~\ref{theorem:relative-margin-form}, we have
  \(\TP(\ico_{y} \bfz) = \TP(\mathbf{A}_{y} \bfz)\). Since \(\bmpi\) is surjective, the preceding identity holds for all \(\bfz \in \mathbb{R}^{k-1}\). Thus
by Lemma~\ref{lemma:uniqueness-of-MLC-3}, \(\mathbf{A}_{y} = \mathbf{S}_{\sigma} \ico_{y}\) for some \(\sigma \in \mathtt{Sym}(k)\).
\end{proof}

% \section{Relationship to simplex code}

% A simplex code \cite{mroueh2012multiclass}
% is a matrix
% \(\mathbf{C} =
% \begin{bmatrix}
%   c_{1} & \cdots, & c_{k}
% \end{bmatrix} \in \mathbb{R}^{(k-1)\times k}
% \)
% such that
% \[
%   c_{i}^{\top} c_{j} =
%   \begin{cases}
%     1 & : i = j \\
%     -\tfrac{1}{k-1} & i \ne j
%   \end{cases}
% \]
% In other words, we have
% \begin{equation}
%   \mathbf{C}^{\top} \mathbf{C} =
%   \tfrac{k}{k-1}\mathbf{I}_{k}-\tfrac{1}{k-1} \mathtt{ones}(k,k)
% \label{equation:simplex-code-matrix-form}
% \end{equation}
% In the simplex code framework, functions of the form
% \(\tilde{g} : \mathcal{X} \to \mathbb{R}^{k-1}\)
% convert to a class-score function via the rule:
% \[
% f(x)=\mathbf{C}\tilde{g}(x)
% \]

% By contrast, the relative-margin form converts a function \(g : \mathcal{X} \to \mathbb{R}^{k-1}\)to a
% class-score
% function via the rule
% \[
%   f(x) =
%   \bmpi^{\dagger} g(x).
% \]
% We note that
% \[
%   \bmpi^{\dagger} g(x)
%   =
%   \bmpi^{\dagger} \bmpi
%   \bmpi^{\dagger}
% g(x)
% \]
% because \(\bmpi\bmpi^{\dagger}\) is the identity.
% Moreover,
% \[
%   \tfrac{k}{k-1}
%   \bmpi^{\dagger} \bmpi
%   =
%   \tfrac{k}{k-1}\mathbf{I}_{k}-\tfrac{1}{k-1} \mathtt{ones}(k,k)
% \]

% Hence, we have
% \[
% \bmpi^{\dagger} g(x) = \mathbf{C} \tilde{g}(x)
% \]
\section{Mathematical Background}
We review mathematical background on fundamental topics that are used throughout this work.

\subsection{Non-singular M-matrix}
\label{section:M-matrix}
We recall some definitions from linear algebra.
\begin{definition}\label{definition:M-matrix}
  Let $\bfA = (a_{ij}) \in \mathbb{R}^{n\times n}$ be a matrix.
  We say that $\bfA$ is a
  \begin{enumerate}
    \item \emph{Z-matrix} if $a_{ij} \le 0$ whenever $i \ne j$.
\item \emph{M-matrix} if $\bfA$ is a Z-matrix and all eigenvalues of $\bfA$ have nonnegative real parts.
\item \emph{strictly diagonally dominant matrix} if
  $|a_{ii}| > \sum_{j \in [n]: j \ne i} |a_{ij}|$ for all $i \in [n]$.
\item \emph{monotone matrix} if for all $\bfx \in \mathbb{R}^n$,  $\bfA \bfx \succeq 0$ implies $\bfx \succeq 0$.
  % If $\bfA$ additionally satisfies $\bfA \bfx > 0$ implies $\bfx >0$,
          If, in addition,  $\bfA \bfx \succ 0$ implies $\bfx \succ 0$,  then $\bfA$ is said to be \emph{strictly monotone}.
  \end{enumerate}
\end{definition}
% The following result is known as the Levy–Desplanques or the Gershgorin circle theorem.
\begin{theorem}[Levy-Desplanques/Gershgorin circle]
  Let $\bfA$ be a strictly diagonally dominant matrix. Then
  $\bfA$ is non-singular whose
eigenvalues all have nonnegative real parts.
\end{theorem}
The above result immediately implies the following:
\begin{corollary}
  \label{corollary: strictly diagonally dominant Z-matrix is nonsingular M-matrix}
A strictly diagonally dominant Z-matrix is a non-singular M-matrix.
\end{corollary}
Non-singular M-matrices have many equivalent characterizations. The one relevant to us is the following:
\begin{theorem}[\cite{plemmons1977m}]
  \label{theorem: non-singular M-matrix iff monotone}
  Let $\bfA$ be a Z-matrix.
  Then $\bfA$ is a non-singular M-matrix if and only $\bfA$ is a monotone matrix.
\end{theorem}

\begin{lemma}
  \label{lemma: strictly monotone lemma}
  Let $\bfA = (a_{ij}) \in \mathbb{R}^{n\times n}$ be a non-singular M-matrix. If the diagonals of $\bfA$ are positive, then $\bfA$ is strictly monotone.
\end{lemma}
\begin{proof}
  Let \(\bfx \in \mathbb{R}^{n}\) be arbitrary such that \(\bfA \bfx \succ \vzero\).
  Our goal is to show that \(\bfx \succ \vzero\).
  First, from Theorem \ref{theorem: non-singular M-matrix iff monotone}, we have that $\bfA$ is monotone.
  Thus, $\bfA \bfx \succ \vzero$ implies $\bfx \succeq \vzero$.
  We only have to check additionally that $\bfx \succ \vzero$.
  Since $\bfA$ is a Z-matrix, the off-diagonals are non-positive, i.e., $a_{ij} \le 0$ for all \(i,j \in [n]\) such that \(i\ne j\).
  Now, let \(i \in [n]\). We need to check that $x_i > 0$.
  To this end, note that
  \(0 < [\bfA \bfx]_i = \sum_{j = 1}^n a_{ij} x_j
  =
  a_{ii} x_i
+ \sum_{j\ne i} a_{ij} x_j
  \le
a_{ii} x_i. \)
Note that \(\sum_{j\ne i} a_{ij} x_j \le 0\)  because \(x \succeq \vzero\) and \(a_{ij} \le 0\).
Finally,  \(x_i > 0\) since \(a_{ii} > 0\).
\end{proof}

\begin{corollary}\label{corollary:strictly-monotone-transpose}
  Let $\bfA = (a_{ij}) \in \mathbb{R}^{n\times n}$ be a non-singular M-matrix. Then \(\bfA^{\top}\) is also a non-singular M-matrix. Moreover,  if the diagonals of $\bfA$ are positive, then $\bfA^{\top}$ is strictly monotone.
\end{corollary}
\begin{proof}
  Note that \(\bfA\) and \(\bfA^{\top}\) are both non-singular and have the same set of eigenvalues. Moreover, \(\bfA\) being a Z-matrix implies that \(\bfA^{\top}\) is a Z-matrix.
  Thus, \(\bfA^{\top}\) is a non-singular M-matrix.
Finally, Lemma~\ref{lemma: strictly monotone lemma} implies that \(\bfA^{\top}\) is a strictly monotone matrix.
\end{proof}

% changing old calculus notation to new ones
\iftoggle{arxiv}{
  \subsection{Vector calculus}\label{section:vector-calculus}
  Let $f = (f_1,\dots, f_m): \mathbb{R}^n \to \mathbb{R}^m$ be a differentiable function.
  Below, let \(\bfx \in \mathbb{R}^{n}\) denote a generic input variable of \(f\). Let \(i \in [n]\) denote the index of the dimensions of \(\bfx\),  and let \(j\in [m]\) denote the index of the component functions of \(f\).
  The \emph{gradient} of $f$ at $\bfx \in \mathbb{R}^n$, denoted $\nabla_f(\bfx)$, is the $n \times m$ matrix
  whose $(i,j)$-th entry
  \(  [\nabla_f(\bfx)]_{ij}
  :=
  \frac{\partial f_i}{\partial x_j}(\bfx)
  \).
  Note that we can write the above as
  \[
    \nabla_f(\bfx)
    =
    \begin{bmatrix}
      \nabla_{f_1} (\bfx) & \cdots & \nabla_{f_m}(\bfx)
    \end{bmatrix}.
  \]
  If $f(\bfx) = \bfA \bfx$ for a matrix $\bfA \in \mathbb{R}^{m\times n}$, then $\nabla_f(\bfx) = \bfA^{\top}$.
  The transpose of the gradient (matrix) is referred to as the \emph{Jacobian matrix}.
  Below, we cite results from \cite{munkres2018analysis} regarding the Jacobian matrix.
  The two  results below are  Ch2.-Theorem 7.1 and Ch.2-Theorem 8.2 from \citet{munkres2018analysis}, respectively, restated in terms of the gradient.
  \begin{theorem}[Chain rule]\label{theorem:chain-rule}
    If \(f : \mathbb{R}^n \to \mathbb{R}^m\) and \(g : \mathbb{R}^m \to \mathbb{R}^l\) are differentiable, then
    we have \( \nabla_{g \circ f}(\bfx) = \nabla_f(\bfx) \nabla_g(f(\bfx)) \).
  \end{theorem}
  A special case of the chain rule is that
if $f(\bfx) = \bfA \bfx$ for a matrix $\bfA \in \mathbb{R}^{m\times n}$,
then
    \(
    \frac{\partial}{\partial \bfx} g(\bfA \bfx)
    =
    \bfA^{\top}
    \nabla_g(\bfA \bfx)
    \).

  \begin{theorem}[Inverse function theorem]\label{theorem:inverse-function-theorem}
    Let $U$ be open in $\mathbb{R}^n$, $f:U \to \mathbb{R}^n$ be $r$-times continuously differentiable and $V = f(U)$.
    If $f$ is one-to-one on $U$ and if $\nabla_f(\bfx)$ is non-singular for all $\bfx \in U$, then $V$ is open in $\mathbb{R}^n$ and the inverse function $g : V \to U$ is $r$-times continuously differentiable.
  \end{theorem}

  The following is an immediate consequence of Theorems~\ref{theorem:inverse-function-theorem} and \ref{theorem:chain-rule}:
  \begin{corollary}\label{corollary:inverse-function-theorem}
    In the setting of Theorem~\ref{theorem:inverse-function-theorem}, we have
    \(\nabla_{f^{-1}}(f(\bfx)) = \nabla_{f}(\bfx)^{-1}. \)
  \end{corollary}

  \begin{remark}[Hessians]\label{remark:hessians}
    If \(\varphi : \mathbb{R}^{n} \to \mathbb{R}\) is a scalar-valued functions of a vector-valued input \(\bfx \in \mathbb{R}^{n}\), then the gradient \(\nabla_{\varphi} : \mathbb{R}^{n} \to \mathbb{R}^{n}\) is a vector-valued function.
    In this case, it makes sense to write \(\nabla_{\nabla_{\varphi}}(\bfx)\) which is a \(n\times n\) matrix. Indeed, this is simply the \emph{Hessian} of \(\varphi\) at \(\bfx\) which we denote by \(\nabla^{2}_{\varphi}(\bfx)\), following the standard convention.
  \end{remark}

  \begin{remark}[Time derivatives]\label{remark:time-derivatives}
    For \(f: \mathbb{R} \to \mathbb{R}^{m}\) differentiable functions of a univariate ``time'' variable \(t \in \mathbb{R}\), note that \(\nabla_{f} (t)
    \) is a \(1\times m\) matrix, i.e., a \emph{row} vector of the time derivatives.
    It is convenient to adopt the notation \(\frac{df}{dt} (t) = f'(t) := \nabla_{f}(t)^{\top}\) to denote the time derivatives as a \emph{column} vector.
    In this notation, the \emph{product rule} for the time derivative of the inner product of two functions \(f\) and \(g: \mathbb{R} \to \mathbb{R}^{m}\) can be stated as
    \begin{equation}
      \tfrac{d}{dt} ( f(t)^{\top}g(t))
      = f'(t)^{\top} g(t)+ f(t)^{\top} g'(t).
      \label{equation:product-rule-for-curves}
    \end{equation}
    \Cref{equation:product-rule-for-curves} follows immediately from the ordinary product rule from elementary calculus.
    If \(h: \mathbb{R}^{m} \to \mathbb{R}^{n}\) is a differentiable function,
    then the \emph{chain rule} for computing the time derivative of the composition \(h \circ g\) can be stated as
    \begin{equation}
      \tfrac{d}{dt} (h(g(t))) =
      \nabla_{h}(g(t))^{\top} g'(t).
      \label{equation:chain-rule-for-curves}
    \end{equation}
    To see this, note that by definition, we have
    \(
    \tfrac{d}{dt} (h(g(t))) =
    \nabla_{h \circ g} (t)^{\top}\). Now, by Theorem~\ref{theorem:chain-rule}
    \(\nabla_{h \circ g} (t) = \nabla_{g}(t) \nabla_{h}(g(t)) = g'(t)^{\top} \nabla_{h}(g(t))\).
    Taking transpose, we get
    \cref{equation:chain-rule-for-curves}.
  \end{remark}
}{}

\end{document}